\documentclass[10pt,twocolumn,letterpaper]{article}
\usepackage[pagebackref]{hyperref}
\usepackage[pagenumbers]{cvpr}
\usepackage{xcolor}
\usepackage{bm}
\usepackage{xcolor}
\usepackage{amsmath}
\usepackage{amsfonts}
\usepackage{amssymb}
\usepackage{mathrsfs}
\usepackage{amsthm}
\usepackage{graphicx}
\usepackage{hyperref}
\hypersetup{colorlinks,linkcolor=blue,anchorcolor=blue,citecolor=blue}
\usepackage{pifont}
\usepackage{booktabs}      
\usepackage{siunitx}
\usepackage{tikz}
\usepackage{pgfplots}
\usepackage{longtable}
\pgfplotsset{compat=1.18} 

\definecolor{myPurple}{RGB}{128,0,128}

\newcommand{\bphi}{\bm{\phi}}

\newcommand{\bomega}{\bm{\omega}}

\newcommand{\bOmega}{\bm{\Omega}}

\newtheorem{theorem}{Theorem}
\newtheorem*{theorem_unnumbered}{Theorem}
\newtheorem{definition}{Definition}

\newtheorem{lemma}{Lemma}

\allowdisplaybreaks
\makeatletter

\usepackage{xcolor}

\definecolor{cvprblue}{rgb}{0.21,0.49,0.74}

\def\paperID{} 
\def\confName{CVPR}
\def\confYear{2026}

\title{Content-Aware Frequency Encoding for Implicit Neural Representations with Fourier-Chebyshev Features}

\author{
    Junbo Ke\textsuperscript{1} \quad
    Yangyang Xu\textsuperscript{1} \quad
    You-Wei Wen\textsuperscript{1,}\thanks{Corresponding author: \href{mailto:wenyouwei@gmail.com}{wenyouwei@gmail.com}.} \quad
    Chao Wang\textsuperscript{2} \\
    \textsuperscript{1}School of Mathematics and Statistics, Hunan Normal University\\
    \textsuperscript{2}Department of Statistics and Data Science, Southern University of Science and Technology\\
}

\begin{document}
\maketitle

\begin{abstract}
Implicit Neural Representations (INRs) have emerged as a powerful paradigm for various signal processing tasks, but their inherent spectral bias limits the ability to capture high-frequency details. Existing methods partially mitigate this issue by using Fourier-based features, which usually rely on fixed frequency bases. This forces multi-layer perceptrons (MLPs) to inefficiently compose the required frequencies, thereby constraining their representational capacity. To address this limitation, we propose Content-Aware Frequency Encoding (CAFE), which builds upon Fourier features through multiple parallel linear layers combined via a Hadamard product. CAFE can explicitly and efficiently synthesize a broader range of frequency bases, while the learned weights enable the selection of task-relevant frequencies. Furthermore, we extend this framework to CAFE+, which incorporates Chebyshev features as a complementary component to Fourier bases. This combination provides a stronger and more stable frequency representation.  Extensive experiments across multiple benchmarks validate the effectiveness and efficiency of our approach, consistently achieving superior performance over existing methods. Our code is available at \url{https://github.com/JunboKe0619/CAFE}.

\end{abstract}
\section{Introduction}
Implicit Neural Representations (INRs)~\cite{sitzmann2020implicit} are a class of methods that learn a continuous mapping from coordinates to signal or scene values using neural networks. Unlike discrete representations~\cite{choy20163d,riegler2017octnet}, INRs generate continuous outputs for accurate reconstruction at arbitrary points. These properties have attracted extensive attention in recent years, with applications spanning continuous image super-resolution~\cite{chen2021learning,jiang2025hiif,cao2023ciaosr}, image compression~\cite{jayasundara2025sinr,han2025towards,strumpler2022implicit}, inverse problems~\cite{9606601,najaf2024towards,zhu2024disorder}, and neural rendering~\cite{mildenhall2021nerf,muller2022instant}. 

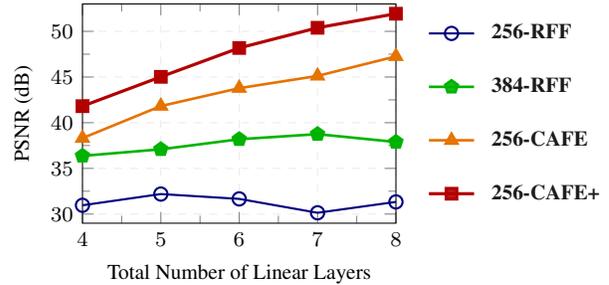
\begin{figure}[htbp]
    \centering
    \begin{tikzpicture}
        \begin{axis}[
            xlabel={Total Number of Linear Layers},
            ylabel={PSNR (dB)},
            xmin=4, xmax=8,
            ymin=29, ymax=53,
            xtick={4,5,6,7,8},
            ytick={30,35,40,45,50},
            axis lines=box,
            width=0.5\linewidth,
            height=0.35\linewidth,
            scale only axis,
            grid=major,
            grid style={dashed, gray!15},
            tick label style={font=\footnotesize},
            label style={font=\footnotesize},
            minor tick num=1,
            legend style={
                at={(1.08,0.5)},       
                anchor=west,
                minimum height=0.06\linewidth,  
                yshift=0pt,
                align=left,
                legend columns=1,
                font=\footnotesize\bfseries,
                draw=none,
                fill=white,
                fill opacity=0.9,
                inner sep=2pt,
                column sep=0pt,
                row sep=6pt,
                column sep=4pt,
                legend cell align=left,
                rounded corners=0pt,
                thick,
            }
        ]

            \addplot[
                color=blue!50!black, line width=0.8pt,
                mark=o, mark size=2.5pt, mark options={fill=white, draw=blue!50!black}
            ] coordinates {
                (4,30.96) (5,32.19) (6,31.66) (7,30.14) (8,31.33)
            };
            \addlegendentry{256-RFF}

            \addplot[
                color=green!70!black, line width=1pt,
                mark=pentagon*, mark size=2.5pt, mark options={fill=green!70!black}
            ] coordinates {
                (4,36.37) (5,37.09) (6,38.19) (7,38.75) (8,37.89)
            };
            \addlegendentry{384-RFF}

            \addplot[
                color=orange!90!black, line width=1pt,
                mark=triangle*, mark size=2.5pt, mark options={fill=orange!90!black}
            ] coordinates {
                (4,38.31) (5,41.82) (6,43.78) (7,45.12) (8,47.26)
            };
            \addlegendentry{256-CAFE}

            \addplot[
                color=red!70!black, line width=1.2pt,
                mark=square*, mark size=2.0pt, mark options={fill=red!70!black}
            ] coordinates {
                (4,41.81) (5,45.02) (6,48.18) (7,50.39) (8,51.93)
            };
            \addlegendentry{256-CAFE+}

        \end{axis}
    \end{tikzpicture}
    \vspace{-0.5em}
    \caption{Effect of the total number of linear layers on the image fitting task in terms of PSNR. For the RFF baselines, linear layers exist only in the MLP. For our methods (CAFE and CAFE+), the total number of linear layers includes those in both the encoding stage and the MLP. Any increase in this total arises from additional layers in the encoding stage, while the MLP depth remains fixed. The legend indicates the number of hidden neurons per layer.
    }
    \label{fig:psnr_time_dualaxis}
    \vspace{-0.5em}
\end{figure}
\setlength{\tabcolsep}{1pt} 
\begin{figure}[htbp]  
    \centering  
    \begin{tabular}{ccc}
        \makebox[0.32\linewidth][c]{\small\text{RFF}} &  
        \makebox[0.32\linewidth][c]{\small\text{CAFE}} &  
        \makebox[0.32\linewidth][c]{\small\text{CAFE+}} \\  

        \includegraphics[width=0.31\linewidth]{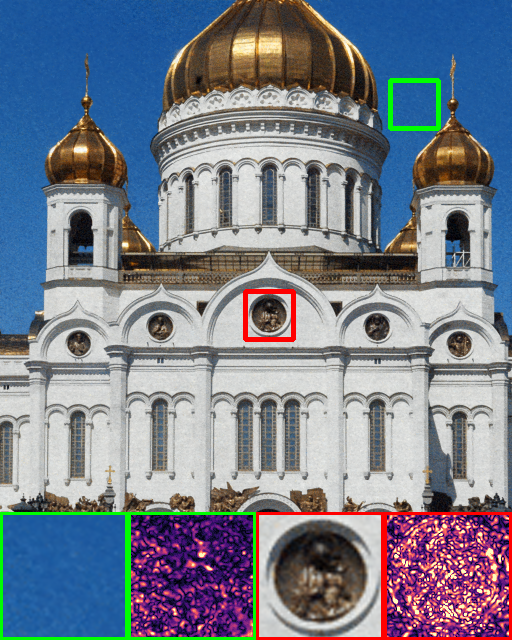} &
        \includegraphics[width=0.31\linewidth]{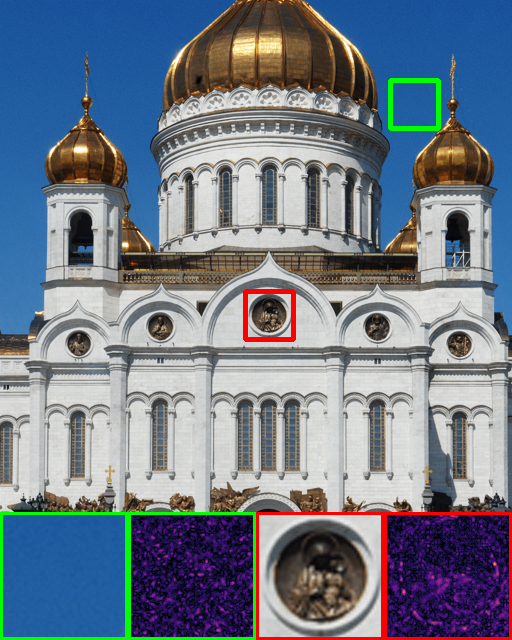} &
        \includegraphics[width=0.31\linewidth]{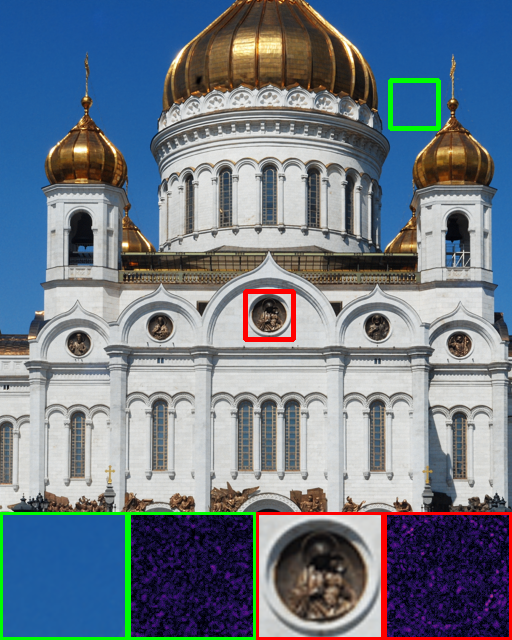} \\

        \makebox[0.31\linewidth][c]{PSNR = 32.19 dB} &
        \makebox[0.31\linewidth][c]{PSNR = 41.82 dB} &
        \makebox[0.31\linewidth][c]{PSNR = 45.02 dB} \\
    \end{tabular}
    \vspace{-0.5em}
    \caption{
    Qualitative comparison of RFF, CAFE, and CAFE+ under the same number of parameters for the image fitting task. High- and low-frequency regions are zoomed in, and corresponding error maps are provided for a detailed evaluation.
    }
    \label{fig:encoding_compare}  %
\end{figure}

 Despite these successes, INRs are known to suffer from spectral bias~\cite{rahaman2019spectral,yuce2022structured}, where neural networks tend to preferentially capture low-frequency components, thereby limiting their ability to represent high-frequency details.
To address the aforementioned limitations, a widely adopted strategy is to use Fourier features mappings, such as Random Fourier Features (RFF)~\cite{tancik2020fourier} or the Positional Encoding (PE)~\cite{mildenhall2021nerf}. These techniques project input coordinates into a high-dimensional space using a predefined set of sinusoidal basis functions. This mapping has been shown to enhance the spectral characteristics of the neural tangent kernel (NTK)~\cite{jacot2018neural}, thereby enabling a multilayer perceptron (MLP) to more accurately approximate the high-frequency components of the target signal.

However, the effectiveness of these approaches is fundamentally limited by their fixed frequency bases.
As a result, the MLP must implicitly synthesize the target frequencies based on these predefined bases through its nonlinear transformations, which is theoretically feasible~\cite{yuce2022structured} but practically inefficient and difficult to optimize. As shown in Fig.~\ref{fig:psnr_time_dualaxis}, although increasing the network depth should, in principle, enhance the model’s ability to compose frequencies, it brings almost no improvement in reconstruction accuracy. Furthermore, widening the network leads to moderate gains at the expense of a substantial increase in parameters. This analysis motivates us to shift the burden of target frequency synthesis from the MLP to the encoding stage, rather than enlarging the network capacity.

Building upon this idea, we propose \emph{Content-Aware Frequency Encoding (CAFE)}. This framework implements the shift by replacing fixed, stochastic bases with a dynamic mechanism that learns to generate a set of frequency bases optimally adapted to the target signal's content.
Specifically, sinusoidal bases generated from Fourier features are projected through $N$ parallel linear layers, whose outputs are combined via Hadamard product to induce frequency interactions.
Our analysis reveals that this design explicitly expands the representable frequency space from $M$ fixed Fourier bases to $\mathcal{O}(M^N  3^{N-1})$ components, while allowing the adaptive selection of task-relevant frequencies through learned weights. This explicit synthesis of target frequencies proves to be more effective, as Fig.~\ref{fig:psnr_time_dualaxis} demonstrates that our method benefits from an increased number of linear layers at the encoding stage.

While CAFE provides a powerful and adaptive mechanism for frequency synthesis, the range of represented frequencies still depends on the initialization of the Fourier features. Since neural networks tend to learn low-frequency components first~\cite{xu2019frequency}, the network may be forced to overuse high-frequency bases to compensate for the missing low-frequency information if CAFE fails to include essential low-frequency bases. This introduces noise in the low-frequency region and also harms the reconstruction of high-frequency signals~\cite{ma2025robustifying, hertz2021sape}.

To achieve a more stable and robust representation of low-frequency components, we propose the Chebyshev features, leveraging the inherent stability and strong approximation properties of Chebyshev polynomials~\cite{powell1981approximation}. We adopt Chebyshev features as a complement to Fourier features, allowing our encoding to more effectively cover the full frequency spectrum. We refer to the enhanced variant, which integrates Fourier-Chebyshev features, as CAFE+. 
As illustrated in Fig.~\ref{fig:encoding_compare}, CAFE+ exhibits better noise suppression in low-frequency regions and stronger fitting capability for high-frequency components, thereby effectively alleviating the aforementioned issues.

To summarize, our main contributions are as follows:
\begin{itemize}
    \item We propose CAFE, a novel encoding framework that adaptively selects task-relevant frequencies from an exponentially expanded spectrum, thereby significantly alleviating the burden of frequency synthesis on the MLP.
    
    \item We introduce the Chebyshev features as a complementary component to Fourier features, providing a stronger and more stable low-frequency representation and further enhancing the  representational capability.
    
    \item Our framework achieves state-of-the-art results, and extensive experiments further demonstrate its effectiveness and generalization ability across various INR tasks.
\end{itemize}

\section{Related Work}

Implicit Neural Representations (INRs) often suffer from spectral bias~\cite{rahaman2019spectral,yuce2022structured} when using standard MLPs, which restricts their ability to capture high-frequency details. In response to this limitation, research has mainly focused on two representative directions: designing specialized activation functions and developing more effective input encoding strategies.

\noindent
{\bf Activation Functions.} SIREN~\cite{sitzmann2020implicit} first introduced sinusoidal activation functions along with an effective initialization scheme. Subsequently, GAUSS~\cite{ramasinghe2022beyond} proposed Gaussian nonlinear activations, while WIRE~\cite{saragadam2023wire} developed continuous complex Gabor wavelet activations, enabling accurate and noise-robust signal reconstruction. FINER~\cite{liu2024finer} further introduced variable-periodic activations, which allow flexible spectral bias tuning to enhance both high-frequency detail representation and overall performance.  In addition, more 
activation functions have also been demonstrated to be effective ~\cite{serrano2024hosc,thennakoon2025bandrc,jayasundara2025pin}. Recently, MIRE~\cite{jayasundara2025mire} introduced a dictionary learning framework that dynamically selects optimal activation functions for each INR layer from a predefined atomic library, enhancing adaptability to task-specific frequency characteristics.

\noindent
{\bf Input Encoding Strategies.} Fourier features, represented by PE~\cite{mildenhall2021nerf} and RFF~\cite{tancik2020fourier}, map coordinates into high-dimensional sinusoidal spaces, boosting MLPs' capability to represent high-frequency signals. These methods, however, are sensitive to the choice of initial sampling frequencies, and their overall representation capability still has room for improvement. To address these issues, SAPE~\cite{hertz2021sape} progressively exposes frequency components across time and space, while SCONE~\cite{li2024learning} uses learnable spatial masks to assign Fourier bases to different regions for improved local representation. Nevertheless, both approaches remain constrained by their reliance on fixed Fourier bases.  MFN~\cite{fathony2020multiplicative} and BACON~\cite{lindell2022bacon} rely on Hadamard products to recursively synthesize frequencies in a serial manner.  In contrast, our method is encoding-based and adopts a parallel architectural design, enabling faster training while eliminating the need for the carefully engineered initialization required by BACON. Besides Fourier features, adaptive wavelet-positional encoding~\cite{zhao2025adaptive} has also been proposed to enhance high-frequency representation. In recent work, SL$^2$A~\cite{rezaeian2025sl2a} introduced the use of KAN~\cite{liu2025kan}, which substantially improves the model’s expressive capacity.

\section{Methodology}

\subsection{Preliminaries}
An INR models a continuous signal as a function 
$f_{\boldsymbol{\theta}}: \mathbb{R}^D \!\to\! \mathbb{R}^{\tilde{D}}$ 
parameterized by an MLP:
\begin{equation}
\begin{aligned}
\mathbf{z}^{(0)} &= \gamma(\mathbf{x}), \quad \mathbf{x} \in \mathbb{R}^D, \\
\mathbf{z}^{(\ell)} &= \rho^{(\ell)}\!\left(\mathbf{W}^{(\ell)} \mathbf{z}^{(\ell-1)} + \mathbf{b}^{(\ell)}\right),
\quad \ell = 1,\ldots,L{-}1, \\
f_{\boldsymbol{\theta}}(\mathbf{x}) &= \mathbf{W}^{(L)} \mathbf{z}^{(L-1)} + \mathbf{b}^{(L)} \in \mathbb{R}^{\tilde{D}},
\end{aligned}
\label{eq:INR}
\end{equation}
where $\mathbf{x}$ denotes a $D$-dimensional coordinate,
$\gamma(\mathbf{x})$ is the input encoding function, and 
$\rho^{(\ell)}$ represents the nonlinear activation of the $\ell$-th layer. To better capture high-frequency details, the input coordinates are often mapped via Fourier features:
\begin{definition}[Fourier Features~\cite{tancik2020fourier}]\label{definition_RFF}
Fourier features map the input coordinate $\mathbf{x} \in \mathbb{R}^D$ into a high-dimensional vector space and can be defined as
\begin{equation}
\Phi_{\mathrm{FF}}(\mathbf{x}) =
\left[ \sin(2\pi \boldsymbol{\omega}_i^\top \mathbf{x}),\;
      \cos(2\pi \boldsymbol{\omega}_i^\top \mathbf{x}) \right]_{i=1}^M,
\label{eq:RFF_M}
\end{equation}
where $\boldsymbol{\omega}_i \in \mathbb{R}^D$ are frequency vectors determined by different sampling strategies. Two common strategies are:

\begin{itemize}
\item \textbf{Positional Encoding (PE)}: 
Frequencies are set as $\boldsymbol{\omega}_i = b^{i/M}\mathbf{e}_d$, $i=0,\dots,M-1$, $d=1,\dots,D$, 
where $\mathbf{e}_d$ denotes the $d$-th standard basis vector in $\mathbb{R}^D$, and $b$ is a constant.

    \item \textbf{Random Fourier Features (RFF)}: 
    Frequencies are sampled from an isotropic Gaussian, $\boldsymbol{\omega}_i \sim \mathcal{N}(\mathbf{0}, s^2 \mathbf{I}_D)$, 
    where $s$ controls the frequency scale.
\end{itemize}
\end{definition}

Building on this, Fourier features networks are a specific instance of INRs, 
in which the encoding function $\gamma(\mathbf{x})$ is defined as the Fourier features mapping 
$\Phi_{\mathrm{FF}}(\mathbf{x})$ in Definition~\ref{definition_RFF} and the activation $\rho^{(\ell)}$ is ReLU~\cite{nair2010rectified}.

\subsection{Content-Aware Frequency Encoding}\label{High-order Encoding}

To understand the frequency synthesis capability of INRs, we first recall the following theorem.

\begin{theorem}[\cite{yuce2022structured}]
 Let $f_{\boldsymbol{\theta}}: \mathbb{R}^{D} \to \mathbb{R}$ be an INR of the form of Eq.~\eqref{eq:INR} with  
$\rho(z) = \sum_{k=0}^{K} \alpha_k z^k$.  
The network takes an input encoding
$\gamma(\mathbf{x})= \sin(\mathbf{\Omega} \mathbf{x} + \mathbf{\bphi})$,
where $\mathbf{\bOmega} \in \mathbb{R}^{M \times D}$ is the frequency matrix, and 
$\mathbf{\bphi} \in \mathbb{R}^M$ is the phase vector.
This architecture can only represent
functions of the form
\[
f_{\boldsymbol{\theta}}(\mathbf{x}) = \sum_{\bomega' \in \mathcal{H}(\bOmega)} c_{\bomega'} 
\, \sin \!\big(\langle \bomega', \mathbf{x} \rangle + \bphi_{\bomega'} \big),
\]
where $c_{\bomega'} \in \mathbb{R}$ are coefficients determined by the network parameters. The admissible set of frequencies is
\[
\mathcal{H}(\bOmega) \subseteq 
\left\{ \sum_{t=1}^{M} s_{t}\,\bOmega_{t} 
\;\middle|\;
s_{t} \in \mathbb{Z}, \;\; \sum_{t=1}^{M} |s_{t}| \leq K^{L-1} \right\},
\]
with $\bOmega_{t}$ denoting the $t$-th row of $\bOmega$. It should be noted that the activation $\rho(\cdot)$ is not limited to polynomial functions, since analytic activations can be effectively approximated by polynomials or Chebyshev expansions.

\label{theorem1}
\end{theorem}

Theorem~\ref{theorem1} shows that an MLP can synthesize new frequencies as integer combinations of the initially sampled frequencies $\bOmega$, representing the signal as a linear combination of sinusoidal functions.
Even when the initial sampling $\bOmega$ is suboptimal, the MLP should theoretically be able to compose the frequencies required by the target signal, thereby achieving strong representational capacity. 
Furthermore, increasing the number of layers can further expand the range of synthesizable frequencies, potentially enhancing expressive power. 
In practice, however, Fourier features networks are highly sensitive to the choice of initial frequencies~\cite{benbarka2022seeing, hertz2021sape, ma2025robustifying}, and simply increasing the number of layers yields only marginal improvements (see Fig.~\ref{fig:psnr_time_dualaxis}). This reflects the inefficiency and optimization difficulty of relying solely on a deep MLP to implicitly synthesize and select the frequencies required by the target signal.

The above analyses motivate the design of an encoding mechanism that synthesizes frequencies according to the signal content, thus relieving the MLP of frequencies synthesis duties and improving overall efficiency.
A straightforward idea is to increase the number of Fourier features $M$ in Eq.~\eqref{eq:RFF_M} and assign learnable weights to each frequency, allowing the model to select task-relevant frequencies from a broader spectrum. However, this approach is inefficient, since the representable frequency set only grows linearly with $M$. 
A more effective alternative is to introduce multiplicative interactions among frequency components, allowing trigonometric product-to-sum identities to generate a broad set of combinatorial frequencies from a fixed set of bases. 
To this end, we propose \emph{Content-Aware Frequency Encoding (CAFE)}, which leverages learnable linear transformations to drive these frequency interactions, adaptively generating a richer spectrum from fixed Fourier bases.

\begin{figure}[htbp]
    \centering
    \setlength{\tabcolsep}{3pt} %
    \begin{tabular}{cc}
        \includegraphics[width=0.46\columnwidth]{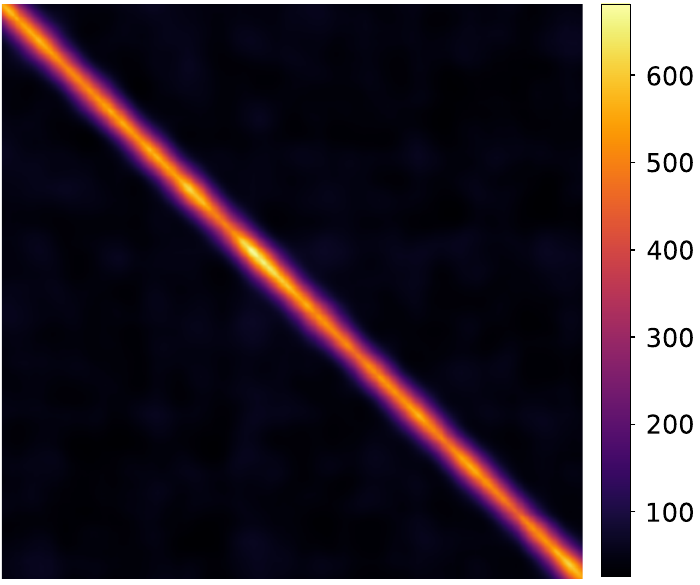} &
        \includegraphics[width=0.46\columnwidth]{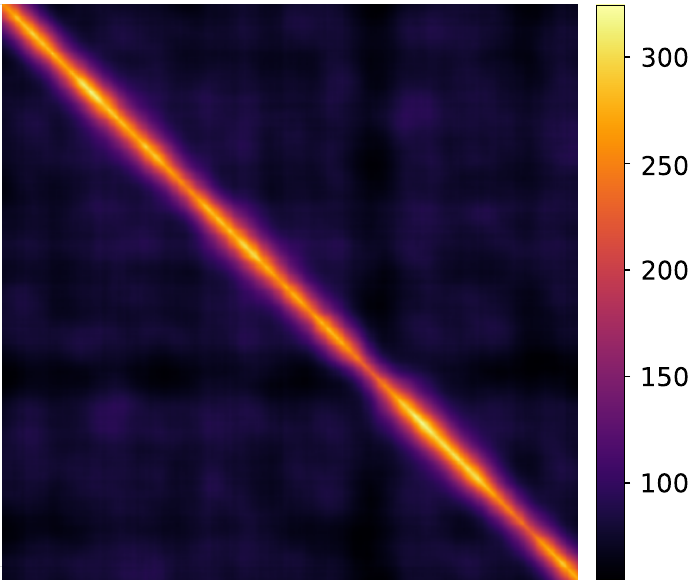} \\
          \hspace{-15pt} (a) CAFE & \hspace{-15pt} (b) RFF
    \end{tabular}
    \caption{NTK matrix for the CAFE and RFF. The network width and the number of sampled points were both set to 1024.}
    \label{fig:NTK}
\end{figure}

Specifically, given an input coordinate $\mathbf{x} \in \mathbb{R}^D$,
we feed it Fourier features $\Phi_{\text{FF}}(x)$ into several parallel linear layers:  
\begin{equation}
H_i(\mathbf{x} ) = \mathbf{W}_i \, \Phi_{\text{FF}}(\mathbf{x} ) + \mathbf{b}_i, \quad i=1,2,\dots,N,
\label{eq:H_i(x)}
\end{equation}
where $\mathbf{W}_i$ and $\mathbf{b}_i$ denote the weight matrix and bias vector of the $i$-th linear layer, respectively. Then, the outputs of all layers are fused via the Hadamard product:
\begin{equation}
\Psi(\mathbf{x}) = \bigodot_{i=1}^{N} H_i(\mathbf{x}),
\label{eq:high-order encoding}
\end{equation}
where $\bigodot$ denotes the Hadamard product.
The resulting $\Psi(\mathbf{x})$ forms the CAFE-encoded features, serving as input to the MLP for signal reconstruction. Moreover, Fig.~\ref{fig:NTK} shows that CAFE produces a better-conditioned NTK matrix, verifying its superiority.

\subsection{Analysis of CAFE}

To gain deeper insight, we further analyze the underlying mechanism behind the improvement brought by CAFE.
Let the input coordinate be $\mathbf{x} \in \mathbb{R}^{D}$, which is encoded by Fourier features as
\begin{equation}
    \mathbf{z} = [\sin(\theta_1), \dots, \sin(\theta_M), \cos(\theta_1), \dots, \cos(\theta_M)]^{\top},
\end{equation}
where $\theta_i = 2\pi\, \boldsymbol{\omega}_i^{\top} \mathbf{x}$.

For simplicity, we consider two parallel linear layers. 
For input $\mathbf{z}$, the $j$-th element of the first projection $(\mathbf{h}^{(1)})_j$ is computed as
\begin{equation}
    (\mathbf{h}^{(1)})_j =
    \sum_{i=1}^{M} w_{1,j,i}^{(s)} \sin(\theta_i)
    + \sum_{i=1}^{M} w_{1,j,i}^{(c)} \cos(\theta_i)
    + b_{1,j},
\end{equation}
where $\mathbf{W}_1 = [\mathbf{W}_1^{(s)}, \mathbf{W}_1^{(c)}] \in \mathbb{R}^{D_h \times 2M}$ 
and $\mathbf{b}_1 \in \mathbb{R}^{D_h}$ are learnable weights and bias, and $w_{1,j,i}^{(s)}$ ($w_{1,j,i}^{(c)}$) connects the $i$-th sine (cosine) features to the $j$-th neuron. 
The second layer $\mathbf{h}^{(2)}$ follows the same formulation.

The Hadamard product $\psi = \mathbf{h}^{(1)} \odot \mathbf{h}^{(2)}$ 
induces frequency mixing through multiplicative interactions between 
the sinusoidal components of $(\mathbf{h}^{(1)})_j$ and $(\mathbf{h}^{(2)})_j$. 
Applying the product-to-sum identities, the multiplicative interactions 
between $(\mathbf{h}^{(1)})_j$ and $(\mathbf{h}^{(2)})_j$ 
give rise to both sum and difference frequency components:
{\small
\begin{align*}
\frac{1}{2}\sum_{i,m}\Big[
    (C_a + C_b)\cos(\theta_i - \theta_m)
    + (C_b - C_a)\cos(\theta_i + \theta_m) \nonumber\\
    + (C_c - C_d)\sin(\theta_i - \theta_m)
    + (C_c + C_d)\sin(\theta_i + \theta_m)
\Big].
\end{align*}
}
The coefficients are defined as
\begin{equation}
\begin{aligned}
    C_a &= w_{1,j,i}^{(s)} w_{2,j,m}^{(s)}, \quad
    C_b = w_{1,j,i}^{(c)} w_{2,j,m}^{(c)}, \\
    C_c &= w_{1,j,i}^{(s)} w_{2,j,m}^{(c)}, \quad
    C_d = w_{1,j,i}^{(c)} w_{2,j,m}^{(s)}.
\end{aligned}
\end{equation}
The linear combinations of $C_a$, $C_b$, $C_c$, and $C_d$ enable the network to selectively enhance or suppress synthesized frequencies.
 For instance, if the target signal does not require the sum-frequency component 
$\bomega_i + \bomega_m$, the network can enforce
\(C_b - C_a = 0, \; C_c + C_d = 0,\)
thereby eliminating the corresponding terms.
A similar analysis applies to the interactions between the bias terms and the sinusoidal components, 
which correspond to frequency components at $\boldsymbol{\omega}_i$ or $\boldsymbol{\omega}_m$ themselves.
The following theorem formalizes the frequency composition of CAFE.
\begin{theorem}\label{theorem:cafe_spectrum}
Let the CAFE encoding $\Psi(\mathbf{x})$ be constructed from a base frequency set 
$\mathcal{F}_{\rm{base}} = \{\boldsymbol{\omega}_1, \dots, \boldsymbol{\omega}_M\}$ 
using $N$ parallel linear layers with bias,
then the admissible frequencies of $\Psi(\mathbf{x})$ form the set
{\small
\begin{align*}
\mathcal{F}_{\rm{CAFE}}^{+}
&= \bigcup_{I=1}^{N} 
\left\{ \sum_{i=1}^{I} \sigma_i \boldsymbol{\omega}_{k_i} 
\;\middle|\;
\boldsymbol{\omega}_{k_i} \in \mathcal{F}_{\rm{base}},\;
\sigma_i \in \{-1, +1\} \right\} \nonumber\\
&= \left\{ \sum_{i=1}^{N} \sigma_i \boldsymbol{\omega}_{k_i }\;\middle|\;
\boldsymbol{\omega}_{k_i} \in \mathcal{F}_{\rm{base}},\;
\sigma_i \in \{-1, 0, +1\} \right\}.
\label{eq:cafe_freq_set}
\end{align*}
}
\end{theorem}
According to Theorem~\ref{theorem:cafe_spectrum}, CAFE can theoretically synthesize $\mathcal{O}(M^N 3^{N-1})$ frequencies.
It further selects task-relevant ones through the learned weights.

\subsection{CAFE\texorpdfstring{+}{+} with Fourier-Chebyshev Features}\label{Hybrid Encoding}
While CAFE significantly enhances the representational capacity of networks, the frequencies it can synthesize still depend on randomly initialized Fourier features.
Since most real-world signals concentrate their energy in low-frequency regions~\cite{field1987relations}, such randomness may fail to adequately cover the essential low-frequency bases, potentially forcing the network to use excessive high-frequency bases to compensate for missing low-frequency structures.
Moreover, Fourier bases are not inherently efficient for modeling smooth low-frequency structures~\cite{boyd2001chebyshev} and may introduce noise in these regions~\cite{landgraf2022pins, yuce2022structured}.
To provide a more stable and expressive low-frequency representation, we introduce the Chebyshev features based on Chebyshev polynomials.

\begin{figure*}[htbp]
    \centering
    \includegraphics[width=0.96\linewidth]{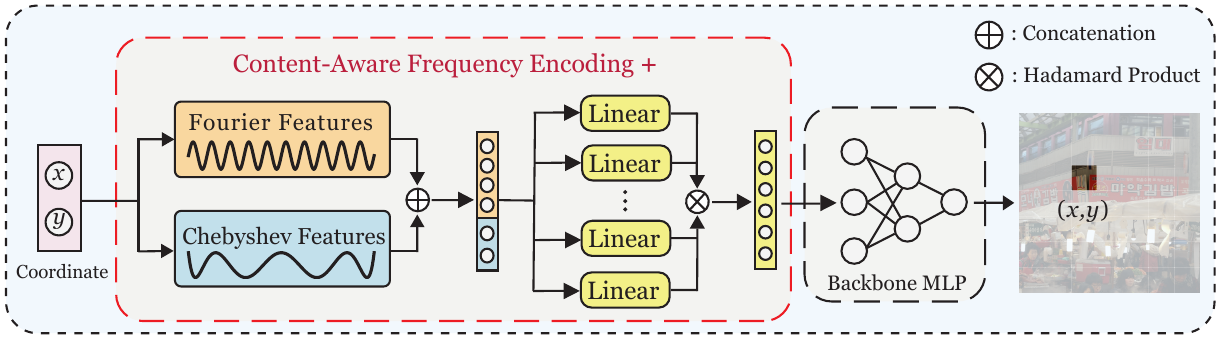}
    \vspace{-0.5em}
    \caption{Overall framework of CAFE+. CAFE+ maps the input coordinates separately to Chebyshev and Fourier features, which are then concatenated to form the inputs of multiple parallel linear layers. The outputs of these layers are combined via the Hadamard product to obtain the CAFE+ encoded features, which are subsequently fed into the backbone MLP to produce the corresponding values.}
    \label{fig:framework}
\end{figure*}

\begin{definition}[Chebyshev Features]\label{definition_CPE}
Given $\mathbf{x}\in\mathbb{R}^D$ and polynomial order $J$, the Chebyshev features mapping is defined as
\begin{equation}
\Phi_{\text{\rm CF}}(\mathbf{x})
=
\big[\,T_j(x_d)\,\big]_{d=1,\dots,D;\,j=0,\dots,J-1}
\;\in \mathbb{R}^{DJ},
\label{def:CPE}
\end{equation}
where $T_j$ denotes the $j$-th Chebyshev polynomial of the first kind, defined recursively as
\[
T_0(x)=1,\; T_1(x)=x,\;
T_{j+2}(x)=2xT_{j+1}(x)-T_{j}(x).
\]
\end{definition}

Chebyshev polynomials offer near-optimal approximation properties for smooth functions,
achieving the smallest maximum error among all polynomial bases of the same degree.
They also exhibit excellent numerical stability due to their orthogonality and bounded oscillation over $[-1, 1]$,
making them particularly suitable for representing low-frequency or smoothly varying structures~\cite{trefethen2013approximation}. These properties make Chebyshev features a natural and effective complement to Fourier features. More importantly, since Chebyshev polynomials also satisfy the product-to-sum identity
\(
T_p(x) T_q(x) = \tfrac{1}{2}\!\left[\,T_{p+q}(x) + T_{|p-q|}(x)\,\right],
\)
the CAFE framework and its analysis can be extended to the Chebyshev features. 
In other words, CAFE can adaptively select from higher-order basis functions synthesized within the Chebyshev domain. 
By incorporating Fourier-Chebyshev features, we obtain an enhanced version of CAFE, denoted as CAFE+, and the overall framework is shown in Fig.~\ref{fig:framework}. 

The complete formulation of CAFE+ is given as follows.
Given a coordinate $\mathbf{x} \in \mathbb{R}^d$, the CAFE with combined features embedding is as 
\begin{equation*}
\Psi(\mathbf{x} ) = \bigodot_{i=1}^{N}\{\mathbf{W}_i\big[\Phi_{\text{FF}}(\mathbf{x}), \; \Phi_{\text{CF}}(\mathbf{x}) \big] +\mathbf{b}_i\},
\label{eq:phi_combined}
\end{equation*}
which is fed into the MLP to produce the predicted value. 

Fig.~\ref{fig:neural_vision} presents the results of masking either the Chebyshev or Fourier features during inference. 
The results demonstrate that the two components play complementary roles in signal representation: 
Chebyshev features provide stable global structures, while Fourier features capture fine high-frequency details. Moreover, we also establish a theorem for the Chebyshev features analogous to Theorem~\ref{theorem1} 
to further complete our theoretical analysis.

\begin{theorem}\label{theorem2}
Let $f_{\boldsymbol{\theta}}: \mathbb{R} \to \mathbb{R}$ be an INR of the form of Eq.~\eqref{eq:INR} with  
$\rho(z) = \sum_{k=0}^{K} \alpha_k z^k$. The encoding is constructed using Chebyshev features as
\[
\gamma(x) = 
\begin{bmatrix}
T_{j_1}(x), T_{j_2}(x), \dots , T_{j_{M}}(x)
\end{bmatrix}^{\!\top}.
\]
Here the indices $j_1,j_2,\dots,j_M$ represent arbitrary non-negative orders of the Chebyshev polynomials.
This architecture can only represent
functions of the form
{
\small
\begin{equation*}
f_{\boldsymbol{\theta}}(x)
= \sum_{j \in \mathcal{H}(J)} c_j T_j(x),
\end{equation*}
}where $c_{j} \in \mathbb{R}$ are coefficients determined by the network parameters and the set of resulting orders $\mathcal{H}(J)$ is given by
{\small
\begin{equation*}
\mathcal{H}(J)
\subseteq
\Bigl\{
\,\bigl|\,\sum_{t=1}^{M} s_t j_t\,\bigr|
~\Bigm|~
s_t \in \mathbb{Z},~
\sum_{t=1}^{M} |s_t| \le K^{L-1}
\Bigr\}.
\end{equation*}
}
\end{theorem}

Theorem~\ref{theorem2} proves that a backbone MLP operating on Chebyshev features
maintains this representation, ensuring that the network can effectively use Chebyshev polynomials to fit the low frequency of signals. The proof is provided in the supplementary material.
\begin{figure}[tbp]
    \centering
    \begin{subfigure}[b]{0.32\columnwidth}
        \includegraphics[width=\linewidth]{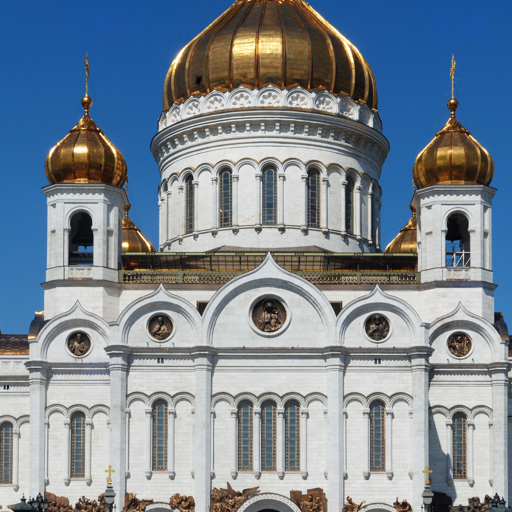}
        \caption{GT}
        \label{fig:vision_gt}
    \end{subfigure}
    \begin{subfigure}[b]{0.32\columnwidth}
        \includegraphics[width=\linewidth]{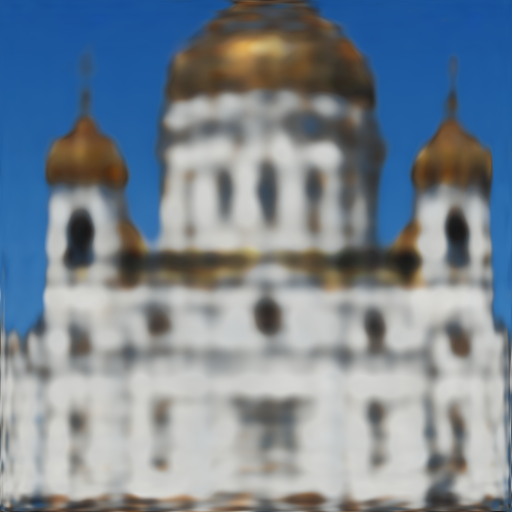}
        \caption{CF-only}
        \label{fig:vision_ce}
    \end{subfigure}
    \begin{subfigure}[b]{0.32\columnwidth}
        \includegraphics[width=\linewidth]{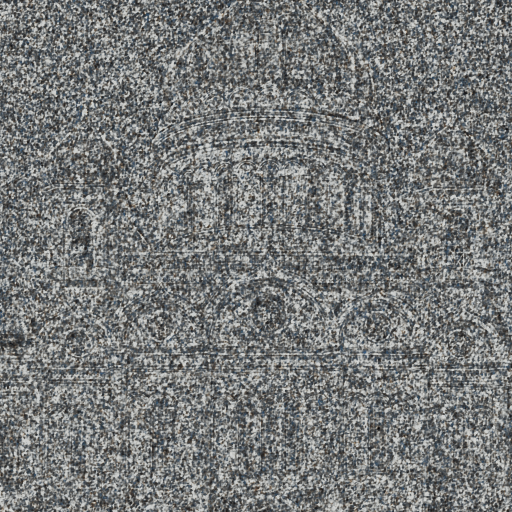}
        \caption{FF-only}
        \label{fig:vision_rff}
    \end{subfigure}
    \vspace{-0.5em}
    \caption{
Illustration of the complementary roles of Chebyshev and Fourier features. The network is first trained with the CAFE+, 
and during inference, we deactivate the neurons corresponding to either 
$\Phi_{\text{CF}}(\mathbf{x})$ or $\Phi_{\text{FF}}(\mathbf{x})$ by setting them to zero 
to isolate their individual effects. 
\textit{CF-only} corresponds to using $[\,\mathbf{0},\, \Phi_{\text{CF}}(\mathbf{x})\,]$
while \textit{FF-only} corresponds to $[\,\Phi_{\text{FF}}(\mathbf{x}),\, \mathbf{0}\,]$. }
\label{fig:neural_vision}
\end{figure}

\section{Experiments}

\begin{figure*}[htbp]
    \resizebox{\textwidth}{!}{
    \centering
    \begin{tabular}{ccccccc} 
        \includegraphics[width=0.16\textwidth]{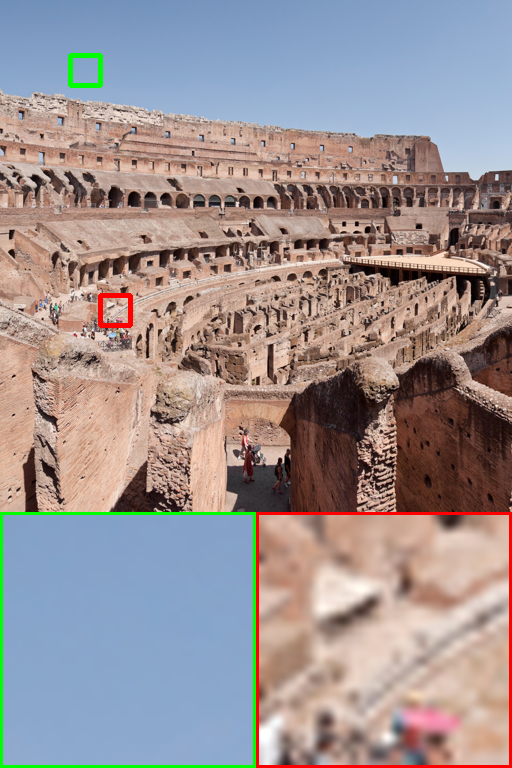} &
        \includegraphics[width=0.16\textwidth]{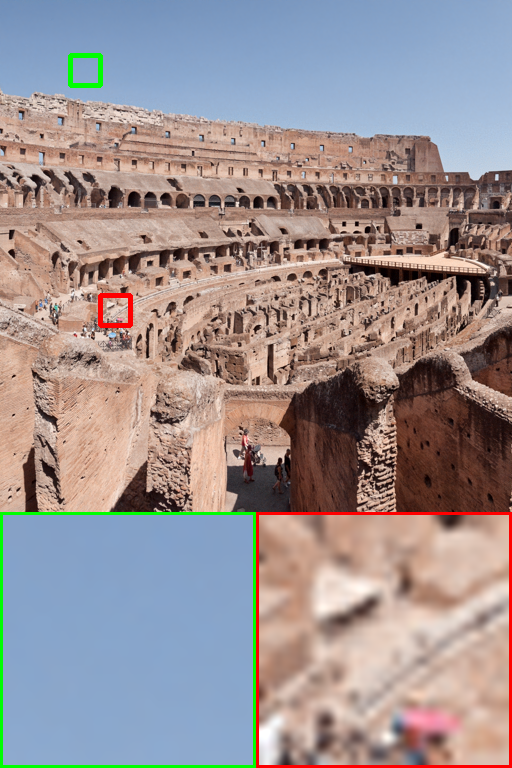} &
        \includegraphics[width=0.16\textwidth]{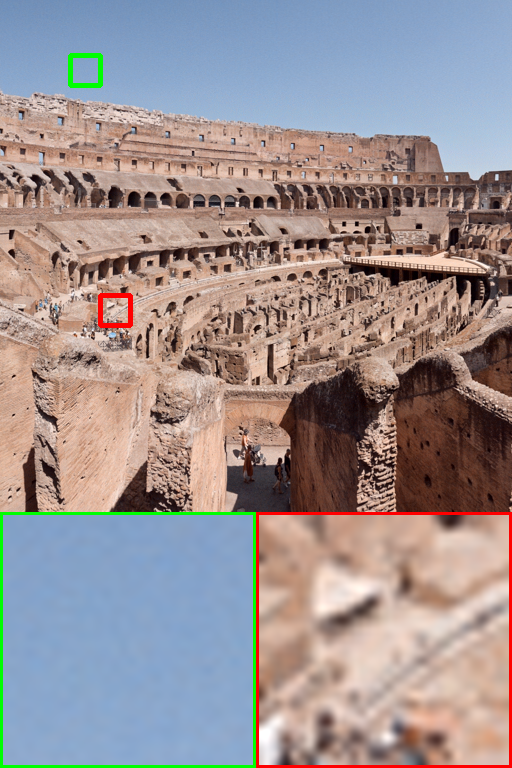} &
        \includegraphics[width=0.16\textwidth]{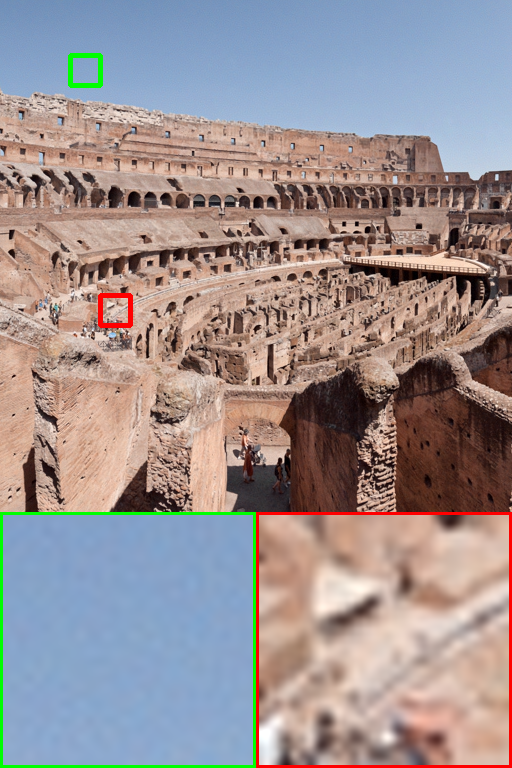} &
        \includegraphics[width=0.16\textwidth]{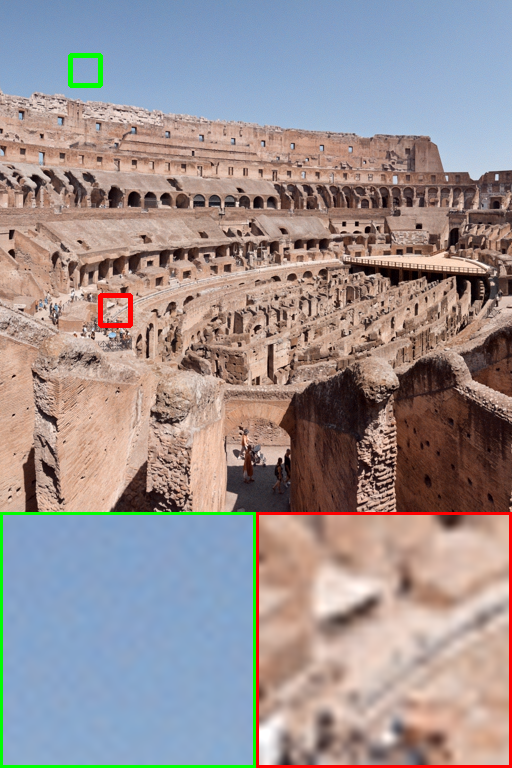} &
        \includegraphics[width=0.16\textwidth]{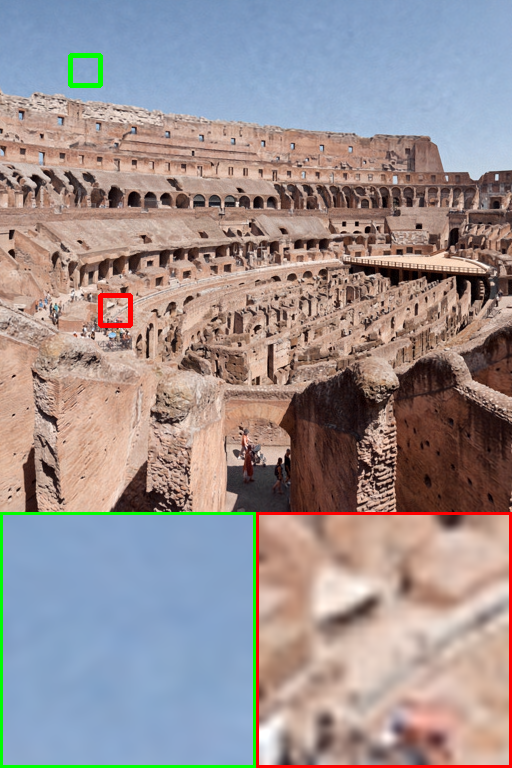} &
        \includegraphics[width=0.16\textwidth]{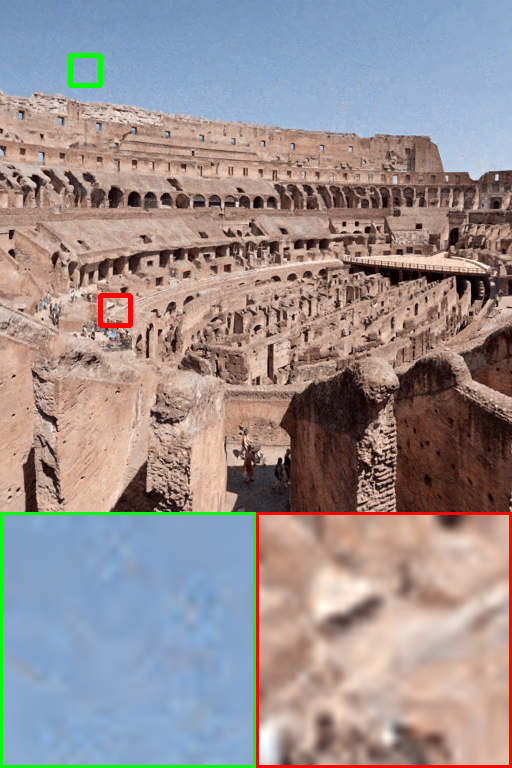} 
        \\
        GT(PSNR) & \textbf{Ours(42.67)} & SL$^2$A(40.85) & FINER(39.70) & SCONE(40.04) & SIREN(36.94) & WIRE(30.67) \\
    \end{tabular}
    }
    \caption{Visualization results of the 2D image fitting task on the D2K4 dataset.}
    \label{fig:2d_img_fiting_vis}
    \vspace{-1em}
\end{figure*}

In this section, we demonstrate the effectiveness of CAFE+ across multiple tasks. To this end, we evaluate our method on the fitting task using several representative approaches, including SIREN~\cite{sitzmann2020implicit}, WIRE~\cite{saragadam2023wire}, FINER~\cite{liu2024finer}, SCONE~\cite{li2024learning}, and SL$^2$A~\cite{rezaeian2025sl2a}. Since SCONE and SL$^2$A have not released their NeRF~\cite{mildenhall2021nerf} implementations, we additionally include GAUSS~\cite{ramasinghe2022beyond} for comparison. All baseline implementations are based on the official code released by the authors. For the fitting task, we employ RFF with a fixed scale parameter of 30, while for NeRF experiments, we use PE. 
All methods are optimized using Adam~\cite{kingma2014adam} with appropriately selected learning rates and scheduling strategies depending on the task. All experiments are conducted on a workstation equipped with an Intel Core i9-12900K CPU, 62 GB RAM, and an NVIDIA RTX 4090 GPU. Additional experiments, implementation details, and visual results are provided in the supplementary material.

\subsection{Main results}
\noindent 
{\bf 2D Image Fitting. }
For the task of 2D image fitting, the goal is to learn a continuous 2D function 
\( f: \mathbb{R}^2 \rightarrow \mathbb{R}^3 \) that maps pixel coordinates to RGB color values.  
In our experiments, we select 8 natural images from the DIV2K dataset~\cite{agustsson2017ntire} provided by the FINER codebase\footnote{\url{https://github.com/liuzhen0212/FINER}}{}, 
and denote them as D2K0--D2K7. 
Both our method and the baselines are trained for 6,000 iterations. We also report a larger variant with training time comparable to the fastest baseline, denoted as Ours (L).
The reconstruction quality is quantitatively evaluated using the PSNR metric~\cite{wang2004image}. In Fig.~\ref{fig:2d_img_fiting_vis}, the fitting results are visualized with magnified views of both low-frequency and high-frequency regions. Compared with the baselines, our method captures high-frequency details more effectively while significantly suppressing noise in low-frequency regions. Quantitative results on the remaining images, as reported in Table~\ref{tab:2d_fitting_comparison}, show that our method achieves the best performance.

\begin{table}[tbp]
\centering
\caption{Quantitative comparison of different methods on the 2D image fitting task in terms of PSNR (dB), parameter count, and training time.}
\scriptsize
\setlength{\tabcolsep}{3pt}
\resizebox{\linewidth}{!}{
\begin{tabular}{lccccccc|cc}
\toprule
\multicolumn{8}{c|}{PSNR (dB)} & Param & Time(s) \\
\midrule
Method & D2K0 & D2K1 & D2K2 & D2K3 & D2K5 & D2K6 & D2K7 &  &  \\
\midrule
SIREN 
& 33.48 
& 40.08 
& 36.68 
& 40.39 
& 34.41 
& 38.78 
& 37.58 
& 0.20M 
& \underline{148.16} \\

WIRE  
& 27.99 
& 34.44 
& 31.12 
& 36.06 
& 28.80 
& 32.50 
& 33.50 
& 0.07M
& 373.60  \\

SCONE 
& 35.32 
& 41.88 
& 38.56 
& \underline{42.27} 
& \underline{36.87} 
& \textbf{41.78} 
& 38.92 
& 0.19M 
& 380.49  \\

FINER 
& 35.93 
& \underline{42.33} 
& 39.44 
& 41.99 
& 36.78 
& 41.08 
& 39.73 
& 0.20M 
& 203.51 \\

SL$^2$A 
& \underline{36.22} 
& 41.91 
& \underline{39.87} 
& 41.87 
& \textbf{37.03} 
& \underline{41.11} 
& \underline{40.23} 
& 0.46M 
& 270.50 \\

Ours
& \textbf{36.92} 
& \textbf{42.57} 
& \textbf{42.13} 
& \textbf{42.92}
& 36.64 
& 41.00 
& \textbf{42.61} 
& 0.22M 
& \textbf{107.83} \\
\midrule
Ours (L)  
& 39.47 
& 44.34 
& 44.18 
& 44.69 
& 39.15 
& 42.63 
& 45.01 
& 0.33M 
& 149.78 \\
\bottomrule
\end{tabular}
}
\label{tab:2d_fitting_comparison}
\end{table}

\noindent 
{\bf 3D Shape Representation. }
We investigate the task of 3D shape modeling~\cite{li2023neuralangelo} based on INRs. 
The goal is to learn a continuous mapping function $f: \mathbb{R}^3 \rightarrow \mathbb{R}^1$, where $f(\mathbf{x})$ outputs the signed distance value for a given 3D point $\mathbf{x}$. 
We select five shapes from a public dataset\footnote{\url{https://graphics.stanford.edu/data/3Dscanrep/}}, including Thai Statue, Lucy, Armadillo, Dragon, and Buddha. 
For training, all voxels are randomly shuffled in each iteration and processed in mini-batches of up to 200,000 points, for a total of 200 iterations. 
Intersection-over-union (IoU) metrics~\cite{mescheder2019occupancy} are computed to quantitatively compare different methods. We visualize the reconstructed 3D shape of the Thai Statue for qualitative analysis (Fig.~\ref{fig:3D_shape}) and summarize the quantitative results in Table~\ref{tab:3d_shape_comparison}, showing strong performance across all compared approaches.

\begin{table}[tbp]
\centering
\caption{Quantitative comparison of different methods on the 3D shape representation task in terms of IoU (\%) and training time.}
\setlength{\tabcolsep}{5pt}
\resizebox{\linewidth}{!}{
\begin{tabular}{lccccc|c}
\toprule
\multicolumn{6}{c|}{IoU (\%)} & Time (s) \\
\midrule
Method & Thai statue & Lucy & Armadillo & Dragon & Buddha &  \\
\midrule
SIREN  & 0.9797 & 0.9866 & 0.9933 & 0.9940 & 0.9934 & \textbf{826} \\
WIRE   & 0.9887 & 0.9930 & 0.9861 & 0.9934 & 0.9945 & 1,824 \\
SCONE  & 0.9875 & 0.9905 & 0.9930 & 0.9918 & 0.9924 & 1,662 \\
FINER  & 0.9893 & 0.9929 & 0.9950 & 0.9959 & 0.9958 & 1,056 \\
SL$^2$A  & \underline{0.9987} & \underline{0.9992} & \underline{0.9993} & \underline{0.9991} & \underline{0.9994} & 1,282 \\
Ours   & \textbf{0.9992} & \textbf{0.9995} & \textbf{0.9996} & \textbf{0.9994} & \textbf{0.9995} & \underline{860} \\
\bottomrule
\end{tabular}
}
\label{tab:3d_shape_comparison}
\vspace{-1em}
\end{table}

{
\setlength{\tabcolsep}{1pt}
\begin{figure*}[htbp]
    \centering
    \resizebox{\linewidth}{!}{
    \begin{tabular}{ccccccc} %
        \includegraphics[width=0.16\textwidth]{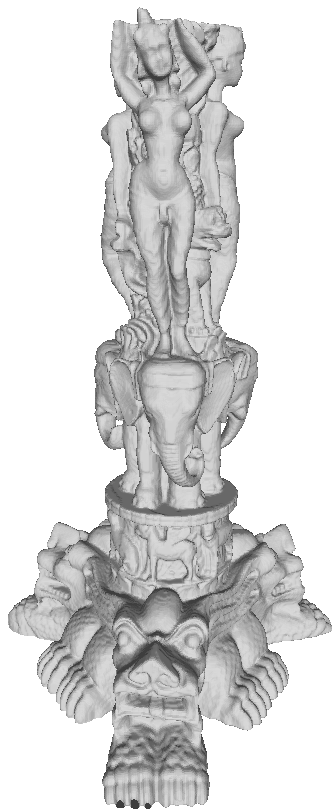} &
        \includegraphics[width=0.16\textwidth]{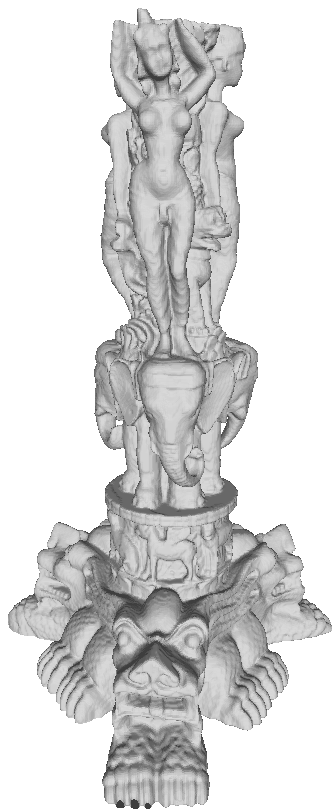} &
        \includegraphics[width=0.16\textwidth]{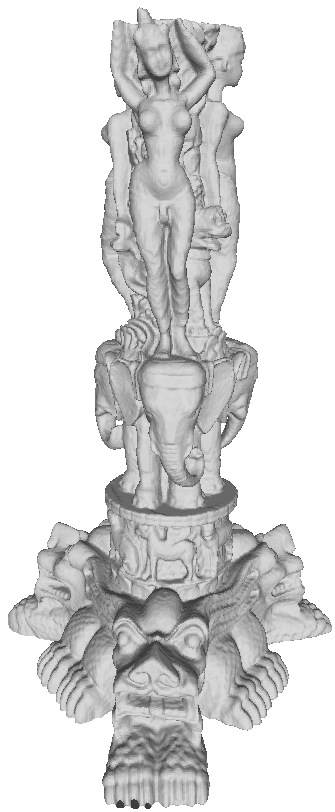} &
        \includegraphics[width=0.16\textwidth]{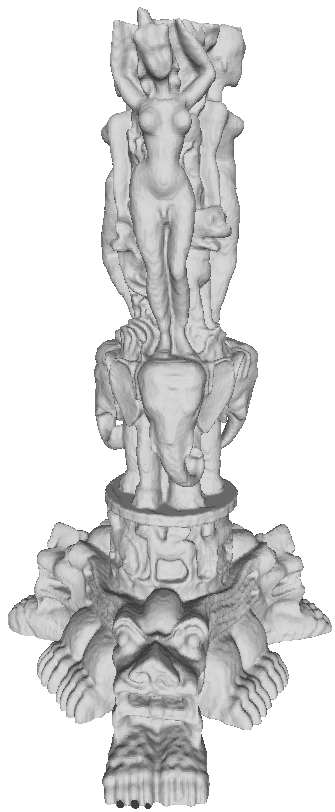} &
        \includegraphics[width=0.16\textwidth]{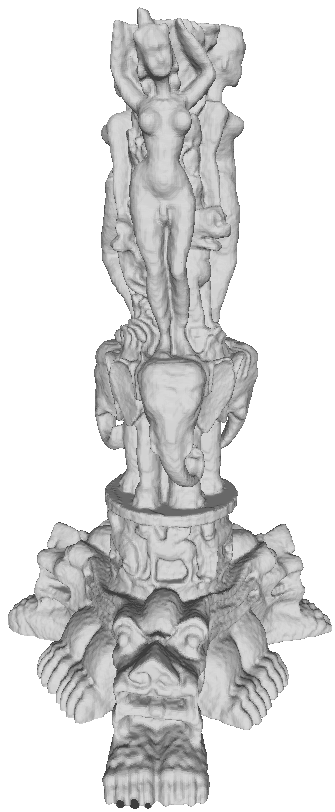} &
        \includegraphics[width=0.16\textwidth]{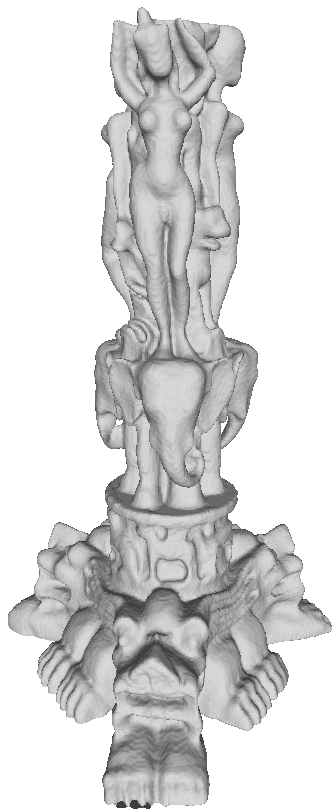} &
        \includegraphics[width=0.16\textwidth]{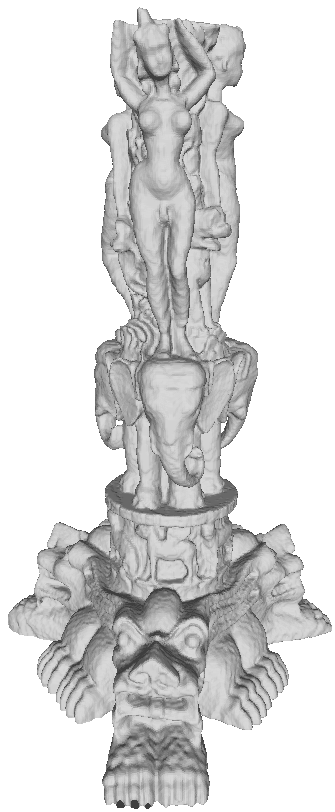} \\
        GT(IOU) & \textbf{Ours(0.9992)} & SL$^2$A(0.9987) & FINER(0.9893) & SCONE(0.9875) & SIREN(0.9797) & WIRE(0.9887) \\
    \end{tabular}
    }
    \caption{Visualization of the 3D shape representation task.
}
\label{fig:3D_shape}
\vspace{-1em}
\end{figure*}
}

\noindent 
{\bf Learning Neural Radiance Fields. }
In novel view synthesis using INRs, each query consists of a 3D spatial coordinate $\mathbf{x} \in \mathbb{R}^3$ and a 2D viewing direction $\mathbf{d} \in \mathbb{R}^2$, which together form a 5D input vector. The network learns a mapping $f: \mathbb{R}^5 \rightarrow \mathbb{R}^4$ that outputs the corresponding RGB color $\mathbf{c} \in \mathbb{R}^3$ and density $\sigma \in \mathbb{R}$ at each point. The final image is obtained by integrating these predicted values along camera rays via differentiable volume rendering~\cite{mildenhall2021nerf}. All methods follow the experimental settings of FINER~\cite{liu2024finer}, using only 25 images for training and employing mixed-precision training.
 We evaluate our framework on the Blender dataset~\cite{mildenhall2021nerf}, specifically on the \textit{Ship}, \textit{Lego}, \textit{Hotdog} and \textit{Drums} scenes. 
Quantitative results are reported in Table~\ref{tab:5d_nerf}, where our method achieves the best performance on three out of four datasets and delivers comparable results on the Drums scene.
We visualize the reconstructed results on the Lego scene in Fig.~\ref{vis:nerf}, where our method shows clear advantages in preserving high-frequency details compared to prior approaches.

\begin{table}[tbp]
\centering
\small
\caption{Quantitative comparison of different methods on the NeRF task in terms of PSNR (dB) and training time.}
\scriptsize
\setlength{\tabcolsep}{5pt}
\resizebox{\linewidth}{!}{
\begin{tabular}{lcccc|c}
\toprule
\multicolumn{5}{c|}{PSNR (dB)} & Time (s) \\
\midrule
Method & Ship & Drums & Lego & Hotdog &  \\
\midrule
Gauss  & 21.24 & 22.46 & 26.40 & 31.28 & 685.2 \\
SIREN  & 21.79 & 25.55 & 29.94 & \underline{33.91} & \underline{403.5} \\
WIRE   & 21.33 & 24.57 & 29.22 & 32.78 & 973.8 \\
FINER  & \underline{22.02} & \textbf{25.66} & \underline{30.78} & 33.80 & 437.0 \\
Ours   & \textbf{23.12} & \underline{25.59} & \textbf{31.86} & \textbf{34.65} & \textbf{363.2} \\
\bottomrule
\end{tabular}
}
\label{tab:5d_nerf}
\end{table}

\begin{table}[t]
\centering
\small
\caption{Ablation study on the effect of CAFE and Chebyshev features components. 
PSNR (dB) is reported on five 2D test images.}
\label{tab:ablation_hybrid_and_high-order}
\setlength{\tabcolsep}{4pt} %
\begin{tabular}{cc|S[table-format=2.2] S[table-format=2.2] S[table-format=2.2] S[table-format=2.2] S[table-format=2.2]}
\toprule
\textbf{CAFE} & \textbf{Chebyshev} & {D2K0} & {D2K1} & {D2K2} & {D2K3} & {D2K4} \\
\midrule
\ding{55}  & \ding{55}  & 26.18 & 30.19 & 30.98 & 29.46 & 30.92 \\
\ding{55}  & \ding{51}  & 33.00 & 37.23 & 35.60 & 36.97 & 35.91 \\
\ding{51}  & \ding{55}  & 34.87 & 37.36 & 40.42 & 40.15 & 40.52 \\
\ding{51}  & \ding{51}  & \textbf{39.47} & \textbf{44.34} & \textbf{44.18} & \textbf{44.69} & \textbf{44.97} \\
\bottomrule
\end{tabular}
\end{table}

\subsection{Ablation study}

\noindent 
{\bf Effect of CAFE and Chebyshev Features. }
We conduct comparative experiments to analyze the effectiveness of CAFE and the Chebyshev features, with all settings trained under identical parameter counts and configurations.
In Table~\ref{tab:ablation_hybrid_and_high-order}, the standard RFF baseline corresponds to the absence of both components, while the configuration using only Chebyshev features feeds Fourier-Chebyshev encodings into the MLP.
The results clearly demonstrate the crucial roles of both CAFE and the Chebyshev features in enhancing representation quality.

\begin{figure}[t]
    \centering
    \begin{tikzpicture}
        \begin{axis}[
            xlabel={Number of Parallel Linear Layers},
            ylabel={PSNR (dB)},
            ymin=40, ymax=52,
            xtick={2,3,4,5,6,7,8},
            ytick={40,42,44,46,48,50,52},
            axis y line*=left,
            axis x line*=bottom,
            width=0.65\linewidth,
            height=0.42\linewidth,
            scale only axis,
            grid=major,
            grid style={dashed, gray!15},
            tick label style={font=\footnotesize},
            label style={font=\footnotesize},
            minor tick num=1,
            axis line style={semithick, blue!60!black},
            legend style={
                at={(0.5,1.05)},   
                anchor=south,
                legend columns=2,
                column sep=0.5em,    
                font=\footnotesize,
                draw=none,
                fill=white,
                inner sep=3pt,
                rounded corners=3pt
            }
        ]
            \addplot[
                color=blue!70!black, semithick, 
                mark=*, mark size=2pt, mark options={fill=blue!70!black, opacity=0.8}
            ] coordinates {
                (2,41.44) (3,44.349) (4,47.09) (5,49.38) (6,51.32) (7,51.42) (8,51.37)
            };
            \addlegendentry{PSNR (dB)}

            \addlegendimage{
                color=orange!70!black, semithick, 
                mark=square*, mark size=2pt, mark options={fill=orange!70!black, opacity=0.8}
            }
            \addlegendentry{Time (s)}
        \end{axis}

        \begin{axis}[
            ylabel={Time (s)},
            ymin=120, ymax=320,
            ytick={120,160,200,240,280,320},
            axis y line*=right,
            axis x line=none,
            width=0.65\linewidth,
            height=0.42\linewidth,
            scale only axis,
            tick label style={font=\footnotesize},
            label style={font=\footnotesize},
            axis line style={semithick, orange!70!black},
        ]
            \addplot[
                color=orange!70!black, semithick, 
                mark=square*, mark size=2pt, mark options={fill=orange!70!black, opacity=0.8}
            ] coordinates {
                (2,127.61) (3,149.78) (4,185.51) (5,214.53) (6,244.05) (7,273.89) (8,302.57)
            };
        \end{axis}
    \end{tikzpicture}
    \caption{Effect of the number of parallel linear layers on PSNR and training time for the 2D image fitting task.}
    \label{fig:order_psnr_time}
\end{figure}
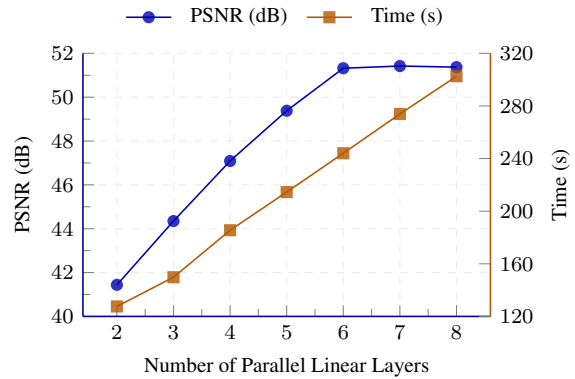

{
\setlength{\tabcolsep}{1pt}
\begin{figure*}[htbp]
    \centering
    \resizebox{\linewidth}{!}{
    \begin{tabular}{cccccc}

        \includegraphics[width=0.15\textwidth]{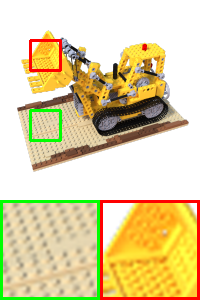} &
        \includegraphics[width=0.15\textwidth]{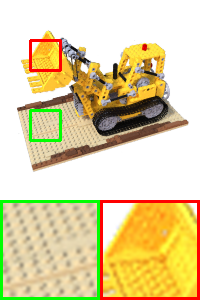} &
        \includegraphics[width=0.15\textwidth]{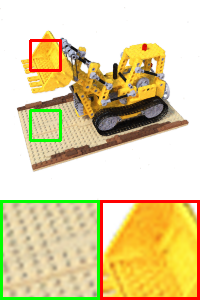} &
        \includegraphics[width=0.15\textwidth]{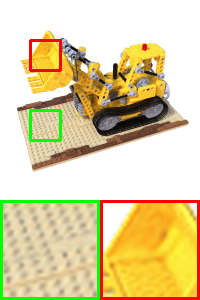} &
        \includegraphics[width=0.15\textwidth]{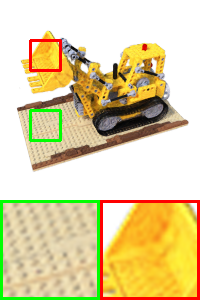} &
        \includegraphics[width=0.15\textwidth]{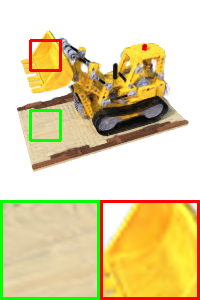} \\
        GT(PSNR) & \textbf{Ours(31.86)} & FINER(30.78) & SIREN(29.94) & WIRE(29.22) & Gauss(26.40) \\
    \end{tabular}
    }
    
    \caption{Visualization of the neural radiance fields task.
    \label{vis:nerf}
}
\end{figure*}
}


\noindent 
{\bf Effect of the Number of Parallel Linear Layers. }
We conduct experiments on the DIV1 dataset to investigate the relationship between the number of parallel linear layers and the fitting capability (see Fig.~\ref{fig:order_psnr_time}). The results show that as the number of layers increases, the performance of our method steadily improves, and gradually saturates when the layer count becomes sufficiently large.
 This phenomenon can be attributed to the fact that increasing the number of layers introduces a richer set of basis functions, thereby enhancing the fitting performance. However, once the provided basis functions sufficiently cover the required frequency range, further increasing the number of layers yields little additional gain. Moreover, the computational cost grows approximately linearly with order, indicating that the CAFE+ introduces only a modest and predictable overhead.

\noindent 
{\bf Hyperparameter Sensitivity Analysis. }
In Fig.~\ref{fig:psnr_lines}, we analyze the influence of two key hyperparameters in CAFE+: the Fourier features scale $s$ in the RFF, and the polynomial order $J$ of the Chebyshev features in Definition~\ref{definition_CPE}.
For the scale analysis, we fix the polynomial order at $J=32$. Our framework achieves strong performance when the scale lies between 20 and 50, and even when the scale is set to slightly lower or higher values, the performance remains within an acceptable range, demonstrating notable robustness.
For the order analysis, we fix the scale at $s=30$ and vary $J$ from 0 to 64. The results show that once $J$ exceeds 16, the model already provides a stable set of low-frequency bases and exhibits strong robustness to further increases in order.

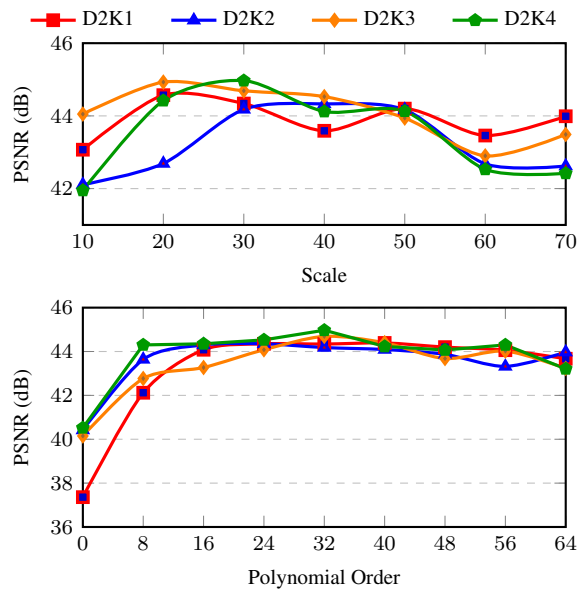
\begin{figure}[!t]
\centering
\begin{tikzpicture}
\begin{axis}[
    hide axis,                    %
    xmin=0, xmax=1, ymin=0, ymax=1,
    legend columns=4,             %
    legend style={
        at={(0.5,1.1)},           %
        anchor=south,             %
        font=\footnotesize,
        draw=none,
        /tikz/every even column/.append style={column sep=0.4cm
        }
    }
]

\addlegendimage{color=red, mark=square*, line width=1.2pt}
\addlegendentry{D2K1}
\addlegendimage{color=blue, mark=triangle*, line width=1.2pt}
\addlegendentry{D2K2}
\addlegendimage{color=orange, mark=diamond*, line width=1.2pt}
\addlegendentry{D2K3}
\addlegendimage{color=green!60!black, mark=pentagon*, line width=1.2pt}
\addlegendentry{D2K4}

\end{axis}
\end{tikzpicture}

\vspace{-0.2cm} 
\begin{tikzpicture}
 
\begin{axis}[
    width=0.96\linewidth,
    height=4.0cm,
    tick label style={font=\footnotesize},
    label style={font=\footnotesize},
    xlabel={Scale},
    ylabel={PSNR (dB)},
    xmin=10, xmax=70,
    ymin=41, ymax=46,
    xtick={10,20,30,40,50,60,70},
    ymajorgrids=true,
    grid style=dashed,
    thick
]

\addplot+[color=red, mark=square*, line width=1.2pt, smooth] coordinates {
    (10,43.07) (20,44.57) (30,44.34) (40,43.59) (50,44.20) (60,43.46) (70,43.99)
};

\addplot+[color=blue, mark=triangle*, line width=1.2pt, smooth] coordinates {
    (10,42.11) (20,42.69) (30,44.18) (40,44.33) (50,44.17) (60,42.68) (70,42.62)
};

\addplot+[color=orange, mark=diamond*, line width=1.2pt, smooth] coordinates {
    (10,44.05) (20,44.93) (30,44.69) (40,44.53) (50,43.94) (60,42.90) (70,43.49)
};

\addplot+[color=green!60!black, mark=pentagon*, line width=1.2pt, smooth] coordinates {
    (10,41.95) (20,44.42) (30,44.97) (40,44.12) (50,44.13) (60,42.53) (70,42.42)
};

\end{axis}

\begin{axis}[
    at={(0,-1.1cm)},  
    anchor=north west,
    width=0.96\linewidth,
    height=4.5cm,
    tick label style={font=\footnotesize},
    label style={font=\footnotesize},
    xlabel={Polynomial Order},
    ylabel={PSNR (dB)},
    xmin=0, xmax=64,
    ymin=36, ymax=46,
    xtick={0,8,16,24,32,40,48,56,64},
    ymajorgrids=true,
    grid style=dashed,
    thick
]

\addplot+[color=red, mark=square*, line width=1.2pt, smooth] coordinates {
    (0,37.36) (8,42.12) (16,44.08) (24,44.35) (32,44.34)
    (40,44.39) (48,44.20) (56,44.06) (64,43.67)
};

\addplot+[color=blue, mark=triangle*, line width=1.2pt, smooth] coordinates {
    (0,40.42) (8,43.64) (16,44.30) (24,44.35) (32,44.18)
    (40,44.09) (48,43.87) (56,43.32) (64,43.97)
};

\addplot+[color=orange, mark=diamond*, line width=1.2pt, smooth] coordinates {
    (0,40.15) (8,42.77) (16,43.27) (24,44.08) (32,44.69)
    (40,44.4) (48,43.68) (56,44.02) (64,43.23)
};

\addplot+[color=green!60!black, mark=pentagon*, line width=1.2pt] coordinates {
    (0,40.52) (8,44.30) (16,44.35) (24,44.53) (32,44.97)
    (40,44.23) (48,44.07) (56,44.3) (64,43.21)
};

\end{axis}

\end{tikzpicture}

\caption{PSNR variations across different Fourier features scales (top) and polynomial orders in the Chebyshev features (bottom).}
\label{fig:psnr_lines}
\end{figure}

\noindent 
{\bf Effect of Backbone MLP Depth. }
We analyze the effect of the backbone MLP depth on fitting performance. Table~\ref{tab:mlp_depth} presents the corresponding quantitative results. Notably, removing the MLP entirely leads to a substantial performance drop, highlighting the indispensable role of the backbone MLP in our framework. While a two-layer MLP outperforms a single-layer one, further increasing the depth provides only marginal gains.
This observation indicates that our encoding effectively externalizes most of the frequency composition process, thereby relieving the MLP from the burden of synthesizing target frequencies. 

\begin{table}[t]
\centering
\caption{Ablation study on the effect of backbone MLP depth. 
PSNR (dB) is reported for three test images.}
\label{tab:mlp_depth}
\small
\setlength{\tabcolsep}{8pt}
\begin{tabular}{c|ccc|c}
\toprule
MLP Layers & D2K5 & D2K6 & D2K7 & Avg. \\
\midrule
0 & 31.23 & 36.78 & 38.66 & 35.56 \\
1 & 37.78 & 42.31 & 42.99 & 41.03 \\
2 & 39.15 & 42.63 & 45.01 & 42.26 \\
3 & 39.36 & 42.49 & 44.67 & 42.17 \\
4 & 39.60 & 42.56 & 45.40 & 42.52 \\
\bottomrule
\end{tabular}

\end{table}

\section{Conclusion}

In this paper, we address the limitations of fixed stochastic Fourier bases and the inefficient implicit frequency composition in MLPs.
We propose CAFE, a novel encoding framework that performs content-aware frequency synthesis, greatly relieving the MLP from the burden of inefficient frequency generation.
Furthermore, we introduce Chebyshev features into CAFE to better model low-frequency components and enhance overall stability.
Extensive experiments demonstrate that our framework consistently outperforms existing methods.
In the future, we will explore integrating advanced activation functions with our framework and extend CAFE+ to a wider range of applications.

{
    \small
    \bibliographystyle{ieeenat_fullname}
    \bibliography{ref}

@inproceedings{chen2021learning,
  title={Learning continuous image representation with local implicit image function},
  author={Chen, Yinbo and Liu, Sifei and Wang, Xiaolong},
  booktitle={Proceedings of the Computer Vision and Pattern Recognition Conference},
  pages={8628--8638},
  year={2021}
}

@inproceedings{jiang2025hiif,
  title={HIIF: Hierarchical encoding based implicit image function for continuous super-resolution},
  author={Jiang, Yuxuan and Kwan, Ho Man and Peng, Tianhao and Gao, Ge and Zhang, Fan and Zhu, Xiaoqing and Sole, Joel and Bull, David},
  booktitle={Proceedings of the Computer Vision and Pattern Recognition Conference},
  pages={2289--2299},
  year={2025}
}

@inproceedings{strumpler2022implicit,
  title={Implicit neural representations for image compression},
  author={Str{\"u}mpler, Yannick and Postels, Janis and Yang, Ren and Gool, Luc Van and Tombari, Federico},
  booktitle={European Conference on Computer Vision},
  pages={74--91},
  year={2022},
  organization={Springer}
}

@inproceedings{han2025towards,
  title={Towards Lossless Implicit Neural Representation via Bit Plane Decomposition},
  author={Han, Woo Kyoung and Lee, Byeonghun and Cho, Hyunmin and Im, Sunghoon and Jin, Kyong Hwan},
  booktitle={Proceedings of the Computer Vision and Pattern Recognition Conference},
  pages={2269--2278},
  year={2025}
}

@inproceedings{jayasundara2025sinr,
  title={SINR: Sparsity Driven Compressed Implicit Neural Representations},
  author={Jayasundara, Dhananjaya and Rajagopalan, Sudarshan and Ranasinghe, Yasiru and Tran, Trac D and Patel, Vishal M},
  booktitle={Proceedings of the Computer Vision and Pattern Recognition Conference},
  pages={3061--3070},
  year={2025}
}

@ARTICLE{9606601,
  author={Sun, Yu and Liu, Jiaming and Xie, Mingyang and Wohlberg, Brendt and Kamilov, Ulugbek S.},
  journal={IEEE Transactions on Computational Imaging}, 
  title={CoIL: Coordinate-Based Internal Learning for Tomographic Imaging}, 
  year={2021},
  volume={7},
  number={},
  pages={1400-1412},
  doi={10.1109/TCI.2021.3125564}}

@inproceedings{najaf2024towards,
  title={Towards a sampling theory for implicit neural representations},
  author={Najaf, Mahrokh and Ongie, Gregory},
  booktitle={2024 58th Asilomar Conference on Signals, Systems, and Computers},
  pages={105--110},
  year={2024},
  organization={IEEE}
}

@article{mildenhall2021nerf,
  title={Nerf: Representing scenes as neural radiance fields for view synthesis},
  author={Mildenhall, Ben and Srinivasan, Pratul P and Tancik, Matthew and Barron, Jonathan T and Ramamoorthi, Ravi and Ng, Ren},
  journal={Communications of the ACM},
  volume={65},
  number={1},
  pages={99--106},
  year={2021},
  publisher={ACM New York, NY, USA}
}

@inproceedings{rahaman2019spectral,
  title={On the spectral bias of neural networks},
  author={Rahaman, Nasim and Baratin, Aristide and Arpit, Devansh and Draxler, Felix and Lin, Min and Hamprecht, Fred and Bengio, Yoshua and Courville, Aaron},
  booktitle={International Conference on Machine Learning},
  pages={5301--5310},
  year={2019},
}

@inproceedings{yuce2022structured,
  title={A structured dictionary perspective on implicit neural representations},
  author={Y{\"u}ce, Gizem and Ortiz-Jim{\'e}nez, Guillermo and Besbinar, Beril and Frossard, Pascal},
  booktitle={Proceedings of the IEEE/CVF Conference on Computer Vision and Pattern Recognition},
  pages={19228--19238},
  year={2022}
}

@article{sitzmann2020implicit,
  title={Implicit neural representations with periodic activation functions},
  author={Sitzmann, Vincent and Martel, Julien and Bergman, Alexander and Lindell, David and Wetzstein, Gordon},
  journal={Advances in Neural Information Processing Systems},
  volume={33},
  pages={7462--7473},
  year={2020}
}

@inproceedings{ramasinghe2022beyond,
  title={Beyond periodicity: Towards a unifying framework for activations in coordinate-mlps},
  author={Ramasinghe, Sameera and Lucey, Simon},
  booktitle={European Conference on Computer Vision},
  pages={142--158},
  year={2022},
  organization={Springer}
}

@inproceedings{saragadam2023wire,
  title={Wire: Wavelet implicit neural representations},
  author={Saragadam, Vishwanath and LeJeune, Daniel and Tan, Jasper and Balakrishnan, Guha and Veeraraghavan, Ashok and Baraniuk, Richard G},
  booktitle={Proceedings of the IEEE/CVF Conference on Computer Vision and Pattern Recognition},
  pages={18507--18516},
  year={2023}
}

@inproceedings{liu2024finer,
  title={Finer: Flexible spectral-bias tuning in implicit neural representation by variable-periodic activation functions},
  author={Liu, Zhen and Zhu, Hao and Zhang, Qi and Fu, Jingde and Deng, Weibing and Ma, Zhan and Guo, Yanwen and Cao, Xun},
  booktitle={Proceedings of the IEEE/CVF Conference on Computer Vision and Pattern Recognition},
  pages={2713--2722},
  year={2024}
}

@inproceedings{jayasundara2025pin,
  title={PIN: Prolate spheroidal wave function-based implicit neural representations},
  author={Jayasundara, Dhananjaya and Zhao, Heng and Labate, Demetrio and Patel, Vishal M},
  booktitle={The Thirteenth International Conference on Learning Representations},
  year={2025}
}

@article{tancik2020fourier,
  title={Fourier features let networks learn high frequency functions in low dimensional domains},
  author={Tancik, Matthew and Srinivasan, Pratul and Mildenhall, Ben and Fridovich-Keil, Sara and Raghavan, Nithin and Singhal, Utkarsh and Ramamoorthi, Ravi and Barron, Jonathan and Ng, Ren},
  journal={Advances in Neural Information Processing Systems},
  volume={33},
  pages={7537--7547},
  year={2020}
}

@inproceedings{lindell2022bacon,
  title={Bacon: Band-limited coordinate networks for multiscale scene representation},
  author={Lindell, David B and Van Veen, Dave and Park, Jeong Joon and Wetzstein, Gordon},
  booktitle={Proceedings of the IEEE/CVF Conference on Computer Vision and Pattern Recognition},
  pages={16252--16262},
  year={2022}
}

@article{martel2021acorn,
  title={Acorn: Adaptive coordinate networks for neural scene representation},
  author={Martel, Julien NP and Lindell, David B and Lin, Connor Z and Chan, Eric R and Monteiro, Marco and Wetzstein, Gordon},
  journal={arXiv preprint arXiv:2105.02788},
  year={2021}
}

@inproceedings{li2024learning,
  title={Learning spatially collaged fourier bases for implicit neural representation},
  author={Li, Jason Chun Lok and Liu, Chang and Huang, Binxiao and Wong, Ngai},
  booktitle={Proceedings of the AAAI Conference on Artificial Intelligence},
  volume={38},
  number={12},
  pages={13492--13499},
  year={2024}
}

@article{ma2025robustifying,
  title={Robustifying Fourier Features Embeddings for Implicit Neural Representations},
  author={Ma, Mingze and Zhu, Qingtian and Zhan, Yifan and Yin, Zhengwei and Wang, Hongjun and Zheng, Yinqiang},
  journal={arXiv preprint arXiv:2502.05482},
  year={2025}
}

@inproceedings{choy20163d,
  title={3d-r2n2: A unified approach for single and multi-view 3d object reconstruction},
  author={Choy, Christopher B and Xu, Danfei and Gwak, JunYoung and Chen, Kevin and Savarese, Silvio},
  booktitle={European Conference on Computer Vision},
  pages={628--644},
  year={2016},
  organization={Springer}
}

@inproceedings{riegler2017octnet,
  title={Octnet: Learning deep 3d representations at high resolutions},
  author={Riegler, Gernot and Osman Ulusoy, Ali and Geiger, Andreas},
  booktitle={Proceedings of the IEEE Conference on Computer Vision and Pattern Recognition},
  pages={3577--3586},
  year={2017}
}

@inproceedings{cao2023ciaosr,
  title={Ciaosr: Continuous implicit attention-in-attention network for arbitrary-scale image super-resolution},
  author={Cao, Jiezhang and Wang, Qin and Xian, Yongqin and Li, Yawei and Ni, Bingbing and Pi, Zhiming and Zhang, Kai and Zhang, Yulun and Timofte, Radu and Van Gool, Luc},
  booktitle={Proceedings of the IEEE/CVF Conference on Computer Vision and Pattern Recognition},
  pages={1796--1807},
  year={2023}
}

@inproceedings{li2023neuralangelo,
  title={Neuralangelo: High-fidelity neural surface reconstruction},
  author={Li, Zhaoshuo and M{\"u}ller, Thomas and Evans, Alex and Taylor, Russell H and Unberath, Mathias and Liu, Ming-Yu and Lin, Chen-Hsuan},
  booktitle={Proceedings of the IEEE/CVF Conference on Computer Vision and Pattern Recognition},
  pages={8456--8465},
  year={2023}
}

@inproceedings{kingma2014adam,
  title={Adam: A method for stochastic optimization},
  author={Kingma, Diederik P},
  booktitle ={International Conference on Learning Representations},
  year={2015}
}

@inproceedings{nair2010rectified,
  title={Rectified linear units improve restricted boltzmann machines},
  author={Nair, Vinod and Hinton, Geoffrey E},
  booktitle={Proceedings of the 27th International Conference on Machine Learning (ICML-10)},
  pages={807--814},
  year={2010}
}

@article{muller2022instant,
  title={Instant neural graphics primitives with a multiresolution hash encoding},
  author={M{\"u}ller, Thomas and Evans, Alex and Schied, Christoph and Keller, Alexander},
  journal={ACM Transactions On Graphics (TOG)},
  volume={41},
  number={4},
  pages={1--15},
  year={2022},
  publisher={ACM New York, NY, USA}
}

@inproceedings{zhao2025adaptive,
  title={Adaptive Wavelet-Positional Encoding for High-Frequency Information Learning in Implicit Neural Representation},
  author={Zhao, Hongxu and Gao, Zelin and Wang, Yue and Xiong, Rong and Zhang, Yu},
  booktitle={Proceedings of the AAAI Conference on Artificial Intelligence},
  volume={39},
  number={10},
  pages={10430--10438},
  year={2025}
}

@inproceedings{agustsson2017ntire,
  title={Ntire 2017 challenge on single image super-resolution: Dataset and study},
  author={Agustsson, Eirikur and Timofte, Radu},
  booktitle={Proceedings of the IEEE Conference on Computer vision and Pattern Recognition Workshops},
  pages={126--135},
  year={2017}
}

@article{serrano2024hosc,
  title={HOSC: A periodic activation function for preserving sharp features in implicit neural representations},
  author={Serrano, Danzel and Szymkowiak, Jakub and Musialski, Przemyslaw},
  journal={arXiv preprint arXiv:2401.10967},
  year={2024}
}

@article{thennakoon2025bandrc,
  title={BandRC: Band Shifted Raised Cosine Activated Implicit Neural Representations},
  author={Thennakoon, Pandula and Ranasinghe, Avishka and De Silva, Mario and Epakanda, Buwaneka and Godaliyadda, Roshan and Ekanayake, Parakrama and Herath, Vijitha},
  journal={arXiv preprint arXiv:2505.11640},
  year={2025}
}

@inproceedings{jacot2018neural,
  title={Neural Tangent Kernel: Convergence and Generalization in Neural Networks},
  author={Jacot, Arthur and Gabriel, Franck and Hongler, Cl{\'e}ment},
  booktitle={Advances in Neural Information Processing Systems},
  year={2018}
}

@inproceedings{jayasundara2025mire,
  title={MIRE: Matched Implicit Neural Representations},
  author={Jayasundara, Dhananjaya and Zhao, Heng and Labate, Demetrio and Patel, Vishal M},
  booktitle={Proceedings of the Computer Vision and Pattern Recognition Conference},
  pages={8279--8288},
  year={2025}
}

@book{boyd2001chebyshev,
  title={Chebyshev and Fourier Spectral Methods},
  author={Boyd, John P},
  year={2001},
  publisher={Courier Corporation}
}

@techreport{trefethen2013approximation,
  title={Approximation theory and approximation practice},
  author={Trefethen, Lloyd N},
  year={2013},
  institution={SIAM}
}

@article{zhu2024disorder,
  title={Disorder-invariant implicit neural representation},
  author={Zhu, Hao and Xie, Shaowen and Liu, Zhen and Liu, Fengyi and Zhang, Qi and Zhou, You and Lin, Yi and Ma, Zhan and Cao, Xun},
  journal={IEEE Transactions on Pattern Analysis and Machine Intelligence},
  volume={46},
  number={8},
  pages={5463--5478},
  year={2024},
  publisher={IEEE}
}

@article{xu2019frequency,
  title={Frequency principle: Fourier analysis sheds light on deep neural networks},
  author={Xu, Zhi-Qin John and Zhang, Yaoyu and Luo, Tao and Xiao, Yanyang and Ma, Zheng},
  journal={arXiv preprint arXiv:1901.06523},
  year={2019}
}

@article{wang2004image,
  title={Image quality assessment: from error visibility to structural similarity},
  author={Wang, Zhou and Bovik, Alan C and Sheikh, Hamid R and Simoncelli, Eero P},
  journal={IEEE Transactions on Image Processing},
  volume={13},
  number={4},
  pages={600--612},
  year={2004},
  publisher={IEEE}
}

@article{hertz2021sape,
  title={Sape: Spatially-adaptive progressive encoding for neural optimization},
  author={Hertz, Amir and Perel, Or and Giryes, Raja and Sorkine-Hornung, Olga and Cohen-Or, Daniel},
  journal={Advances in Neural Information Processing Systems},
  volume={34},
  pages={8820--8832},
  year={2021}
}

@article{field1987relations,
  title={Relations between the statistics of natural images and the response properties of cortical cells},
  author={Field, David J},
  journal={Journal of the Optical Society of America A},
  volume={4},
  number={12},
  pages={2379--2394},
  year={1987},
  publisher={OSA}
}

@inproceedings{landgraf2022pins,
  title     = {PINs: Progressive Implicit Networks for Multi-Scale Neural Representations},
  author    = {Landgraf, Zoe and Sorkine-Hornung, Alexander and Cabral, Ricardo Silveira},
  booktitle = {Proceedings of the 39th International Conference on Machine Learning (ICML)},
  year      = {2022},
  pages     = {11969--11984},
  publisher = {PMLR},
  volume    = {162},
}

@inproceedings{benbarka2022seeing,
  title={Seeing implicit neural representations as fourier series},
  author={Benbarka, Nuri and H{\"o}fer, Timon and Zell, Andreas and others},
  booktitle={Proceedings of the IEEE/CVF Winter Conference on Applications of Computer Vision},
  pages={2041--2050},
  year={2022}
}

@book{powell1981approximation,
  title={Approximation theory and methods},
  author={Powell, Michael James David},
  year={1981},
  publisher={Cambridge university press}
}

@inproceedings{mescheder2019occupancy,
  title={Occupancy networks: Learning 3d reconstruction in function space},
  author={Mescheder, Lars and Oechsle, Michael and Niemeyer, Michael and Nowozin, Sebastian and Geiger, Andreas},
  booktitle={Proceedings of the IEEE/CVF Conference on Computer Vision and Pattern Recognition},
  pages={4460--4470},
  year={2019}
}

@inproceedings{rezaeian2025sl2a,
  title={SL2A-INR: Single-Layer Learnable Activation for Implicit Neural Representation},
  author={Rezaeian, Reza and Heidari, Moein and Azad, Reza and Merhof, Dorit and Soltanian-Zadeh, Hamid and Hacihaliloglu, Ilker},
  booktitle={Proceedings of the IEEE/CVF International Conference on Computer Vision},
  pages={26065--26074},
  year={2025}
}

@inproceedings{
  liu2025kan,
  title={{KAN}: Kolmogorov--Arnold Networks},
  author={Liu, Ziming and Wang, Yixuan and Vaidya, Sachin and Ruehle, Fabian and Halverson, James and Soljacic, Marin and Hou, Thomas Y. and Tegmark, Max},
  booktitle={The Thirteenth International Conference on Learning Representations},
  year={2025}

}

@inproceedings{fathony2020multiplicative,
  title={Multiplicative filter networks},
  author={Fathony, Rizal and Sahu, Anit Kumar and Willmott, Devin and Kolter, J Zico},
  booktitle={International Conference on Learning Representations},
  year={2020}
}
}

\clearpage
\setcounter{page}{1}
\maketitlesupplementary
\appendix
We provide the proof of Theorem~\ref{theorem2} in Section~\ref{theorem3_proof}.
The parameter settings for the main experiments in the paper are described in Section~\ref{Experimental Settings}.
Additional experiments are presented in Section~\ref{Additional Experiments}, and further ablation studies are included in Section~\ref{Additional Ablation}.
\section{Proofs of Theorems}\label{theorem3_proof}


Before presenting the theorem, we first restate the framework of INRs. INRs are typically defined as a coordinate mapping function $\gamma: \mathbb{R}^D \to \mathbb{R}^M$ followed by an MLP. 
The MLP is parameterized by weight matrices $W^{(\ell)} \in \mathbb{R}^{F_{\ell-1} \times F_\ell}$, bias vectors $b^{(\ell)} \in \mathbb{R}^{F_\ell}$, and element-wise activation functions $\rho^{(\ell)}: \mathbb{R} \to \mathbb{R}$, applied at each layer $\ell = 1, \ldots, L-1$. 
Let $z^{(\ell)}$ denote the post-activation output of layer $\ell$. The forward computation is expressed as:
\begin{equation}
\begin{split}
z^{(0)} &= \gamma(x), \\
z^{(\ell)} &= \rho^{(\ell)}\!\left( W^{(\ell)} z^{(\ell-1)} + b^{(\ell)} \right), \quad \ell = 1, \ldots, L-1, \\
f_{\boldsymbol{\theta}}(x) &= W^{(L)} z^{(L-1)} + b^{(L)}.
\end{split}
\label{eq:inr_architecture}
\end{equation}


\begin{theorem_unnumbered}
Let $f_{\boldsymbol{\theta}}: \mathbb{R} \to \mathbb{R}$ be an INR of the form of Eq.~\eqref{eq:inr_architecture} with depth $L$, where the activation functions are polynomials of degree at most $K$, i.e., $\rho(z) = \sum_{k=0}^{K} \alpha_k z^k$. 
Let $J = \{j_1, \dots, j_{M}\} \subset \mathbb{N}$ be a set of $M$ arbitrary non-negative integer orders.
The input mapping $\gamma(x)$ embeds the coordinate $x$ into a vector of Chebyshev polynomials:
\[
z^{(0)} = \gamma(x) = 
\begin{bmatrix}
T_{j_1}(x), T_{j_2}(x), \dots , T_{j_{M}}(x)
\end{bmatrix}^{\!\top}.
\]
Specifically, the functional space of this architecture is constrained to the form:
\begin{equation*}
f_{\boldsymbol{\theta}}(x) = \sum_{m \in \mathcal{H}(J)} c_m T_m(x),
\end{equation*}
where $c_m \in \mathbb{R}$ are coefficients determined by the network parameters, and the set of reachable spectral orders $\mathcal{H}(J)$ is a subset of the integer combination set:

\begin{align*}
\mathcal{H}(J)
\subseteq
\tilde{\mathcal{H}}(J)
:=
\Bigl\{
    h = \lvert\,\sum_{t=1}^{M} s_t j_t\,\rvert
    ~\Bigm|\,
    s_t \in \mathbb{Z},\\
    \sum_{t=1}^{M} |s_t|
    \le K^{L-1}
\Bigr\}.
\end{align*}

\end{theorem_unnumbered}


Since Chebyshev polynomials share the trigonometric properties of cosine functions, they satisfy the product-to-sum identity:
\begin{equation}
T_p(x) T_q(x) = \tfrac{1}{2}\bigl[\,T_{p+q}(x) + T_{|p-q|}(x)\,\bigr].
\label{eq:chebyshev_product}
\end{equation}
Leveraging this property, we extend the theoretical framework to Chebyshev polynomials via the following lemmas.

\begin{lemma}[Chebyshev Product]
Let $\{T_p(x)\}_{p \in P}$ and $\{T_q(x)\}_{q \in Q}$ be two collections of first-kind Chebyshev polynomials indexed by sets $P, Q \subset \mathbb{N}$. 
Let $\{\beta_p^{(1)}\}_{p \in P}$ and $\{\beta_q^{(2)}\}_{q \in Q}$ be their corresponding scalar coefficients. Then,
\begin{equation}
\left( \sum_{p \in P} \beta_p^{(1)} T_p(x) \right)
\left( \sum_{q \in Q} \beta_q^{(2)} T_q(x) \right)
= \sum_{h \in \mathcal{S}} \hat{\beta}_h T_h(x),
\label{eq:sup_1}
\end{equation}
where the resulting index set $\mathcal{S}$ is defined as:
\begin{equation}
\mathcal{S} = \bigl\{\, h \mid h = p+q \text{ or } h = |p-q|,~ p \in P,~ q \in Q \,\bigr\},
\label{eq:sup_2}
\end{equation}
and $\{\hat{\beta}_h \mid h \in \mathcal{S}\} \subset \mathbb{R}$ are the combined coefficients.
\end{lemma}

\begin{proof}
Expanding the product of the two sums and applying the identity in Eq.~\eqref{eq:chebyshev_product}:
\begin{align}
\begin{split}
&\left( \sum_{p \in P} \beta_p^{(1)} T_p(x) \right)
\left( \sum_{q \in Q} \beta_q^{(2)} T_q(x) \right)\\
&= \sum_{p \in P}\sum_{q \in Q} \beta_p^{(1)} \beta_q^{(2)} T_p(x)T_q(x) \\
&= \sum_{p \in P}\sum_{q \in Q} 
\frac{1}{2}\beta_p^{(1)}\beta_q^{(2)}
\bigl[\,T_{p+q}(x)+T_{|p-q|}(x)\,\bigr] \\
&= \sum_{h \in \mathcal{S}} \hat{\beta}_h T_h(x).
\end{split}
\end{align}
The final step follows by grouping terms associated with the same Chebyshev order $h \in \mathcal{S}$.
\end{proof}

\begin{lemma}[Chebyshev Power]
Let $\{T_j(x)\}_{j \in J}$ be a collection of first-kind Chebyshev polynomials indexed by $J \subset \mathbb{N}$, with coefficients $\{\beta_j\}_{j \in J} \subset \mathbb{R}$. 
For any integer $k \ge 1$,
\begin{equation}
\left( \sum_{j \in J} \beta_j T_j(x) \right)^k 
= \sum_{h \in \mathcal{S}_k} \tilde{\beta}_h T_h(x),
\label{eq:sup_3}
\end{equation}
where the order set $\mathcal{S}_k$ is recursively defined as:
\begin{align*}
\mathcal{S}_1 &= J, \\
\mathcal{S}_{k+1} &= 
\begin{aligned}[t]
    &\mathcal{D}(\mathcal{S}_k, J)
    := \bigl\{\, h' \mid h' = h + j ~\text{or}\\
    &\quad ~ h' = |h - j|,~ h \in \mathcal{S}_k,~ j \in J \,\bigr\}.
\end{aligned}
\end{align*}
Moreover, $\mathcal{S}_k$ is a subset of the ``order combination'' set $\tilde{\mathcal{S}}_k(J)$:
\begin{equation}
\mathcal{S}_k \subseteq \tilde{\mathcal{S}}_k(J)
:= \Bigl\{ h = \Bigl| \sum_{j \in J} c_j j \Bigr| 
~\Bigm|~ c_j \in \mathbb{Z},~ \sum_{j \in J} |c_j| \le k \Bigr\}.
\label{eq:sup_4}
\end{equation}
\label{lemma2}
\end{lemma}

\begin{proof}
We proceed by induction.

\noindent\textbf{1. Proof of Eq.~\eqref{eq:sup_3} (Recursive Structure)} \\
\textit{Base Case ($k=1$):} Trivially, $\mathcal{S}_1 = J$. \\
\textit{Inductive Step:} Assume Eq.~\eqref{eq:sup_3} holds for some $k \ge 1$. Then,
\begin{equation}
	\begin{aligned}
		\left( \sum_{j \in J} \beta_j T_j(x) \right)^{k+1}
		&= \left( \sum_{j \in J} \beta_j T_j(x) \right)^k
		\left( \sum_{j \in J} \beta_j T_j(x) \right) \\
		&= \left( \sum_{h \in \mathcal{S}_k} \tilde{\beta}_{h} T_{h}(x) \right)
		\left( \sum_{j \in J} \beta_j T_j(x) \right).
	\end{aligned}
	\label{eq:sup_5}
\end{equation}

Applying the Chebyshev Product Lemma to Eq.~\eqref{eq:sup_5}, the resulting indices $h'$ belong to the set formed by the sum and difference of indices in $\mathcal{S}_k$ and $J$. Thus, $\mathcal{S}_{k+1} = \mathcal{D}(\mathcal{S}_k, J)$, and Eq.~\eqref{eq:sup_3} holds for $k+1$.

\medskip
\noindent
\textbf{2. Proof of $\mathcal{S}_k \subseteq \tilde{\mathcal{S}}_k(J)$} \\
\textit{Base Case ($k=1$):} Any $j \in J$ can be written as $|\sum_{i \in J} c_i i|$ with $c_j=1$ and others zero, satisfying $\sum |c_i| = 1$. Thus $\mathcal{S}_1 \subseteq \tilde{\mathcal{S}}_1(J)$.

\textit{Inductive Step:} Assume $\mathcal{S}_k \subseteq \tilde{\mathcal{S}}_k(J)$. 
Consider $h' \in \mathcal{S}_{k+1}$. By definition, $h' = |h \pm j|$ for some $h \in \mathcal{S}_k$ and $j \in J$.
By the inductive hypothesis, $h = |\sum_{i \in J} c_i i|$ with $\sum |c_i| \le k$.
Using $T_{-n}(x) = T_n(x)$, we can absorb signs into the coefficients:
\begin{equation}
h' = \Bigl| \Bigl(\pm \sum_{i \in J} c_i i \Bigr) \pm j \Bigr| = \Bigl| \sum_{i \in J} c'_i i \Bigr|,
\end{equation}
where $c'_j = \pm c_j \pm 1$ and $c'_i = \pm c_i$ for $i \neq j$. The complexity satisfies $\sum |c'_i| \le \sum |c_i| + 1 \le k + 1$.
Therefore, $h' \in \tilde{\mathcal{S}}_{k+1}(J)$.
\end{proof}


With the above lemmas established, we now proceed to prove the theorem.

\begin{proof}[Proof of Theorem]
The proof proceeds by induction on the layer index $\ell$ to track the propagation of the order complexity.

\noindent\textbf{Base Case ($\ell=0$).}
The initial representation is $z^{(0)} = [T_{j_1}(x), \dots, T_{j_{M}}(x)]^\top$. 
Each component corresponds to a Chebyshev polynomial of order $m \in J = \{j_1, \dots, j_M\}$. 
Evaluating the coefficient complexity defined in $\tilde{\mathcal{H}}(J)$, for the $t$-th component $T_{j_t}(x)$, we can define the integer weights as $s_t=1$ and $s_{t'}=0$ for $t' \neq t$. 
The resulting spectral complexity is $\sum_{i=1}^{M} |s_i| = 1 = K^0$. 
Thus, the statement holds for $\ell=0$.

\medskip
\noindent\textbf{Inductive Hypothesis.}
Assume that for the hidden layer $\ell-1$ (where $1 \le \ell < L$), every component of $z^{(\ell-1)}$ is a linear combination of Chebyshev polynomials $T_{m}(x)$, where each order $m$ belongs to the set:
\begin{equation*}
\Omega_{\ell-1} = \Bigl\{ \Bigl| \sum_{t=1}^{M} s_t j_t \Bigr| ~\Bigm|~ \sum_{t=1}^{M} |s_t| \le K^{\ell-1} \Bigr\}.
\end{equation*}

\medskip
\noindent\textbf{Inductive Step.}
Consider the representation at layer $\ell$:
\[
z^{(\ell)} = \rho^{(\ell)}\!\left( W^{(\ell)} z^{(\ell-1)} + b^{(\ell)} \right).
\]
Let $v = W^{(\ell)} z^{(\ell-1)} + b^{(\ell)}$ be the pre-activation vector. 
Since $z^{(\ell-1)}$ consists of polynomials with orders in $\Omega_{\ell-1}$, and the bias $b^{(\ell)}$ corresponds to $T_0(x)$ (where all $s_t=0$, satisfying $\sum |s_t| = 0 \le K^{\ell-1}$), every component of $v$ is a linear combination of Chebyshev polynomials with orders in $\Omega_{\ell-1}$.

The activation function is a polynomial of degree $K$, i.e., $\rho(z) = \sum_{k=0}^{K} \alpha_k z^k$. 
We analyze the term $v^k$ for any component of $v$. 
By the Lemma~\ref{lemma2}, raising a sum of Chebyshev polynomials to the $k$-th power generates new orders. 
Specifically, if the input orders have coefficient complexity bounded by $C$, the output orders have complexity bounded by $k \cdot C$.

Applying this to our case with input complexity $C = K^{\ell-1}$ and maximum power $k = K$, the resulting orders $m$ in $z^{(\ell)}$ satisfy:
\[
\sum_{t=1}^{M} |s'_t| \le K \cdot K^{\ell-1} = K^{\ell}.
\]
Thus, the elements of $z^{(\ell)}$ are composed of Chebyshev polynomials with orders in $\Omega_{\ell}$.

Finally, the network output is a linear transformation $f_{\boldsymbol{\theta}}(x) = W^{(L)} z^{(L-1)} + b^{(L)}$. 
Since the linear operation does not increase the polynomial degree complexity, the spectral orders remain bounded by that of $z^{(L-1)}$. 
Substituting $\ell = L-1$ into our inductive result, the output orders belong to $\mathcal{H}(J)$ with the constraint:
\[
\sum_{t=1}^{M} |s_t| \le K^{L-1}.
\]
This concludes the proof.
\end{proof}

\section{Experimental Settings for Main Tasks}\label{Experimental Settings}
We briefly review the CAFE+ and clarify the two symbols $M$ and $J$ used in our experimental descriptions.
Given a coordinate input $\mathbf{x}\in\mathbb{R}^D$, CAFE+ maps it to a high-dimensional vector through
\begin{equation}
\Psi(\mathbf{x})
=
\bigodot_{i=1}^{N}
\left\{
\mathbf{W}_i
\big[
\Phi_{\mathrm{FF}}(\mathbf{x}),\;
\Phi_{\mathrm{CF}}(\mathbf{x})
\big]
+
\mathbf{b}_i
\right\},
\label{eq:phi_combined_supp}
\end{equation}
where $\Phi_{\mathrm{FF}}$ and $\Phi_{\mathrm{CF}}$ denote the Fourier and Chebyshev encodings, respectively.

The Fourier feature mapping is defined as
\begin{equation}
\Phi_{\mathrm{FF}}(\mathbf{x})
=
\left[
\sin(2\pi \boldsymbol{\omega}_i^\top \mathbf{x}),\;
\cos(2\pi \boldsymbol{\omega}_i^\top \mathbf{x})
\right]_{i=1}^{M},
\label{eq:RFF_M_supp}
\end{equation}
where $M$ indicates the number of sampled Fourier frequencies.

The Chebyshev feature mapping is given by
\begin{equation}
\Phi_{\mathrm{CF}}(\mathbf{x})
=
\big[\,T_j(x_d)\,\big]_{\,d=1,\dots,D;\, j=0,\dots,J-1}
\in \mathbb{R}^{DJ},
\label{eq:CPE_supp}
\end{equation}
where $T_j(\cdot)$ is the Chebyshev polynomial of degree $j$, and $J$ specifies the number of Chebyshev basis functions per dimension.

Throughout our experiments, we use $M$ and $J$ to explicitly denote the number of Fourier features and Chebyshev features employed in the coordinate encoding.

\paragraph{Implementation Note.} 
For our method, the number of Fourier features $M$ and the number of Chebyshev features $J$ need to be specified.
 In practice, our model is relatively robust to the choice of $J$, and we recommend choosing $M$ and $J$ such that the feature dimensions after Fourier and Chebyshev encoding maintain an approximate ratio of $3\!:\!1$. This ratio does not need to be strictly satisfied; it simply provides a convenient guideline for selecting $M$ and $J$. For example, in a 2D image fitting task with a target hidden dimension of 256, one may set $M=96$ and $J=32$.  
After encoding, the Fourier features produce a vector of length $2M=192$, while the Chebyshev features yield a vector of length $2J=64$.  
Their ratio is therefore $192:64 = 3:1$, and the concatenated representation matches the desired dimension of 256.

We provide the detailed experimental settings for 2D image fitting, 3D shape representation, and NeRF experiments as follows:
\begin{table*}[t]
\centering
\caption{PSNR comparison on DIV2K images 0873, 0878, and 0891 under full-resolution and $\times 2$ downsampling settings. 
The number of parameters and average inference time per image are also reported.}

\renewcommand{\arraystretch}{1.2}
\setlength{\tabcolsep}{8pt}

\begin{tabular}{c|cccccc|cc}
\hline
Method & 0873 & 0873 ($\times 2$) & 0878 & 0878 ($\times 2$) & 0891 & 0891 ($\times 2$) & Params (M) & Avg. Time (s) \\
\hline
SIREN     & 24.91 & 31.56 & 32.42 & 36.94 & 27.20 & 32.48 & 0.45 & 292.50 \\
WIRE      & 23.45 & 28.98 & 29.12 & 32.04 & 23.99 & 27.23 & 0.18 & 707.48 \\
SCONE     & \underline{26.10} & 33.14 & \underline{34.09} & \underline{39.12} & \underline{27.68} & \underline{33.19} & 0.45 & 647.06 \\
FINER     & 25.28 & 31.59 & 32.66 & 36.86 & 26.77 & 32.07 & 0.45 & 375.28 \\
SL$^2$A   & 25.95 & \underline{34.28} & 34.00 & 38.59 & 26.82 & 32.53 & 0.84 & 456.12 \\
Ours      & \textbf{29.65} & \textbf{37.18} & \textbf{35.79} & \textbf{40.82} & \textbf{30.79} & \textbf{36.46} & 0.74 & 326.95 \\
\hline
\end{tabular}
\label{tab:div2k_psnr_vertical}
\end{table*}

\setlength{\tabcolsep}{2pt}
\begin{figure*}[htbp]
    \centering

    \begin{tabular}{ccc}
        \includegraphics[width=0.30\textwidth]{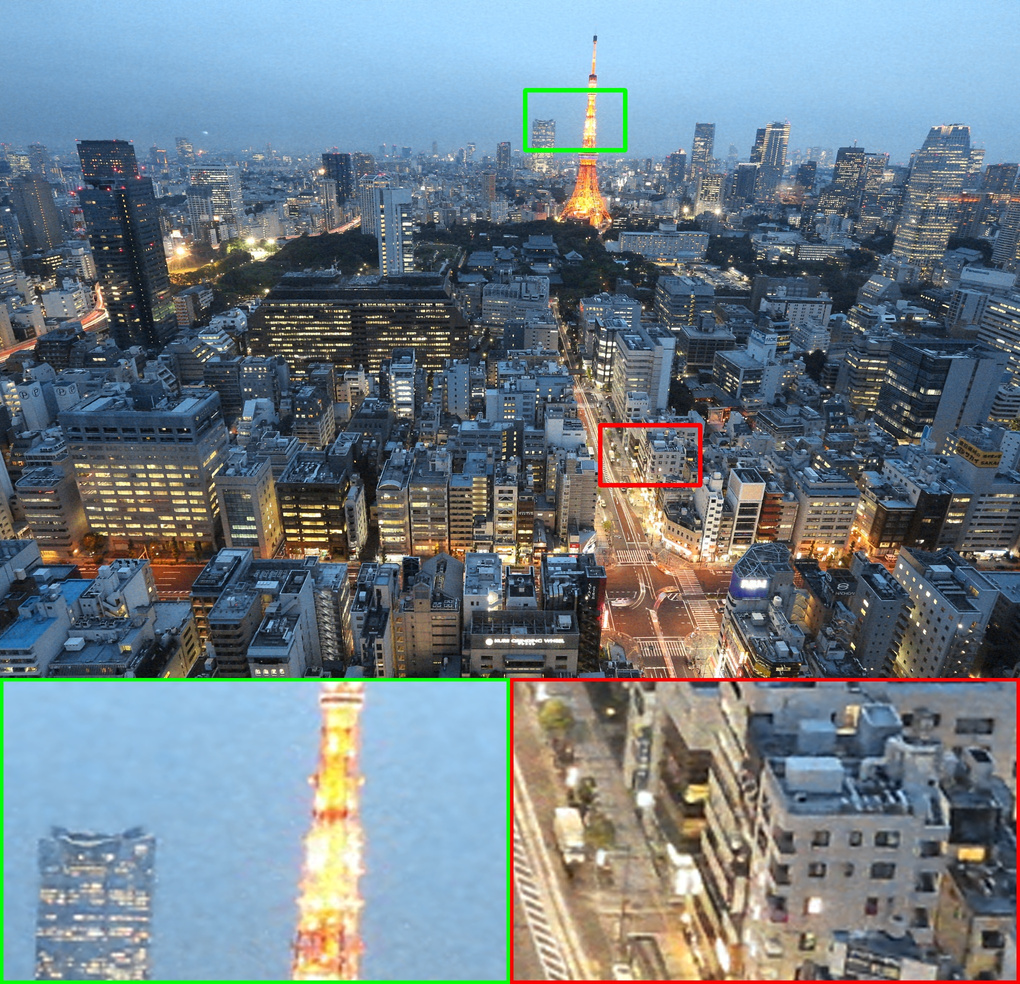} &
        \includegraphics[width=0.30\textwidth]{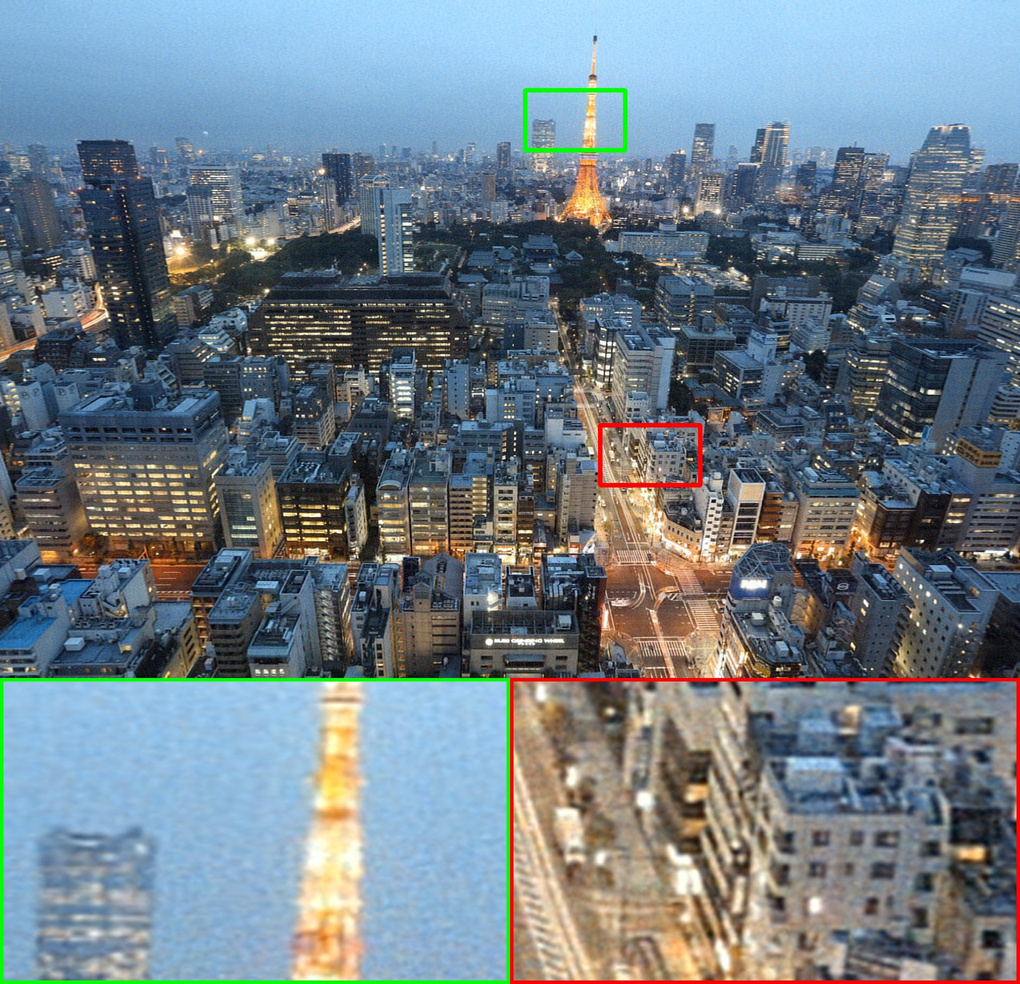} &
        \includegraphics[width=0.30\textwidth]{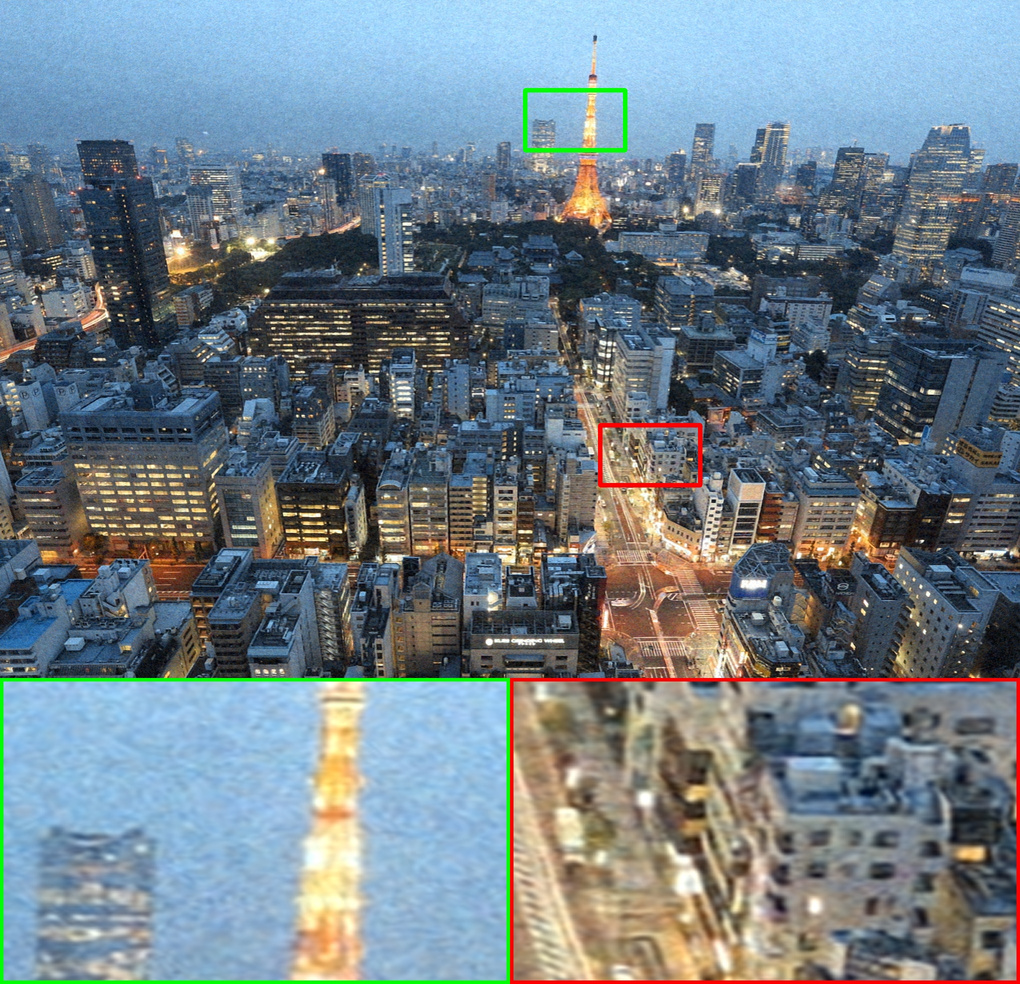} \\
        \textbf{Ours (29.65)} & SL$^2$A (25.95) & FINER (25.28) \\
       \includegraphics[width=0.30\textwidth]{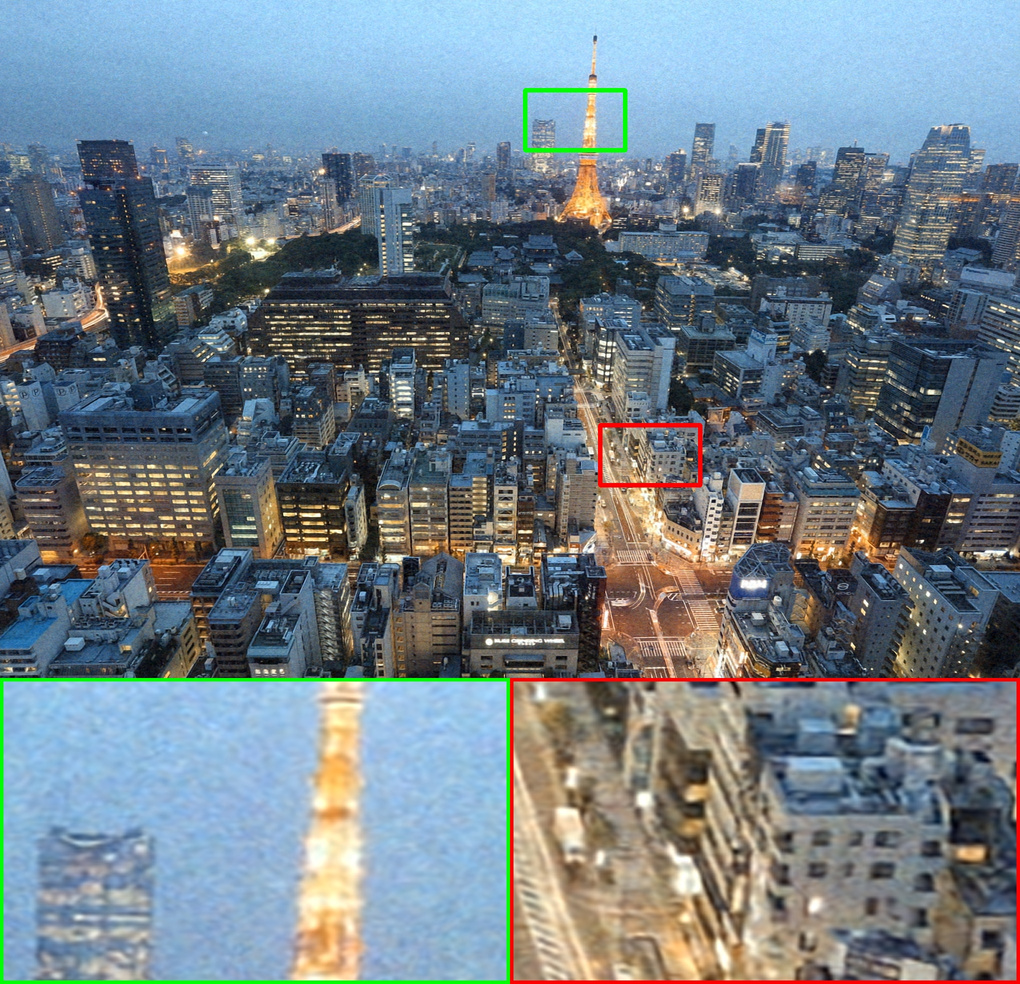} &
        \includegraphics[width=0.30\textwidth]{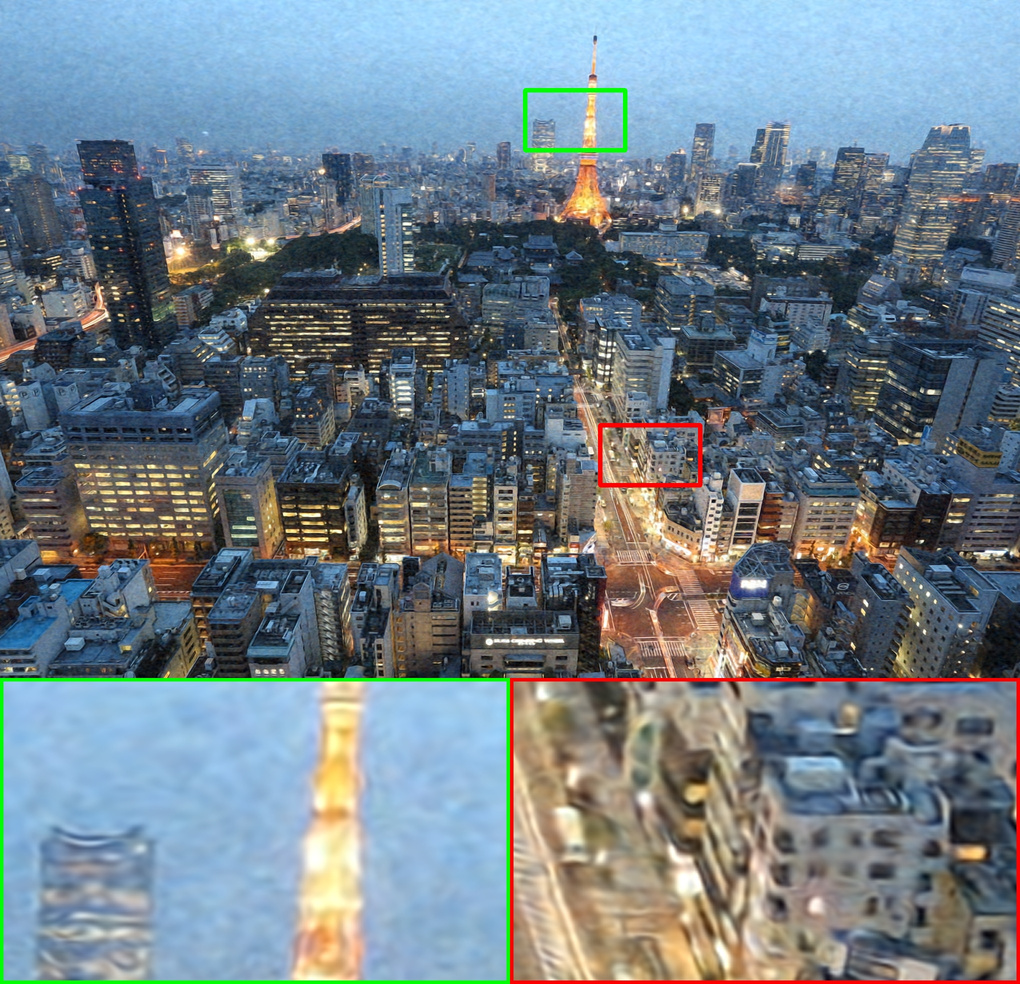} &
        \includegraphics[width=0.30\textwidth]{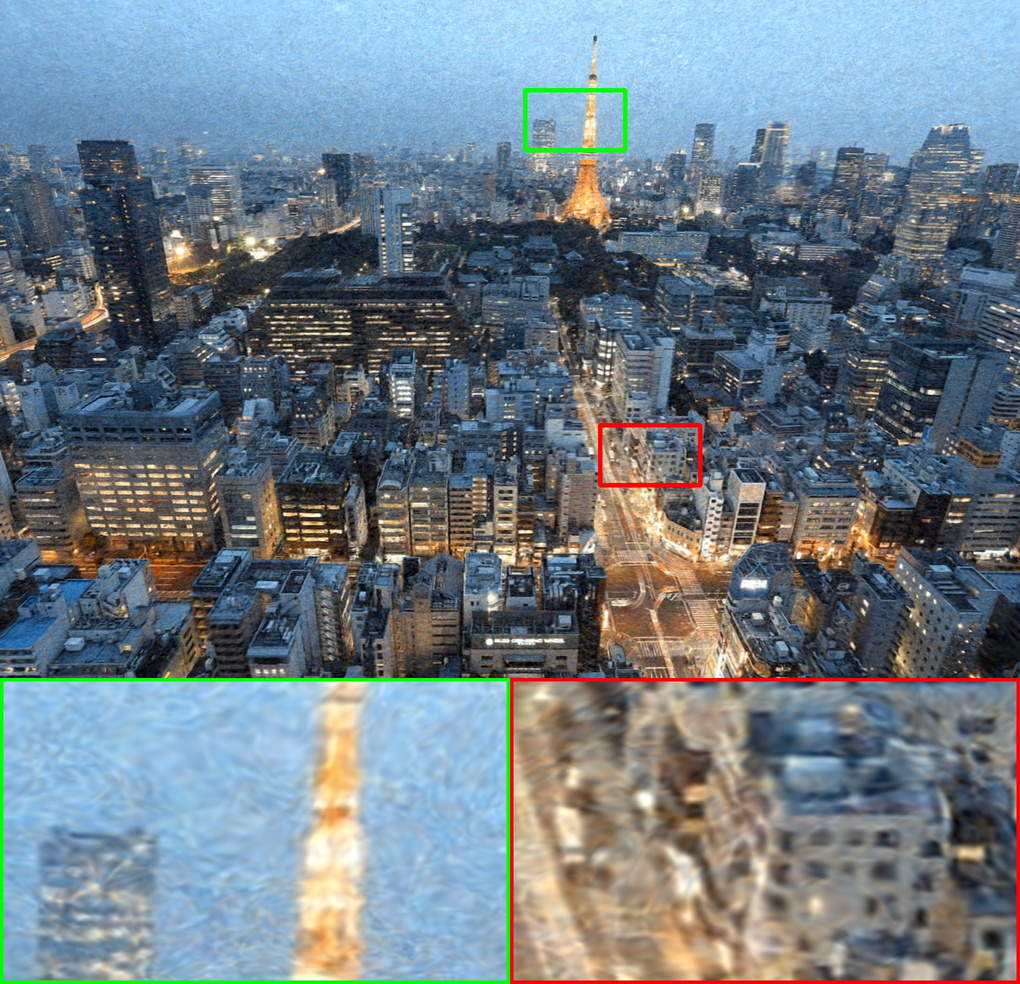} \\
        SCONE (26.10) & SIREN (24.91) & WIRE (23.45) \\  
        
    \end{tabular}

    \caption{Visualization results on the full-resolution DIV2K image 0873, where we zoom in on two selected regions.}
    \label{vis:div2k_full}
\end{figure*}

\begin{itemize}
    \item \textbf{2D Image Fitting.} 
    For SIREN~\cite{sitzmann2020implicit} and FINER~\cite{liu2024finer}, we use $\omega_0 = 30$ and three hidden layers. 
    For Wire~\cite{saragadam2023wire}, we use $\omega_0 = 20$, $s_0 = 30$ and two hidden layers. 
    SCONE~\cite{li2024learning} uses three hidden layers with $\omega_0 = 30$ and $\omega_{\ell}=[90,\,60,\,30,\,10]$. 
    SL$^2$A~\cite{rezaeian2025sl2a} uses three hidden layers with $\text{degree}=512$ and $\text{rank}=128$. 
    Our method employs three parallel linear layers with an RFF scale of 30. The default configuration uses Fourier features $M=88$ and Chebyshev features $J=30$ with one hidden layer in the backbone MLP, while the larger-parameter variant adopts $M=96$ and $J=32$ with two backbone MLP hidden layers.
    Except for SCONE which uses 250 hidden neurons, all other methods use 256 hidden neurons. 
    We train all models for 6,000 iterations.

    \item \textbf{3D Shape Representation.} 
    For SIREN and FINER, we use $\omega_0 = 30$ and three hidden layers. 
    Wire uses $\omega_0 = 20$, $s_0 = 30$ and two hidden layers. 
    SCONE uses two hidden layers with $\omega_0 = 30$ and $\omega_{\ell}=[60,\,30,\,10]$. 
    SL$^2$A uses two hidden layers with $\text{degree}=256$ and $\text{rank}=128$. Our method adopts two parallel linear layers with an RFF scale of 30. 
    We use $M=98$ Fourier features and $J=20$ Chebyshev features, ensuring that the encoded dimension is $2M + 3J = 256$.
    Except for SCONE which uses 250 hidden neurons, all other methods use 256 hidden neurons. 
    During training, all voxels are randomly shuffled in each iteration and processed in mini-batches of up to 200,000 points, for a total of 200 iterations.

    \item \textbf{NeRF~\cite{mildenhall2021nerf}.} 
    Our implementation follows the settings provided by FINER. 
    For our method, we set the frequency features in PE to 10 and concatenate Chebyshev features with $J=8$, resulting in a concatenated feature length of $3 \times 2 \times 10 + 3 \times 8 = 84$. 
    Two parallel linear projections map these features to length 182 and are combined via Hadamard product. 
    The backbone MLP has two hidden layers with 182 neurons each. 
    This produces a 182-dimensional feature vector. Excluding the first feature, the remaining 181 features are concatenated with the input coordinates, processed through two parallel linear layers, and then passed through the two-layer backbone MLP to obtain the final output.

\end{itemize}

Additional visualizations are provided for 2D image fitting (Fig.~\ref{fig:sup_2d_img_fiting_vis}), 3D shape representation (Fig.~\ref{sup_fig:3D_shape}), and NeRF (Fig.~\ref{vis:nerf1}).

\section{Additional Experiments}\label{Additional Experiments}

\subsection{Full-Resolution DIV2K Representation}

We perform image fitting on DIV2K~\cite{agustsson2017ntire} images 0873 (2040$\times$1456), 0878 (2040$\times$1740), and 0891 (1668$\times$2040) at full resolution, as well as their 2$\times$ downsampled versions. We train for 6{,}000 iterations, randomly sampling 512$\times$512 points per iteration.
All methods use hidden layers with 384 neurons. 
For SIREN and FINER, we adopt three hidden layers with $\omega_0=30$. 
For Wire, we use two hidden layers with $\omega_0=20$ and $s_0=20$. 
For SCONE, we use three hidden layers with $\omega_0=30$ and $\omega_{\ell}=[90, 60, 30, 10]$. 
For SL$^2$A, we use three hidden layers with $\text{degree}=512$ and $\text{rank}=192$. 
For our method, we employ three parallel linear layers with an RFF scale of 30, using $M=144$ Fourier features and $J=48$ Chebyshev features. We present the visualization results of the DIV2K image 0873 at its original resolution in Fig.~\ref{vis:div2k_full}. The PSNR, model parameters, and training time are reported in Table~\ref{tab:div2k_psnr_vertical}. These results further demonstrate the effectiveness of our method.

\subsection{Image Restoration}

We applied our framework to three classical image restoration tasks, namely image denoising, inpainting, and super-resolution. Specifically, we considered Gaussian noise with a standard deviation of 0.2, randomly removed 50\% of pixels for inpainting, and performed $\times4$ super-resolution. 

The detailed configurations for each task are summarized below:

\begin{itemize}
    \item \textbf{Denoising.}
    All methods use hidden layers with 256 neurons and two hidden layers.
    SIREN and FINER use $\omega_0=10$; Wire uses $\omega_0=4$ and $s_0=4$;
    SCONE uses $\omega_0=2$ with $\omega_{\ell}=[60,\,30,\,10]$;
    SL$^2$A uses $\text{degree}=64$ and $\text{rank}=64$.
    For our method, we use two parallel linear layers with $M=96$ Fourier features, $J=32$ Chebyshev features, and an RFF scale of 2.
{
\setlength{\tabcolsep}{0.6pt}
\begin{figure*}[htbp]
    \centering
    {\scriptsize
    \begin{tabular}{cccccccc} 
        \includegraphics[width=0.12\textwidth]{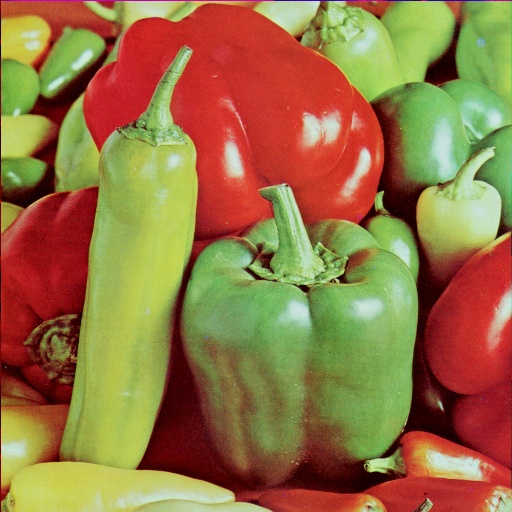} &
        \includegraphics[width=0.12\textwidth]{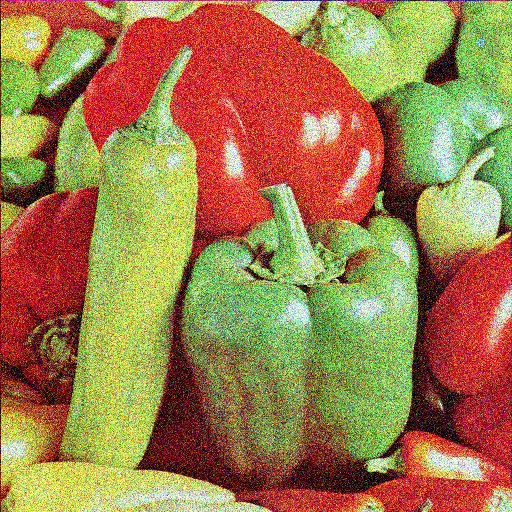} &
        \includegraphics[width=0.12\textwidth]{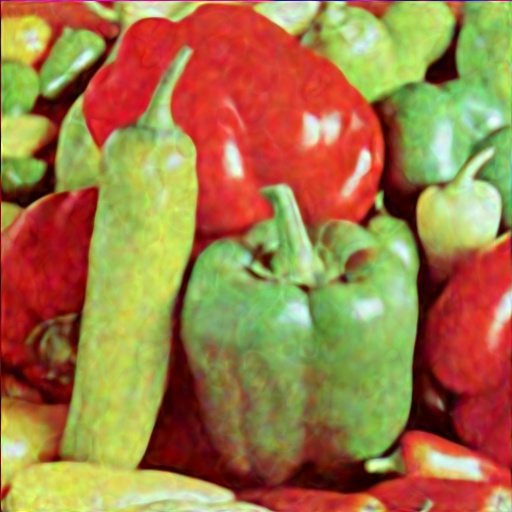} &
        \includegraphics[width=0.12\textwidth]{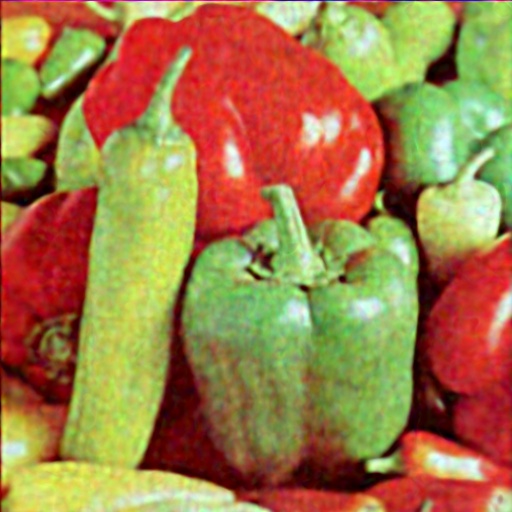} & 
        \includegraphics[width=0.12\textwidth]{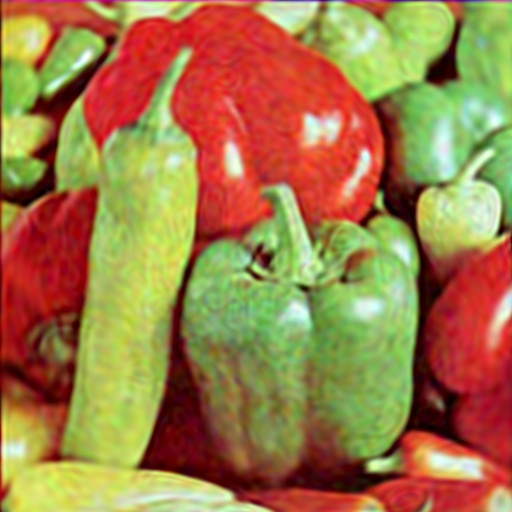} &
        \includegraphics[width=0.12\textwidth]{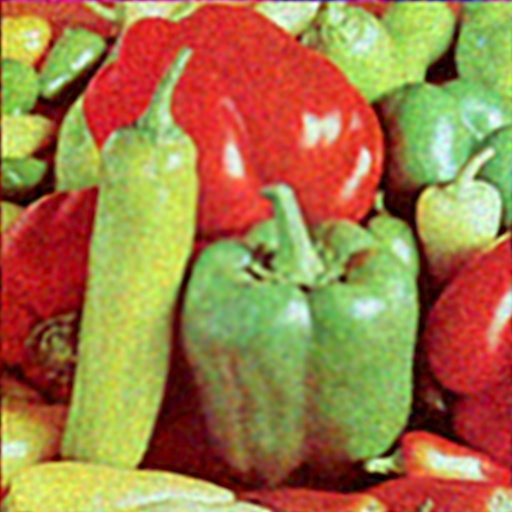} &
        \includegraphics[width=0.12\textwidth]{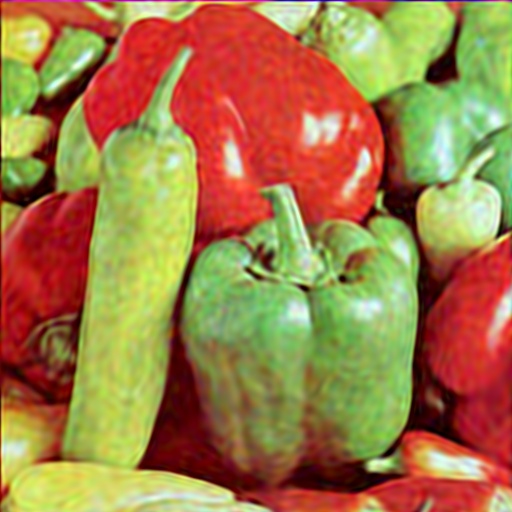} &
        \includegraphics[width=0.12\textwidth]{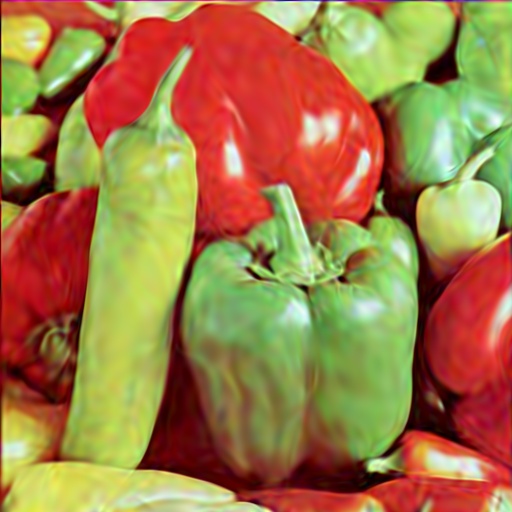} 
        \\
         GT(PSNR) & Observed(13.98) & \textbf{Ours(26.46)} &
        SL$^2$A(25.59) & FINER(25.12) & SCONE(25.05) & SIREN(25.72) & WIRE(26.28) \\
        
        \includegraphics[width=0.12\textwidth]{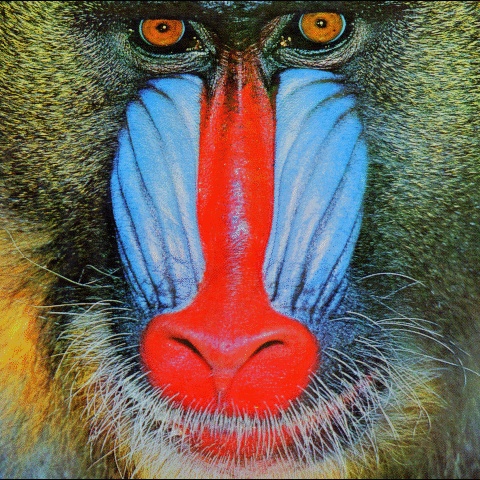} &
        \includegraphics[width=0.12\textwidth]{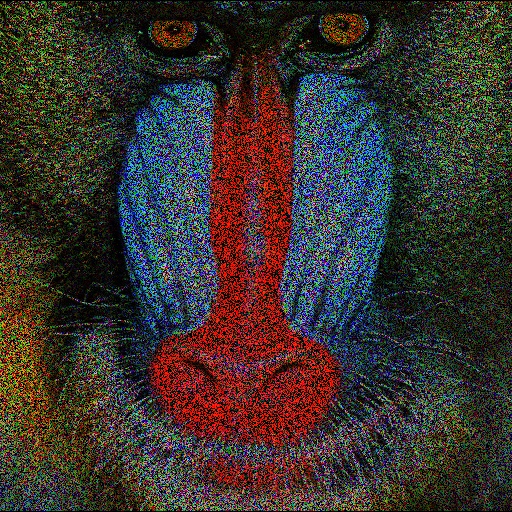} &
        \includegraphics[width=0.12\textwidth]{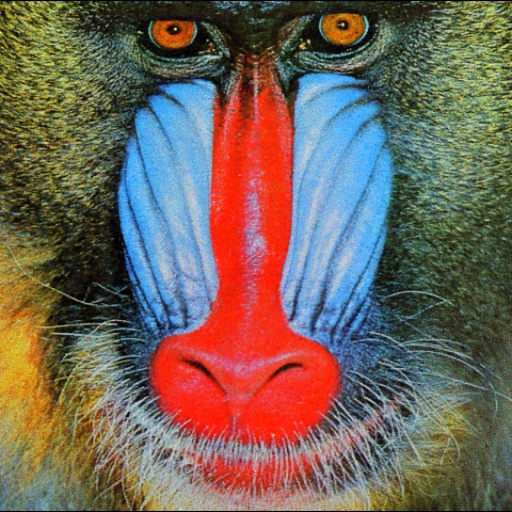} &
        \includegraphics[width=0.12\textwidth]{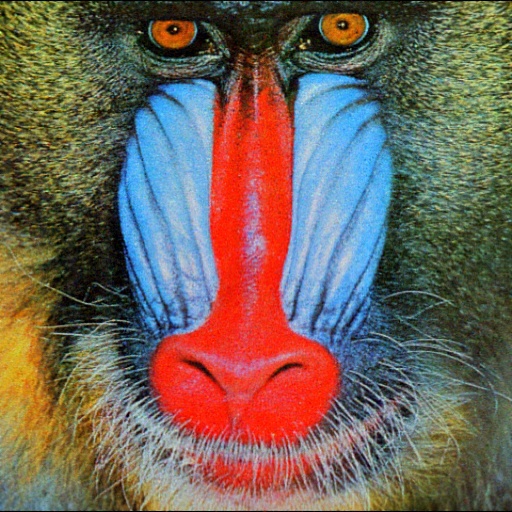} &
        \includegraphics[width=0.12\textwidth]{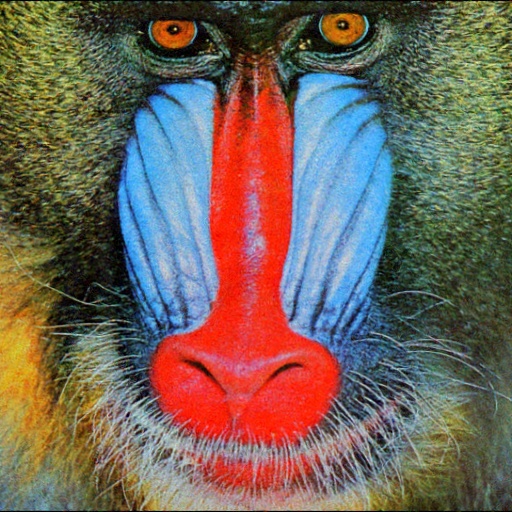} &
        \includegraphics[width=0.12\textwidth]{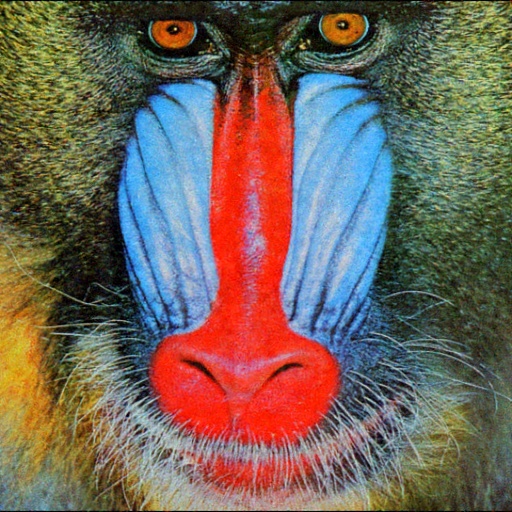} &
        \includegraphics[width=0.12\textwidth]{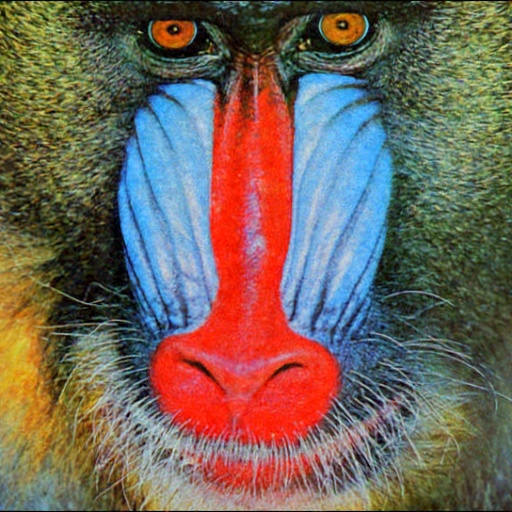} &
        \includegraphics[width=0.12\textwidth]{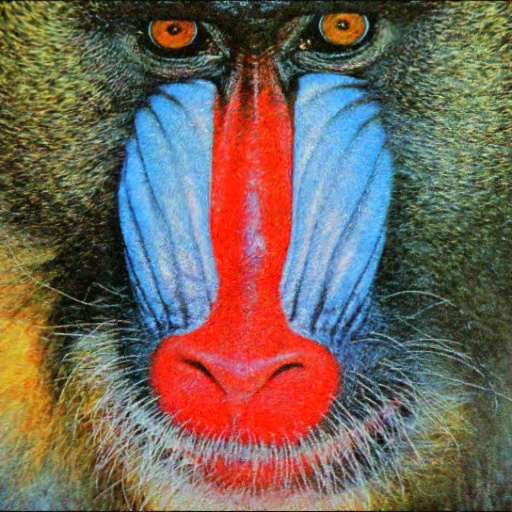}
        \\
         GT(PSNR) & Observed(9.59) & \textbf{Ours(27.28)} &
        SL$^2$A(26.94) & FINER(26.70) & SCONE(26.98) & SIREN(25.73) & WIRE(24.80) \\

        \includegraphics[width=0.12\textwidth]{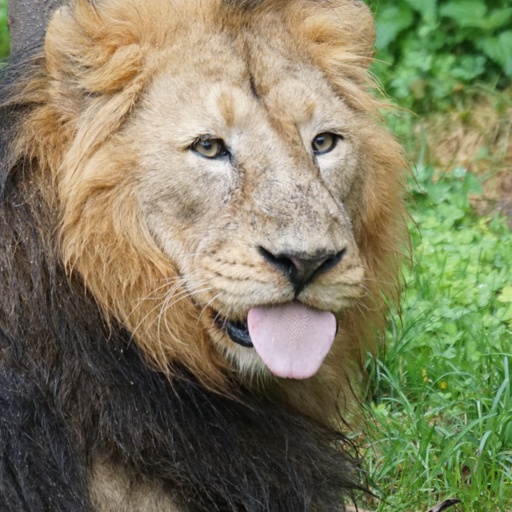} &
        \includegraphics[width=0.12\textwidth]{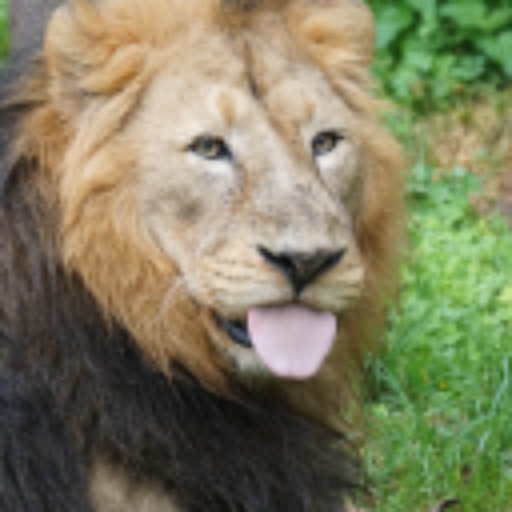} &
        \includegraphics[width=0.12\textwidth]{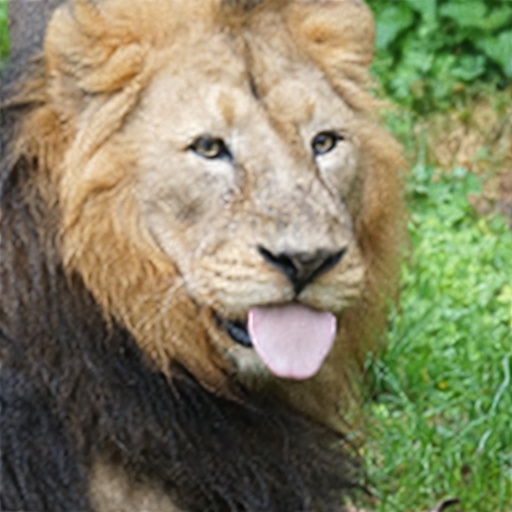} &
        \includegraphics[width=0.12\textwidth]{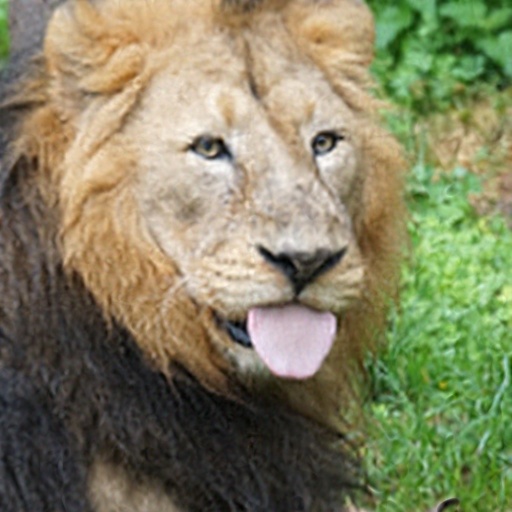} &
        \includegraphics[width=0.12\textwidth]{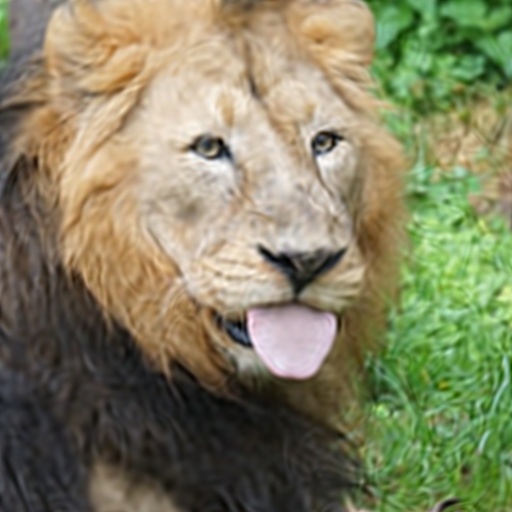} &
        \includegraphics[width=0.12\textwidth]{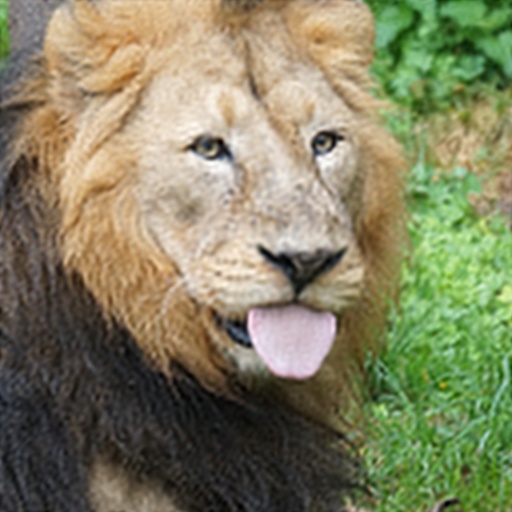} &
        \includegraphics[width=0.12\textwidth]{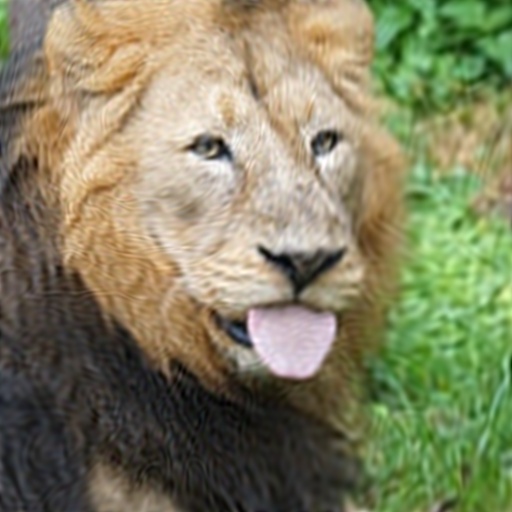} &
        \includegraphics[width=0.12\textwidth]{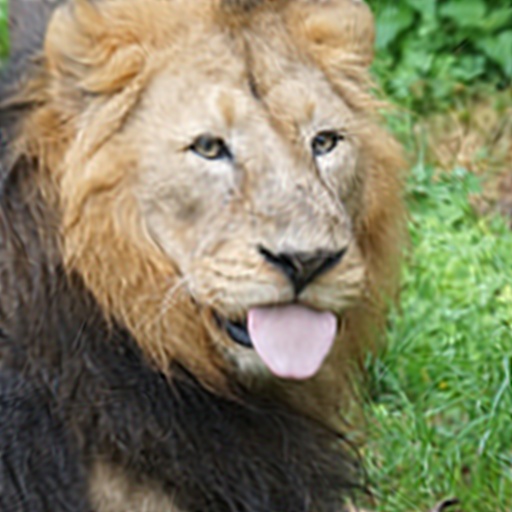}
        \\
         GT(PSNR) & Bilinear(26.71) & \textbf{Ours(27.18)} &
        SL$^2$A(26.94) & FINER(26.57) & SCONE(27.01) & SIREN(25.73) & WIRE(26.87) \\
    \end{tabular}
    }
    \caption{
        Visualization of image denoising, inpainting, and single-image super-resolution tasks (top to bottom).
        For denoising, Gaussian noise with standard deviation 0.2 is added.
        Inpainting randomly removes 50\% of pixels.
        Super-resolution uses a scale factor of 4.
    }
    \label{fig:inverse_problem}
\end{figure*}
}
\begin{figure*}[h]
    \centering
    \setlength{\tabcolsep}{0.6pt} %
    \renewcommand{\arraystretch}{1.0} %
    \begin{tabular}{ccc}
        \includegraphics[height=0.21\textheight]{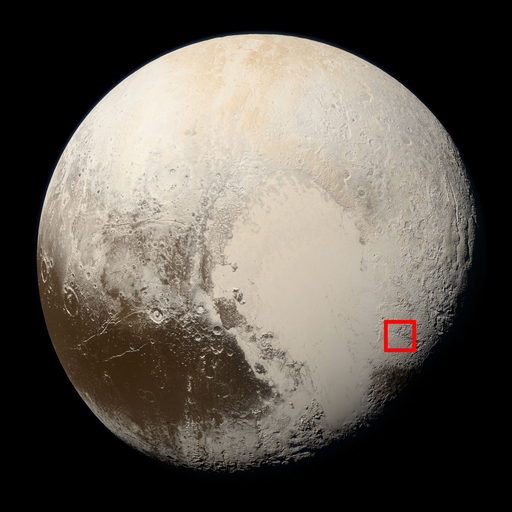} &
        \includegraphics[height=0.21\textheight]{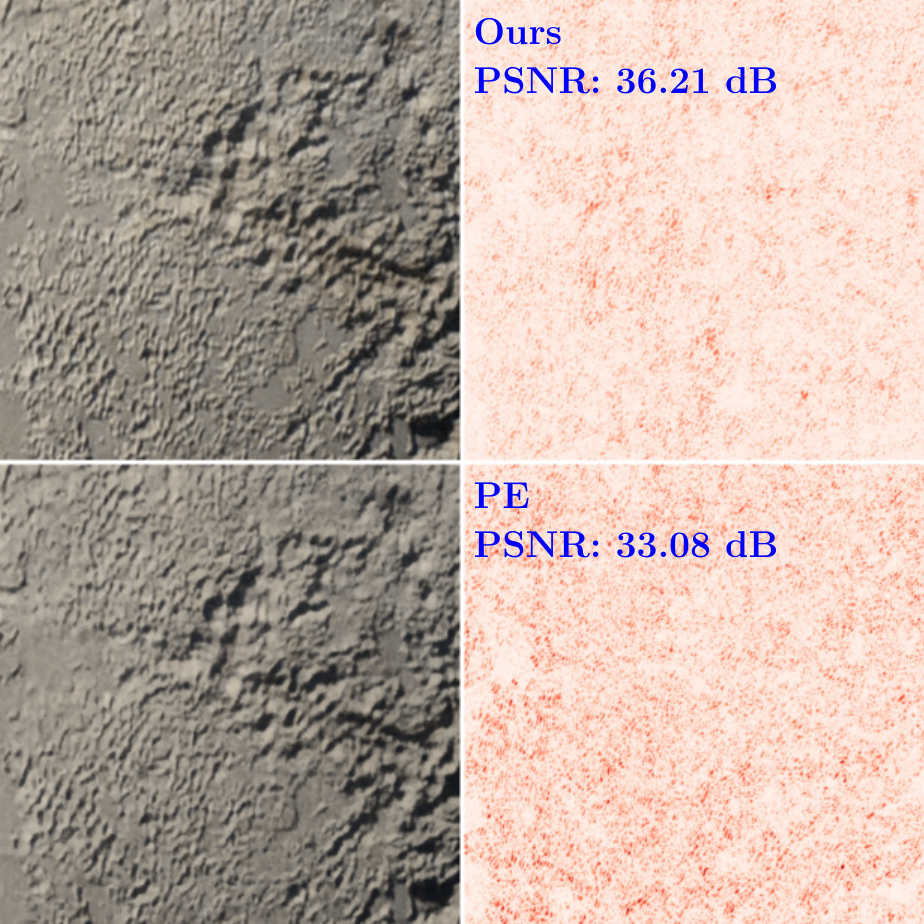} &
        \includegraphics[height=0.21\textheight]{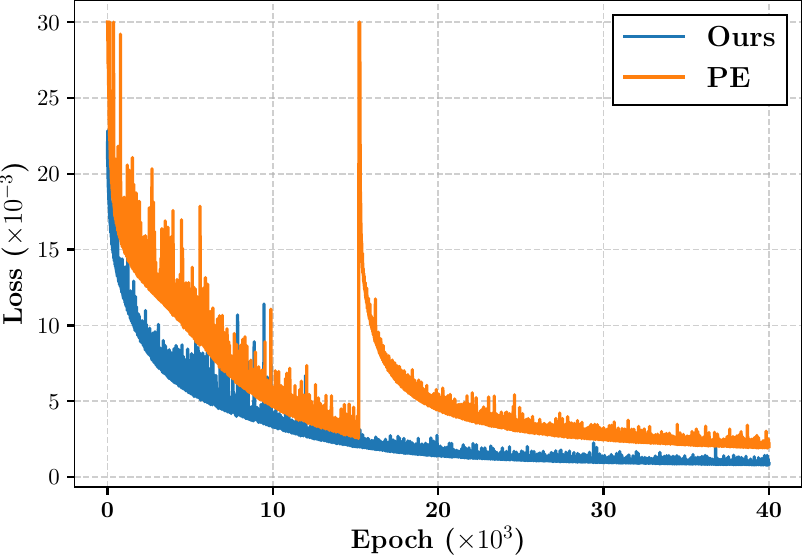} \\
        (a) GT & (b) Cropped patch with residual & (c) Training loss curve
    \end{tabular}
    \caption{Visualization of the gigapixel Pluto image fitting results. Subfigure~(a) shows the original $8000 \times 8000$ image, (b) presents the $512 \times 512$ patches fitted by our method and PE, along with the corresponding residual maps
    (c) displays the training loss curve.}
    \label{fig:three_images_tabular}
\end{figure*}
    \item \textbf{Inpainting.}
    All methods use hidden layers with 256 neurons and two hidden layers.
    SIREN and FINER use $\omega_0=30$; Wire uses $\omega_0=20$ and $s_0=20$;
    SCONE uses $\omega_0=30$ with $\omega_{\ell}=[60,\,30,\,10]$;
    SL$^2$A uses $\text{degree}=256$ and $\text{rank}=128$.
    For our method, we use two parallel linear layers with $M=96$ Fourier features, $J=32$ Chebyshev features, and an RFF scale of 30.

    \item \textbf{Super-resolution ($\times4$).}
    All methods use hidden layers with 256 neurons and two hidden layers.
    SIREN and FINER use $\omega_0=10$; Wire uses $\omega_0=8$ and $s_0=6$;
    SCONE uses $\omega_0=2$ with $\omega_{\ell}=[60,\,30,\,10]$;
    SL$^2$A uses $\text{degree}=64$ and $\text{rank}=64$.
    For our method, we use two parallel linear layers with $M=96$ Fourier features, $J=32$ Chebyshev features, and an RFF scale of 2.
\end{itemize}

For each task, we select one representative image from standard datasets and visualize the results in Fig.~\ref{fig:inverse_problem}. Our method also demonstrates strong performance across these image restoration tasks.

\subsection{Gigapixel Image
Representation}

We further demonstrate the superiority of the proposed CAFE+ on a gigapixel image representation task. Following ACORN~\cite{martel2021acorn}, we use a widely adopted $8000 \times 8000$ Pluto image and train for 40k iterations. For the baseline, the PE method employs an MLP with four hidden layers of 1,536 neurons each. In contrast, our method augments the original PE encoding by concatenating Chebyshev features with $J=8$, followed by two linear layers (each mapping 62 inputs to 1,536 outputs) whose outputs are combined via a Hadamard product before being fed into the backbone MLP. This modification allows us to use a backbone with only two hidden layers of 1,536 neurons while keeping all other settings identical to PE. As a result, the number of parameters is reduced from 9.52M to 4.92M. As shown in Fig.~\ref{fig:three_images_tabular}, our method reconstructs finer details, achieves a 3.13~dB PSNR improvement, and exhibits faster and more stable convergence. These results clearly demonstrate the effectiveness of our approach.

\subsection{Larger Scale NeRF and Scene-Level SDF Reconstruction}

\noindent\textbf{Larger Scale NeRF.} We further evaluate NeRF using all training images at the original resolution on the standard scenes (\textit{Lego} / \textit{Ship} / \textit{Drums} / \textit{Hotdog}). We compare our method with the two strongest baselines, SIREN and FINER. The quantitative results are shown in Table~\ref{tab:nerf_fullres}. 
\noindent{\textbf{Scene-Level SDF.}} We evaluate the methods on the \textit{Replica room0} scene using IoU as the metric under the same parameter budget. The results are: SIREN (0.854), FINER (0.844), and Ours (\textbf{0.899}).

\begin{table}[ht]
\centering
\caption{Quantitative comparison on NeRF scenes using all training images at the original resolution. PSNR is reported in dB.}
\label{tab:nerf_fullres}
\setlength{\tabcolsep}{6pt}
\resizebox{0.7\linewidth}{!}{
\begin{tabular}{lcccc}
\toprule
Method & Lego & Ship & Drums & Hotdog \\
\midrule
SIREN & 26.94 & 20.62 & 23.09 & 30.89 \\
FINER & 28.00 & 20.88 & \textbf{23.65} & 31.64 \\
Ours  & \textbf{30.26} & \textbf{21.92} & 23.51 & \textbf{32.35} \\
\bottomrule
\end{tabular}
}
\end{table}

\subsection{Comparison with MFN and BACON}

MFN, BACON, and our method all employ the Hadamard product, where MFN and BACON progressively synthesize frequency components across layers via a serial recursive architecture. In contrast, our approach is encoding-based and adopts a parallel design, which leads to faster training and removes the need for carefully modulated initialization. As summarized in Table~\ref{tab:dik_psnr_bacon}, our method achieves consistently better reconstruction quality while significantly reducing training time compared with MFN and BACON.

\begin{table}[ht]
\centering
\captionsetup{skip=2pt}
\caption{Comparison with BACON and MFN on Image Fitting}
\label{tab:dik_psnr_bacon}
\setlength{\tabcolsep}{2.4pt}
\resizebox{\columnwidth}{!}{
\footnotesize
\renewcommand{\arraystretch}{0.9}
\begin{tabular}{@{}lcccccccccc@{}}
\toprule
Method & D2K0 & D2K1 & D2K2 & D2K3 & D2K4 & D2K5 & D2K6 & D2K7 & Params & Time (s) \\
\midrule
MFN & 35.6 & 40.1 & 36.4 & 41.9 & 38.4 & 35.2 & 39.9 & 38.9 & 0.33M & 279.1 \\
BACON & 34.0 & 39.0 & 37.6 & 39.6 & 39.5 & 35.7 & 38.2 & 36.0 & 0.33M & 245.6 \\
Ours & \textbf{39.5} & \textbf{44.3} & \textbf{44.2} & \textbf{44.7} &
\textbf{45.0} & \textbf{39.2} & \textbf{42.6} & \textbf{45.0} &
0.33M & \textbf{149.8} \\
\bottomrule
\end{tabular}
}
\end{table}

\section{Additional Ablation Studies}\label{Additional Ablation}
\subsection{Comparison of Chebyshev and Fourier Features on 1D Function Fitting}

In Fig.~\ref{fig:1d}, we compare Fourier features and Chebyshev features (denoted as CF in the figure) by fitting two 1D functions to intuitively reveal their respective strengths. In our experimental setup, each coordinate is mapped to a 32-dimensional feature vector, and the RFF scale is set to 5. The MLP architecture consists of a single hidden layer with 32 neurons. We sample 1,000 points for training over 2,000 epochs and evaluate the fitting error on 4,000 uniformly sampled test points. As shown in the figure, Chebyshev features exhibit clear advantages when modeling smooth function regions, whereas Fourier features tend to introduce oscillations that lead to noisy predictions. Conversely, when approximating high-frequency functions, Chebyshev features incur significantly larger errors, while Fourier features maintain superior accuracy.
\begin{figure}[tbp]
    \centering
    \includegraphics[width=\linewidth]{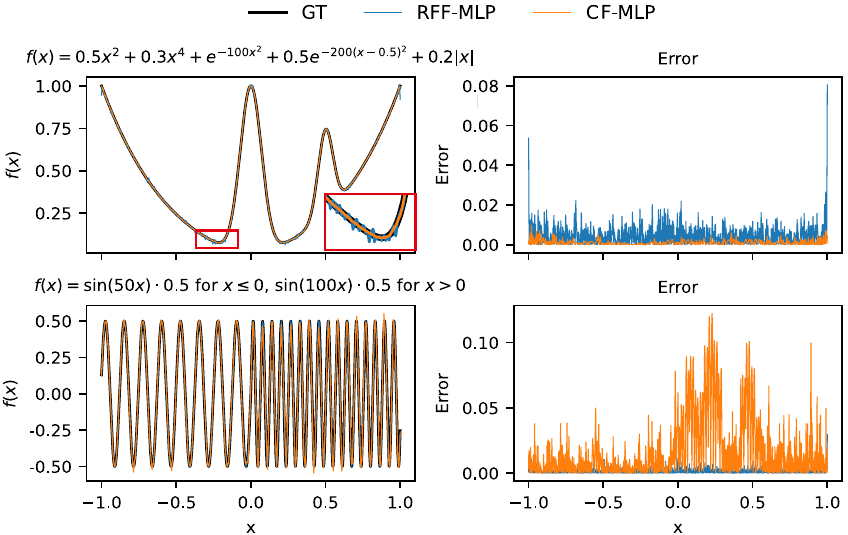}
    \caption{Comparison of RFF and CF in function fitting. CF excels at capturing smooth, low-frequency components but struggles with high-frequency variations, whereas RFF effectively represents high-frequency content yet introduces oscillatory errors on low-frequency functions.}
    \label{fig:1d}
\end{figure}

\subsection{Impact of High-Frequency Ratios on Reconstruction}

We introduce Chebyshev features with the primary motivation of enhancing the representation of low-frequency signals. 
However, this naturally raises an important concern regarding whether incorporating Chebyshev features may affect the fitting capability of CAFE when the signal contains a large proportion of high-frequency components. 
\begin{figure}[tbp]
    \centering
    \begin{subfigure}[b]{0.32\columnwidth}
        \includegraphics[width=\linewidth]{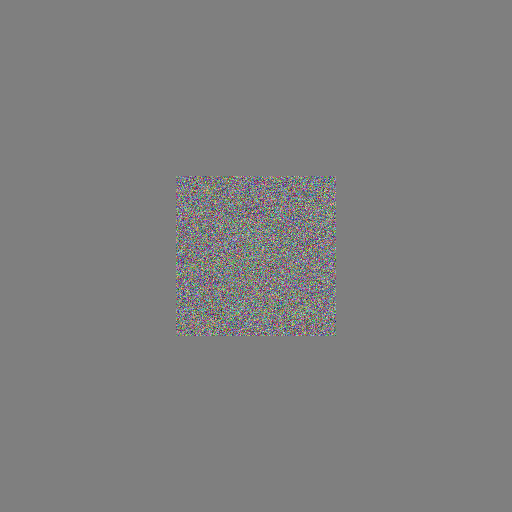}
        \caption{$\text{ratio=0.1}$}
    \end{subfigure}
    \begin{subfigure}[b]{0.32\columnwidth}
        \includegraphics[width=\linewidth]{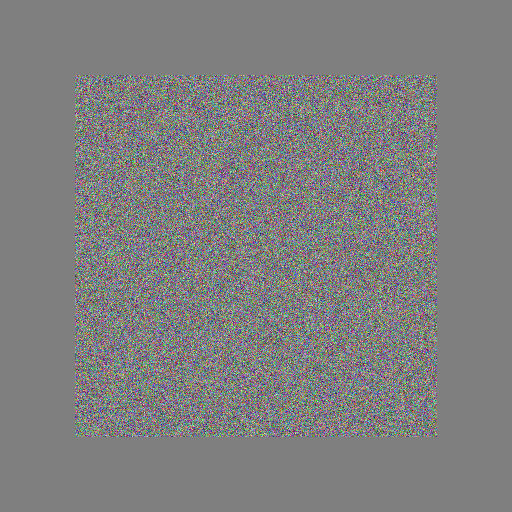}
        \caption{\text{ratio=0.5}}
    \end{subfigure}
    \begin{subfigure}[b]{0.32\columnwidth}
        \includegraphics[width=\linewidth]{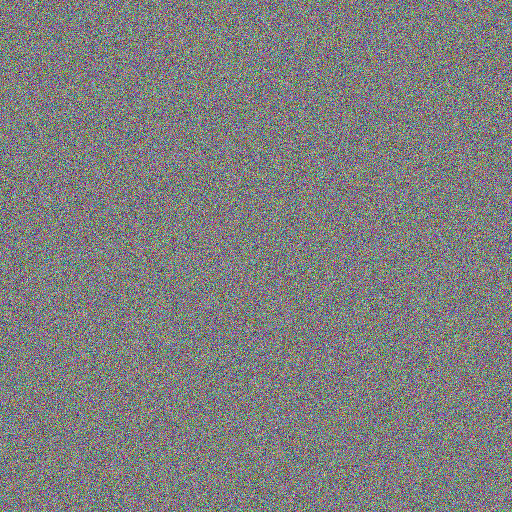}
        \caption{\text{ratio=1.0}}
    \end{subfigure}
    \caption{Visualization of the constructed noisy images under different noise ratios. }
\label{fig:noise_vis}
\end{figure}

\begin{figure}[ht]
    \centering
    \begin{tikzpicture}
    \begin{axis}[
        width=0.93\linewidth,
        height=5.5cm,        
        xlabel={Noise Ratio},
        ylabel={PSNR},
        label style={font=\small},
        tick label style={font=\small},       
        xmin=0, xmax=1,
        ymin=15, ymax=75, 
        xtick={0,0.1,0.2,0.3,0.4,0.5,0.6,0.7,0.8,0.9,1},        
        ymajorgrids=true,
        xmajorgrids=true,
        grid style={dashed, gray!40},         
        legend style={
            at={(0.5,1.02)}, 
            anchor=south,
            legend columns=2,
            font=\small,
            draw=none, 
            fill=none  
        },        
        every axis plot/.append style={thick}
    ]
    \addplot[
        color=blue,       
        mark=square*,     
        mark size=2.0pt, 
        mark options={solid} 
    ] coordinates {
        (0.1,68.51) (0.2,46.63) (0.3,35.24) (0.4,31.96) (0.5,25.92)
        (0.6,24.57) (0.7,22.95) (0.8,22.40) (0.9,21.64) (1.0,21.42)
    };
    \addlegendentry{CAFE}
    \addplot[
        color=red,       
        mark=*,           
        mark size=2.5pt, 
        dashed,           
        mark options={solid} 
    ] coordinates {
        (0.1,71.35) (0.2,55.25) (0.3,43.11) (0.4,34.64) (0.5,29.61)
        (0.6,26.66) (0.7,24.84) (0.8,23.14) (0.9,21.36) (1.0,21.35)
    };
    \addlegendentry{CAFE+}
    \end{axis}
    \end{tikzpicture}
    \caption{PSNR comparison of CAFE and CAFE+ across different noise ratios.}
    \label{cafe+ vs cafe}
\end{figure}
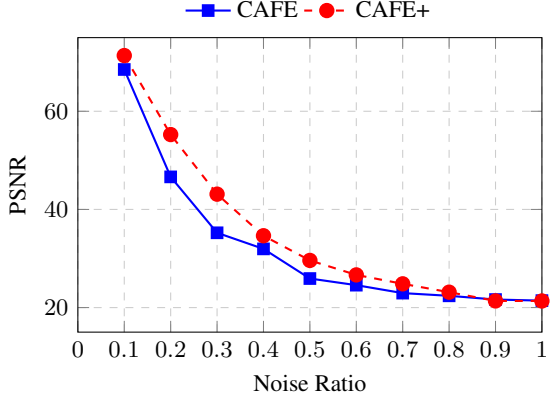

\begin{figure}[t]
\centering
\includegraphics[width=0.6\linewidth]{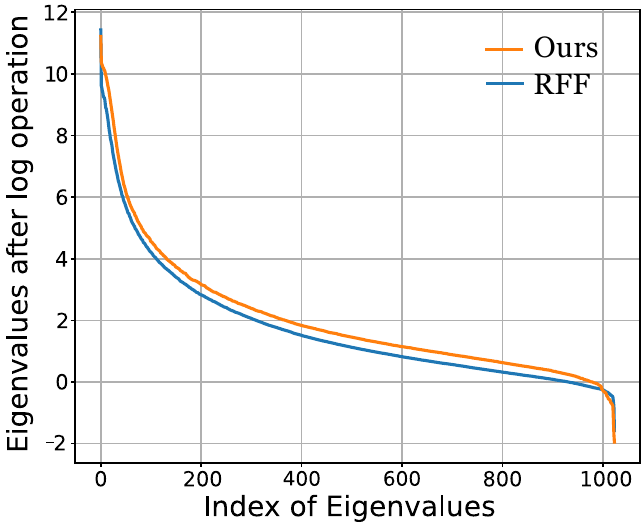}
\caption{NTK eigenvalues visualization.}
\label{fig:eigenvalues}
\end{figure}

To investigate this, we construct a synthetic dataset specifically designed to control the ratio of high-frequency content. As shown in Fig.~\ref{fig:noise_vis}, the dataset consists of a high-frequency central region and a flat background. Let $\mathbf{I}_{\text{noise}}$ denote a Gaussian noise map and $\mathbf{I}_{\text{smooth}} = \mathbf{0}$ represent a zero-filled background. We define a central binary mask $\mathbf{O}\rho$, where the parameter $\rho \in (0, 1]$ determines the ratio of the masked area to the total image area. The final ground-truth image $\mathbf{I}$ is synthesized as
\begin{equation*}
\mathbf{I} = \mathbf{O}\rho \odot \mathbf{I}_{\text{noise}} + (1 - \mathbf{O}\rho) \odot \mathbf{I}_{\text{smooth}},
\end{equation*}
where $\odot$ denotes element-wise multiplication. In our experiments, we vary $\rho$ from $0.1$ to $1.0$ to simulate different proportions of high-frequency information.

As illustrated in the comparison between CAFE+ and CAFE, CAFE+ consistently outperforms CAFE across a wide range of ratios. When the ratio exceeds 0.9, its performance becomes comparable to that of CAFE. This indicates that CAFE+ effectively enhances the representation of mixed spectral distributions without compromising performance in high-frequency dominant scenarios, demonstrating the practical applicability and robustness of our method. 

\subsection{NTK Eigenvalues}

We visualize the NTK eigenvalue of CAFE and RFF, as shown in Fig.~\ref{fig:eigenvalues}. CAFE exhibits a more favorable eigenvalue distribution.

\subsection{Loss Curve Comparison with SIREN}

We plot the training loss curves of our method and SIREN on the image fitting task, as shown in Fig.~\ref{fig:training_loss}. On the one hand, our method converges faster. On the other hand, even beyond the predefined training iterations (6,000 in our experiments), our method continues to achieve lower loss values. These results demonstrate the superior optimization efficiency of our framework.

\begin{figure}[t]
\centering
\includegraphics[width=0.8\linewidth]{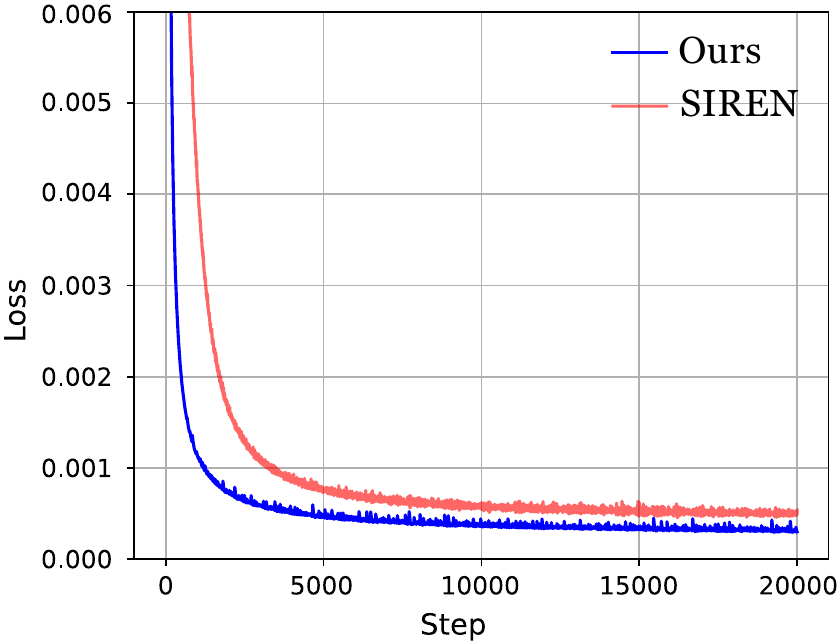}
\caption{Training loss curves of our method and SIREN. Our method converges faster and reaches a lower loss value.}
\label{fig:training_loss}
\end{figure}

{
\setlength{\tabcolsep}{2pt}  

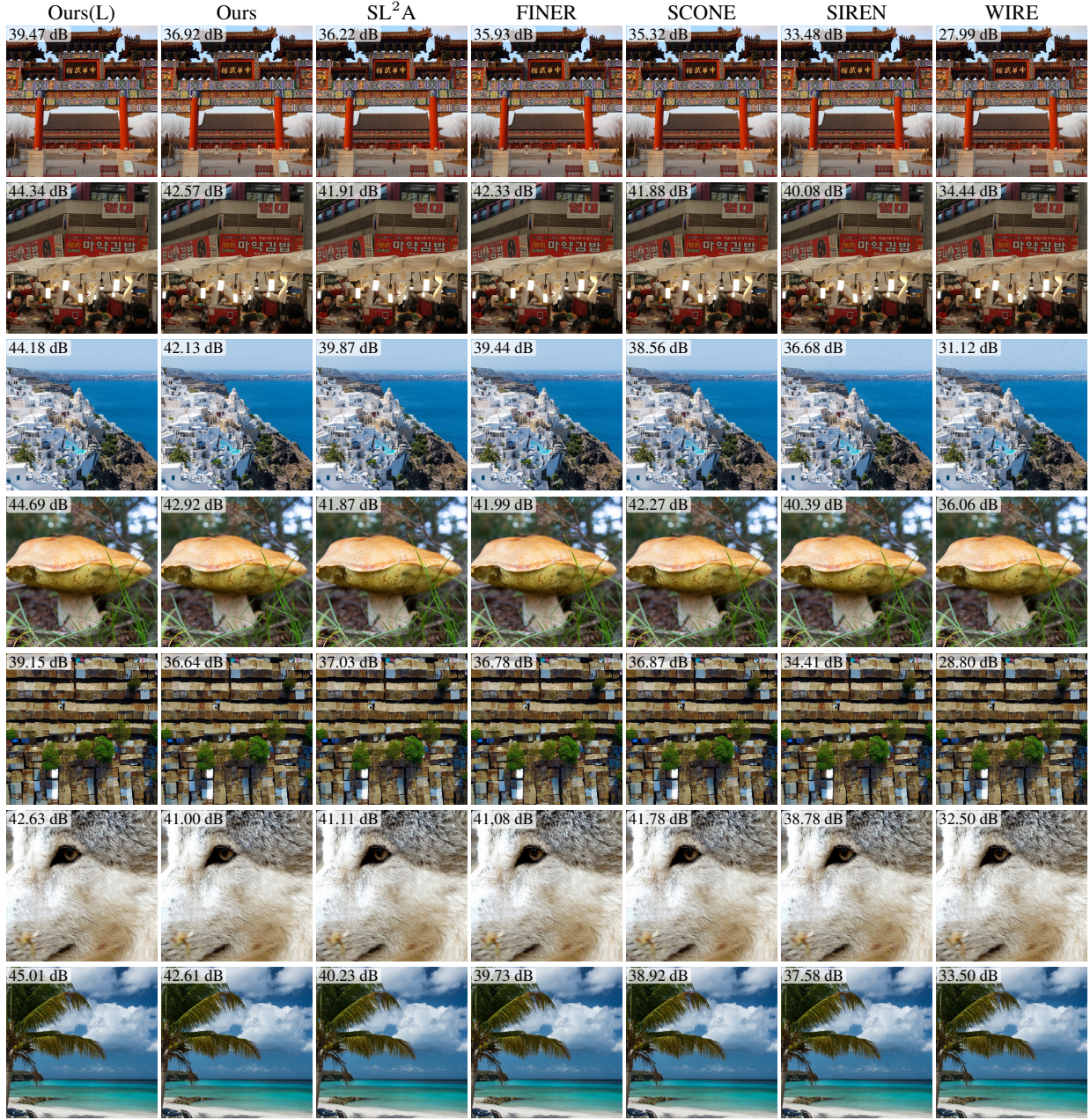
\begin{figure*}[htbp]
\renewcommand{\arraystretch}{0.8}
\setlength\tabcolsep{1pt}
\centering
\begin{tabular}{*{7}{>{\centering\arraybackslash}m{0.13\linewidth}}}
\small{Ours(L)} & \small{Ours} & 
 \small{SL$^2$A} &
\small{FINER} & \small{SCONE} & \small{SIREN} & \small{WIRE}  \\

\begin{tikzpicture}
\node[anchor=north west, inner sep=0] (image) at (0,0)
{\includegraphics[width=\linewidth]{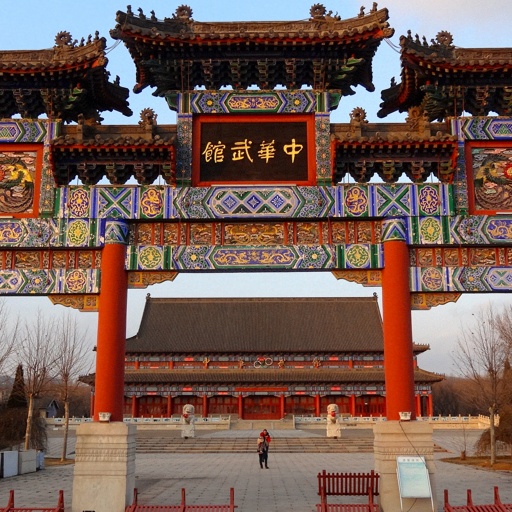}};
\begin{scope}[x={(image.south east)}, y={(image.north west)}]
\node[anchor=north west,
      fill=white, fill opacity=0.7, text opacity=1,
      inner sep=1pt, font=\scriptsize, rounded corners=1pt]
      at (0,1) {39.47 dB};
\end{scope}
\end{tikzpicture} &

\begin{tikzpicture}
\node[anchor=north west, inner sep=0] (image) at (0,0)
{\includegraphics[width=\linewidth]{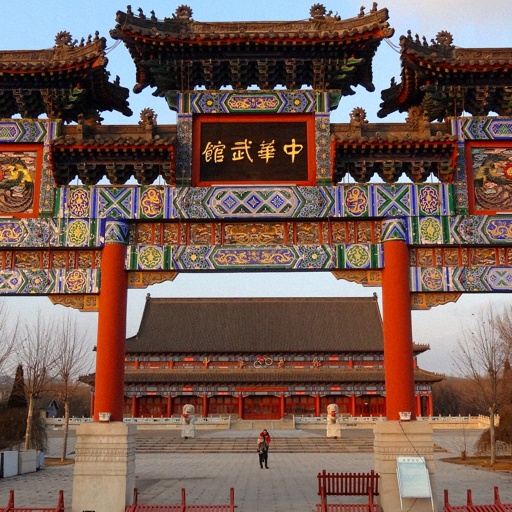}};
\begin{scope}[x={(image.south east)}, y={(image.north west)}]
\node[anchor=north west,
      fill=white, fill opacity=0.7, text opacity=1,
      inner sep=1pt, font=\scriptsize, rounded corners=1pt]
      at (0,1) {36.92 dB};
\end{scope}
\end{tikzpicture} &

\begin{tikzpicture}
\node[anchor=north west, inner sep=0] (image) at (0,0)
{\includegraphics[width=\linewidth]{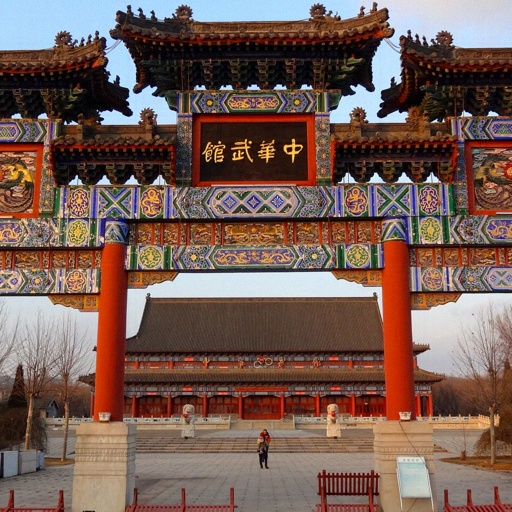}};
\begin{scope}[x={(image.south east)}, y={(image.north west)}]
\node[anchor=north west,
      fill=white, fill opacity=0.7, text opacity=1,
      inner sep=1pt, font=\scriptsize, rounded corners=1pt]
      at (0,1) {36.22 dB};
\end{scope}
\end{tikzpicture} &

\begin{tikzpicture}
\node[anchor=north west, inner sep=0] (image) at (0,0)
{\includegraphics[width=\linewidth]{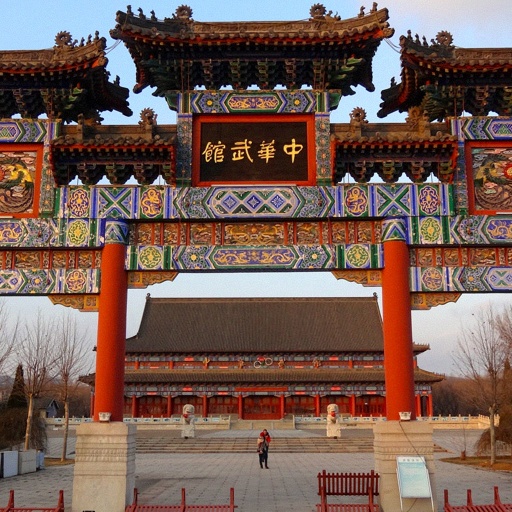}};
\begin{scope}[x={(image.south east)}, y={(image.north west)}]
\node[anchor=north west,
      fill=white, fill opacity=0.7, text opacity=1,
      inner sep=1pt, font=\scriptsize, rounded corners=1pt]
      at (0,1) {35.93 dB};
\end{scope}
\end{tikzpicture} &

\begin{tikzpicture}
\node[anchor=north west, inner sep=0] (image) at (0,0)
{\includegraphics[width=\linewidth]{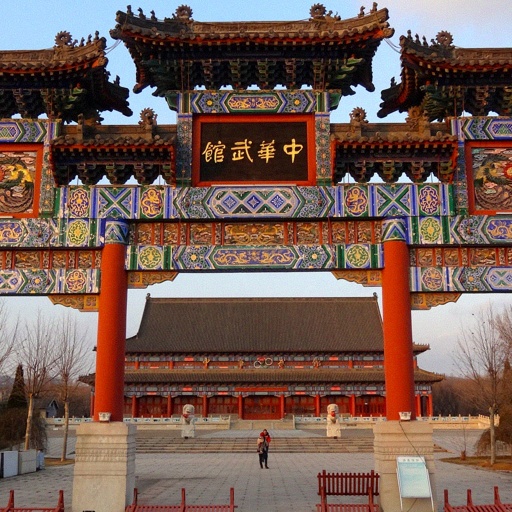}};
\begin{scope}[x={(image.south east)}, y={(image.north west)}]
\node[anchor=north west,
      fill=white, fill opacity=0.7, text opacity=1,
      inner sep=1pt, font=\scriptsize, rounded corners=1pt]
      at (0,1) {35.32 dB};
\end{scope}
\end{tikzpicture} &

\begin{tikzpicture}
\node[anchor=north west, inner sep=0] (image) at (0,0)
{\includegraphics[width=\linewidth]{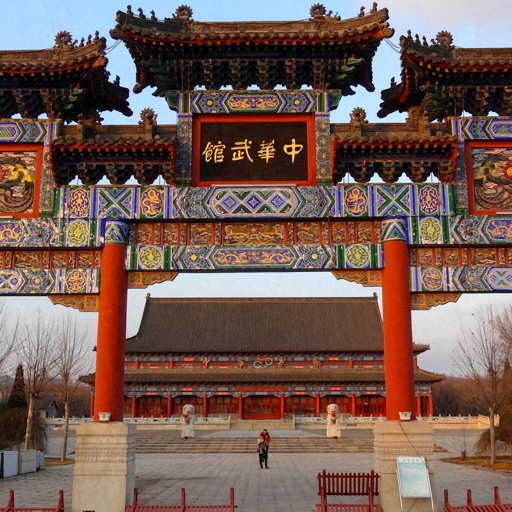}};
\begin{scope}[x={(image.south east)}, y={(image.north west)}]
\node[anchor=north west,
      fill=white, fill opacity=0.7, text opacity=1,
      inner sep=1pt, font=\scriptsize, rounded corners=1pt]
      at (0,1) {33.48 dB};
\end{scope}
\end{tikzpicture} &

\begin{tikzpicture}
\node[anchor=north west, inner sep=0] (image) at (0,0)
{\includegraphics[width=\linewidth]{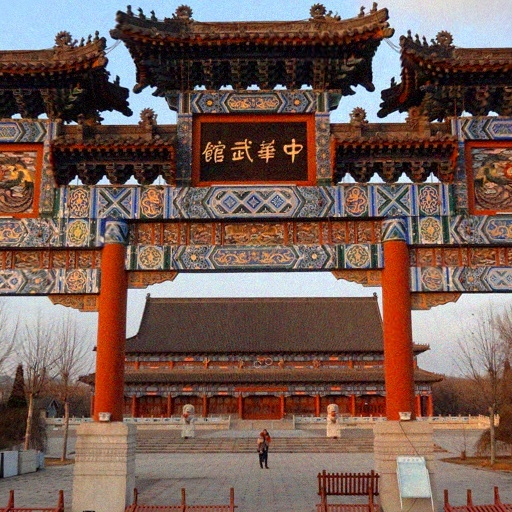}};
\begin{scope}[x={(image.south east)}, y={(image.north west)}]
\node[anchor=north west,
      fill=white, fill opacity=0.7, text opacity=1,
      inner sep=1pt, font=\scriptsize, rounded corners=1pt]
      at (0,1) {27.99 dB};
\end{scope}
\end{tikzpicture}
\\


\begin{tikzpicture}
\node[anchor=north west, inner sep=0] (image) at (0,0)
{\includegraphics[width=\linewidth]{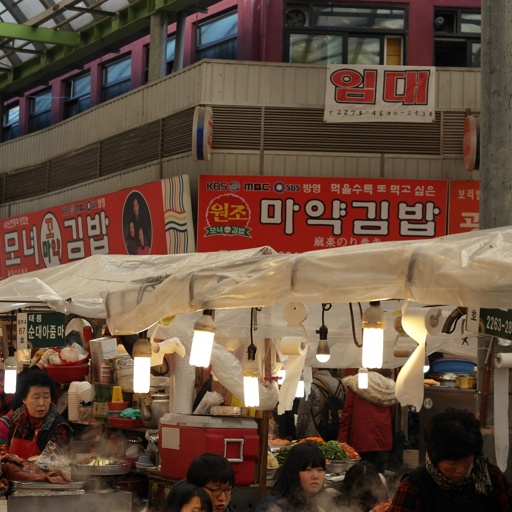}};
\begin{scope}[x={(image.south east)}, y={(image.north west)}]
\node[anchor=north west,
      fill=white, fill opacity=0.7, text opacity=1,
      inner sep=1pt, font=\scriptsize, rounded corners=1pt]
      at (0,1) {44.34 dB};
\end{scope}
\end{tikzpicture} &

\begin{tikzpicture}
\node[anchor=north west, inner sep=0] (image) at (0,0)
{\includegraphics[width=\linewidth]{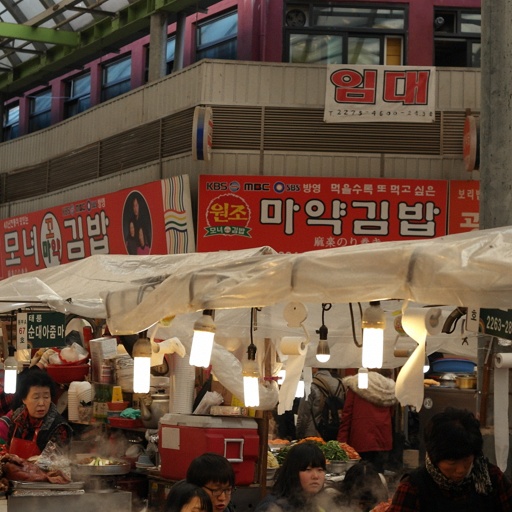}};
\begin{scope}[x={(image.south east)}, y={(image.north west)}]
\node[anchor=north west,
      fill=white, fill opacity=0.7, text opacity=1,
      inner sep=1pt, font=\scriptsize, rounded corners=1pt]
      at (0,1) {42.57 dB};
\end{scope}
\end{tikzpicture} &

\begin{tikzpicture}
\node[anchor=north west, inner sep=0] (image) at (0,0)
{\includegraphics[width=\linewidth]{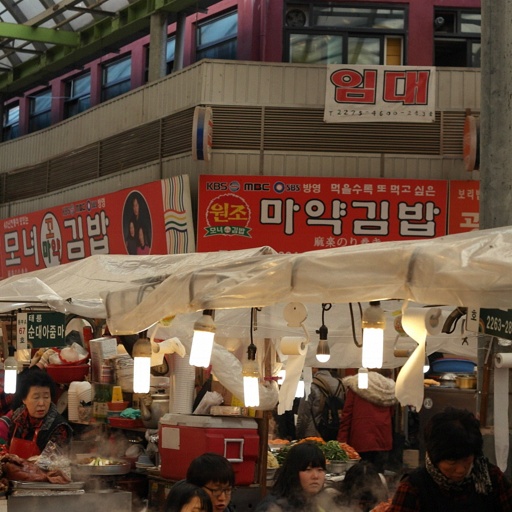}};
\begin{scope}[x={(image.south east)}, y={(image.north west)}]
\node[anchor=north west,
      fill=white, fill opacity=0.7, text opacity=1,
      inner sep=1pt, font=\scriptsize, rounded corners=1pt]
      at (0,1) {41.91 dB};
\end{scope}
\end{tikzpicture} &

\begin{tikzpicture}
\node[anchor=north west, inner sep=0] (image) at (0,0)
{\includegraphics[width=\linewidth]{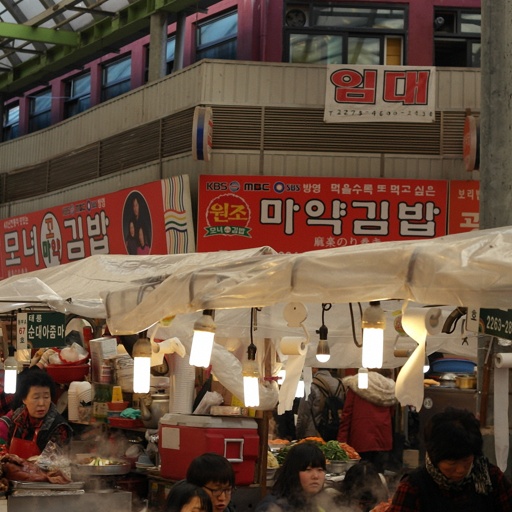}};
\begin{scope}[x={(image.south east)}, y={(image.north west)}]
\node[anchor=north west,
      fill=white, fill opacity=0.7, text opacity=1,
      inner sep=1pt, font=\scriptsize, rounded corners=1pt]
      at (0,1) {42.33 dB};
\end{scope}
\end{tikzpicture} &

\begin{tikzpicture}
\node[anchor=north west, inner sep=0] (image) at (0,0)
{\includegraphics[width=\linewidth]{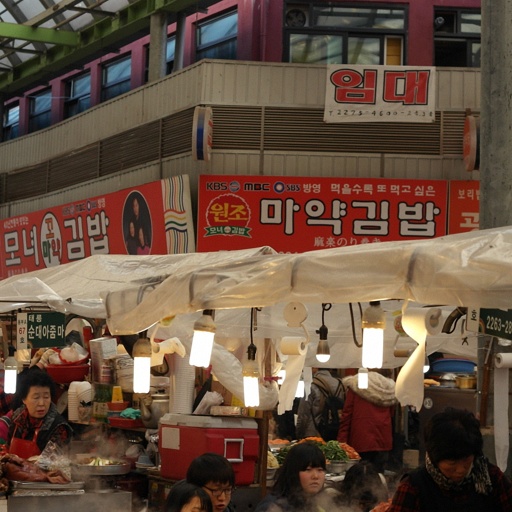}};
\begin{scope}[x={(image.south east)}, y={(image.north west)}]
\node[anchor=north west,
      fill=white, fill opacity=0.7, text opacity=1,
      inner sep=1pt, font=\scriptsize, rounded corners=1pt]
      at (0,1) {41.88 dB};
\end{scope}
\end{tikzpicture} &

\begin{tikzpicture}
\node[anchor=north west, inner sep=0] (image) at (0,0)
{\includegraphics[width=\linewidth]{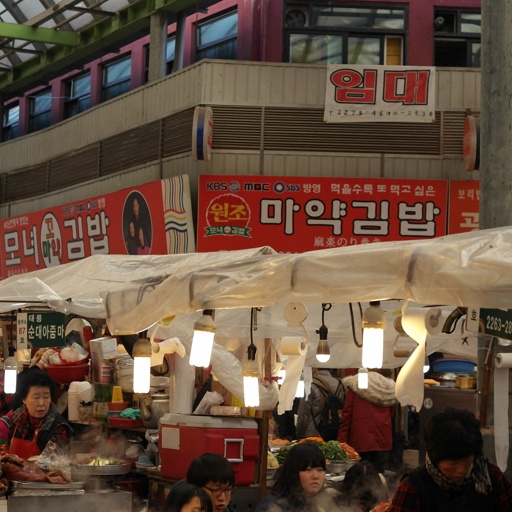}};
\begin{scope}[x={(image.south east)}, y={(image.north west)}]
\node[anchor=north west,
      fill=white, fill opacity=0.7, text opacity=1,
      inner sep=1pt, font=\scriptsize, rounded corners=1pt]
      at (0,1) {40.08 dB};
\end{scope}
\end{tikzpicture} &

\begin{tikzpicture}
\node[anchor=north west, inner sep=0] (image) at (0,0)
{\includegraphics[width=\linewidth]{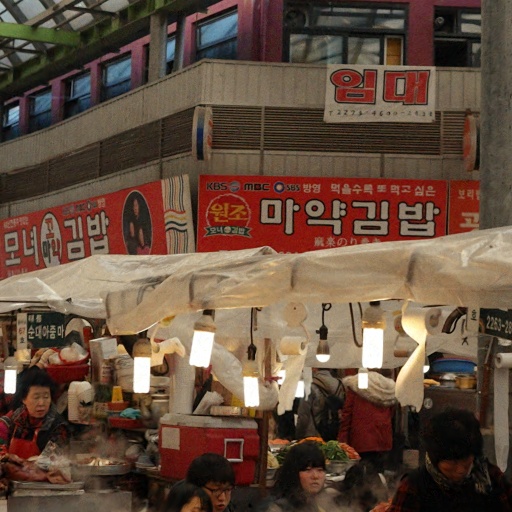}};
\begin{scope}[x={(image.south east)}, y={(image.north west)}]
\node[anchor=north west,
      fill=white, fill opacity=0.7, text opacity=1,
      inner sep=1pt, font=\scriptsize, rounded corners=1pt]
      at (0,1) {34.44 dB};
\end{scope}
\end{tikzpicture}
\\


\begin{tikzpicture}
\node[anchor=north west, inner sep=0] (image) at (0,0)
{\includegraphics[width=\linewidth]{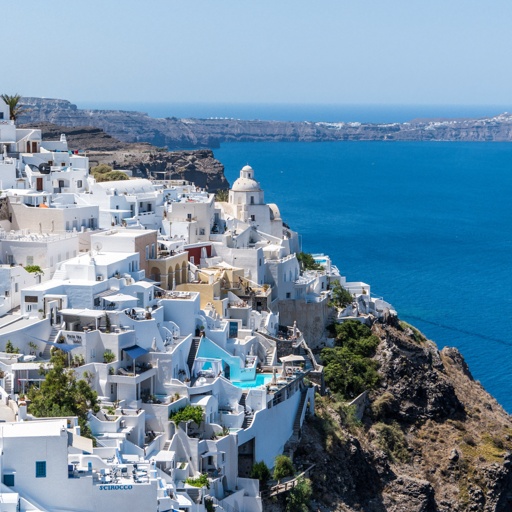}};
\begin{scope}[x={(image.south east)}, y={(image.north west)}]
\node[anchor=north west,
      fill=white, fill opacity=0.7, text opacity=1,
      inner sep=1pt, font=\scriptsize, rounded corners=1pt]
      at (0,1) {44.18 dB};
\end{scope}
\end{tikzpicture} &

\begin{tikzpicture}
\node[anchor=north west, inner sep=0] (image) at (0,0)
{\includegraphics[width=\linewidth]{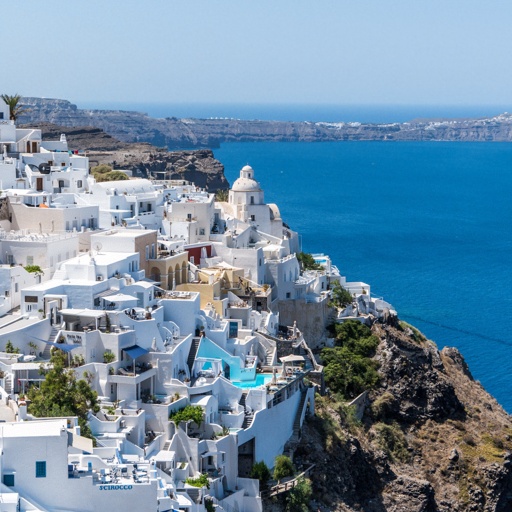}};
\begin{scope}[x={(image.south east)}, y={(image.north west)}]
\node[anchor=north west,
      fill=white, fill opacity=0.7, text opacity=1,
      inner sep=1pt, font=\scriptsize, rounded corners=1pt]
      at (0,1) {42.13 dB};
\end{scope}
\end{tikzpicture} &

\begin{tikzpicture}
\node[anchor=north west, inner sep=0] (image) at (0,0)
{\includegraphics[width=\linewidth]{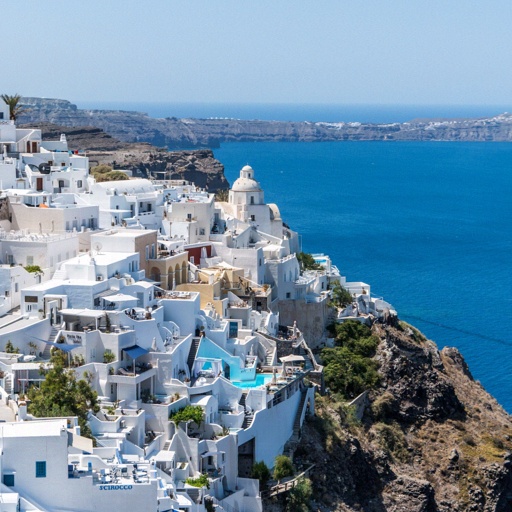}};
\begin{scope}[x={(image.south east)}, y={(image.north west)}]
\node[anchor=north west,
      fill=white, fill opacity=0.7, text opacity=1,
      inner sep=1pt, font=\scriptsize, rounded corners=1pt]
      at (0,1) {39.87 dB};
\end{scope}
\end{tikzpicture} &

\begin{tikzpicture}
\node[anchor=north west, inner sep=0] (image) at (0,0)
{\includegraphics[width=\linewidth]{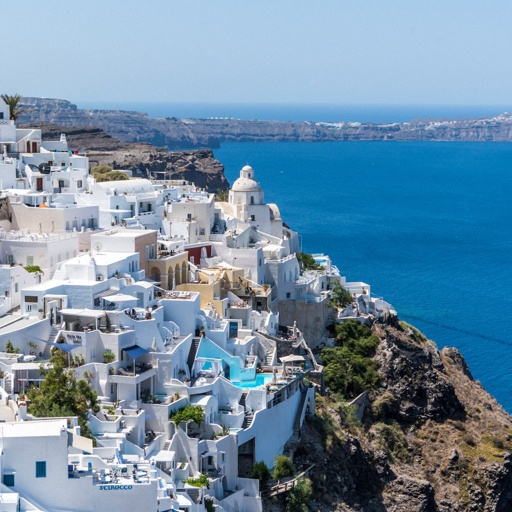}};
\begin{scope}[x={(image.south east)}, y={(image.north west)}]
\node[anchor=north west,
      fill=white, fill opacity=0.7, text opacity=1,
      inner sep=1pt, font=\scriptsize, rounded corners=1pt]
      at (0,1) {39.44 dB};
\end{scope}
\end{tikzpicture} &

\begin{tikzpicture}
\node[anchor=north west, inner sep=0] (image) at (0,0)
{\includegraphics[width=\linewidth]{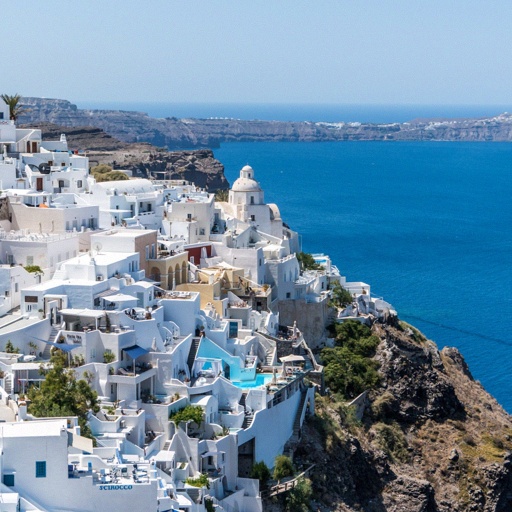}};
\begin{scope}[x={(image.south east)}, y={(image.north west)}]
\node[anchor=north west,
      fill=white, fill opacity=0.7, text opacity=1,
      inner sep=1pt, font=\scriptsize, rounded corners=1pt]
      at (0,1) {38.56 dB};
\end{scope}
\end{tikzpicture} &

\begin{tikzpicture}
\node[anchor=north west, inner sep=0] (image) at (0,0)
{\includegraphics[width=\linewidth]{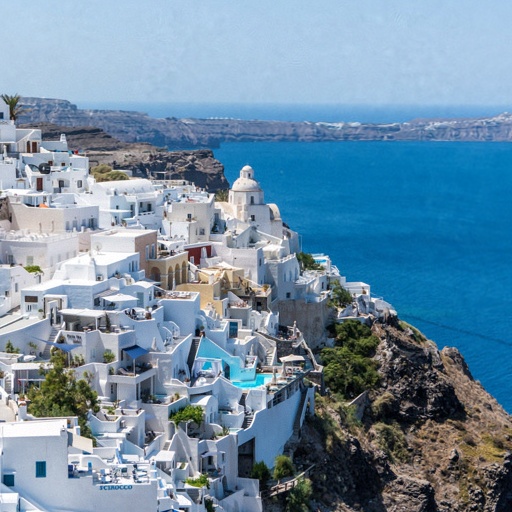}};
\begin{scope}[x={(image.south east)}, y={(image.north west)}]
\node[anchor=north west,
      fill=white, fill opacity=0.7, text opacity=1,
      inner sep=1pt, font=\scriptsize, rounded corners=1pt]
      at (0,1) {36.68 dB};
\end{scope}
\end{tikzpicture} &

\begin{tikzpicture}
\node[anchor=north west, inner sep=0] (image) at (0,0)
{\includegraphics[width=\linewidth]{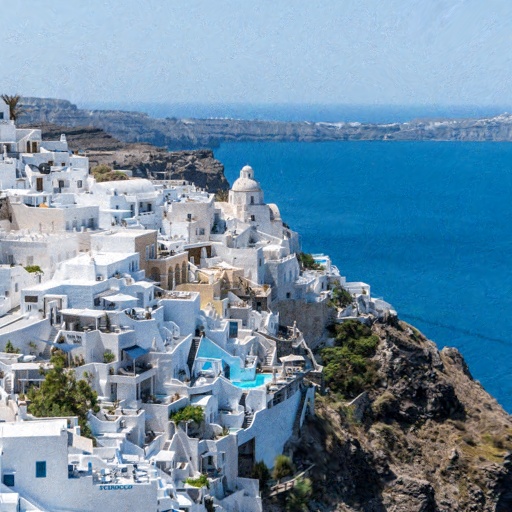}};
\begin{scope}[x={(image.south east)}, y={(image.north west)}]
\node[anchor=north west,
      fill=white, fill opacity=0.7, text opacity=1,
      inner sep=1pt, font=\scriptsize, rounded corners=1pt]
      at (0,1) {31.12 dB};
\end{scope}
\end{tikzpicture}
\\


\begin{tikzpicture}
\node[anchor=north west, inner sep=0] (image) at (0,0)
{\includegraphics[width=\linewidth]{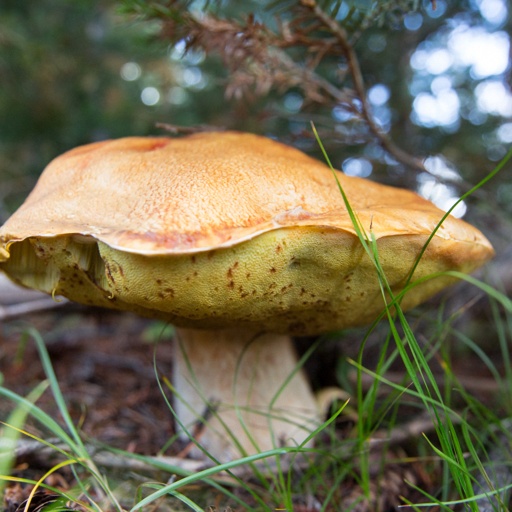}};
\begin{scope}[x={(image.south east)}, y={(image.north west)}]
\node[anchor=north west,
      fill=white, fill opacity=0.7, text opacity=1,
      inner sep=1pt, font=\scriptsize, rounded corners=1pt]
      at (0,1) {44.69 dB};
\end{scope}
\end{tikzpicture} &

\begin{tikzpicture}
\node[anchor=north west, inner sep=0] (image) at (0,0)
{\includegraphics[width=\linewidth]{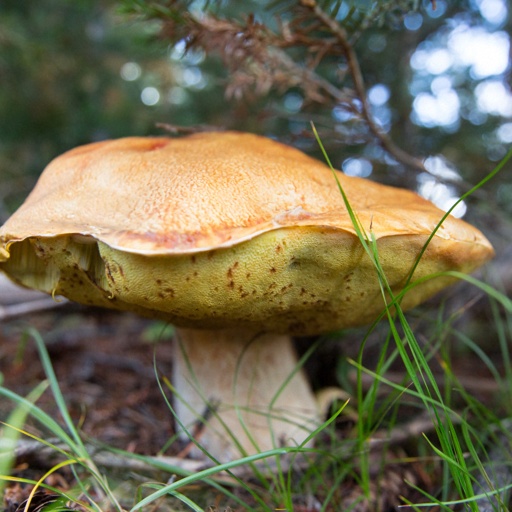}};
\begin{scope}[x={(image.south east)}, y={(image.north west)}]
\node[anchor=north west,
      fill=white, fill opacity=0.7, text opacity=1,
      inner sep=1pt, font=\scriptsize, rounded corners=1pt]
      at (0,1) {42.92 dB};
\end{scope}
\end{tikzpicture} &

\begin{tikzpicture}
\node[anchor=north west, inner sep=0] (image) at (0,0)
{\includegraphics[width=\linewidth]{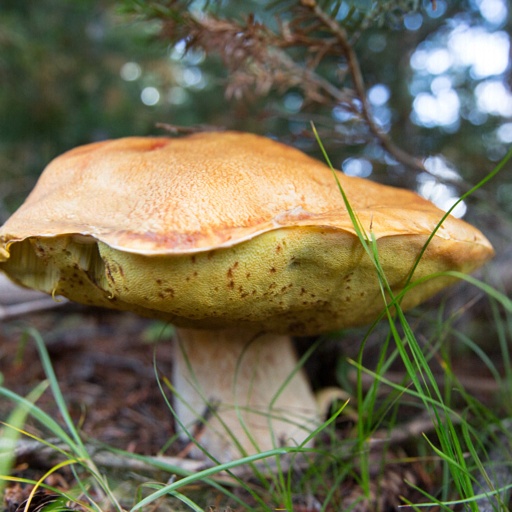}};
\begin{scope}[x={(image.south east)}, y={(image.north west)}]
\node[anchor=north west,
      fill=white, fill opacity=0.7, text opacity=1,
      inner sep=1pt, font=\scriptsize, rounded corners=1pt]
      at (0,1) {41.87 dB};
\end{scope}
\end{tikzpicture} &

\begin{tikzpicture}
\node[anchor=north west, inner sep=0] (image) at (0,0)
{\includegraphics[width=\linewidth]{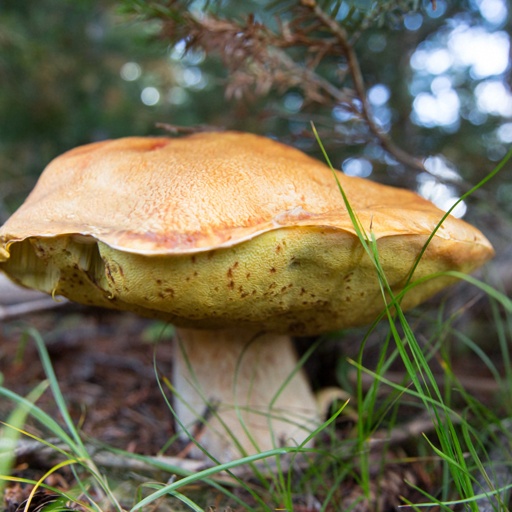}};
\begin{scope}[x={(image.south east)}, y={(image.north west)}]
\node[anchor=north west,
      fill=white, fill opacity=0.7, text opacity=1,
      inner sep=1pt, font=\scriptsize, rounded corners=1pt]
      at (0,1) {41.99 dB};
\end{scope}
\end{tikzpicture} &

\begin{tikzpicture}
\node[anchor=north west, inner sep=0] (image) at (0,0)
{\includegraphics[width=\linewidth]{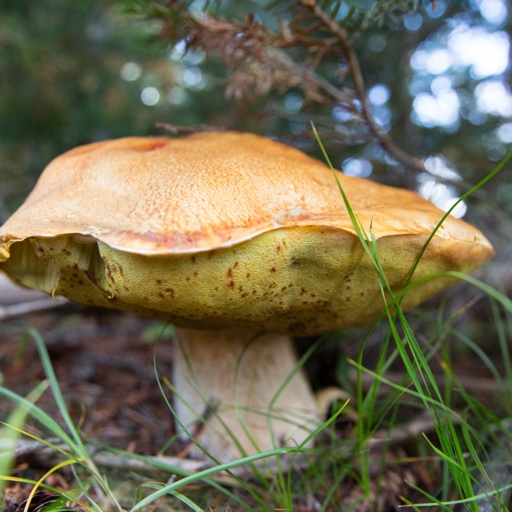}};
\begin{scope}[x={(image.south east)}, y={(image.north west)}]
\node[anchor=north west,
      fill=white, fill opacity=0.7, text opacity=1,
      inner sep=1pt, font=\scriptsize, rounded corners=1pt]
      at (0,1) {42.27 dB};
\end{scope}
\end{tikzpicture} &

\begin{tikzpicture}
\node[anchor=north west, inner sep=0] (image) at (0,0)
{\includegraphics[width=\linewidth]{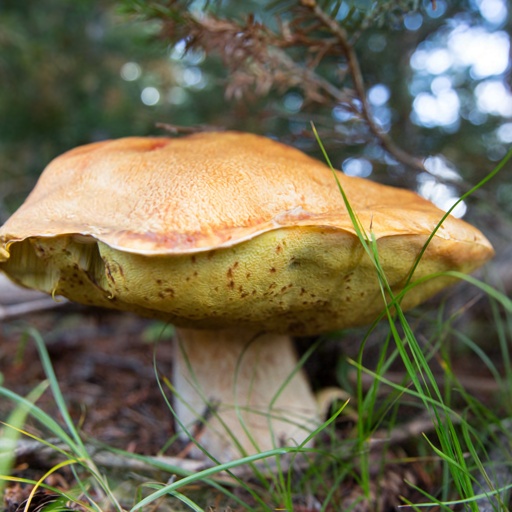}};
\begin{scope}[x={(image.south east)}, y={(image.north west)}]
\node[anchor=north west,
      fill=white, fill opacity=0.7, text opacity=1,
      inner sep=1pt, font=\scriptsize, rounded corners=1pt]
      at (0,1) {40.39 dB};
\end{scope}
\end{tikzpicture} &

\begin{tikzpicture}
\node[anchor=north west, inner sep=0] (image) at (0,0)
{\includegraphics[width=\linewidth]{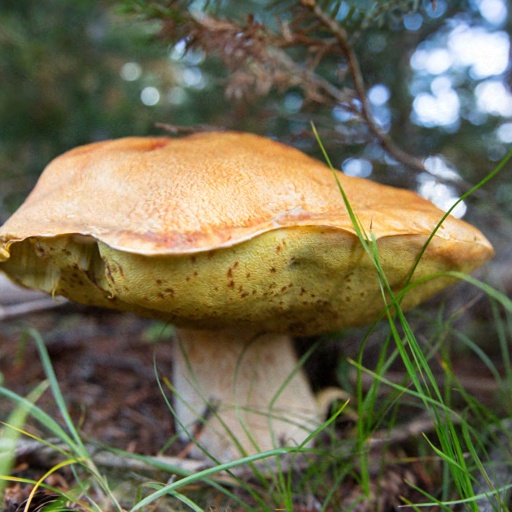}};
\begin{scope}[x={(image.south east)}, y={(image.north west)}]
\node[anchor=north west,
      fill=white, fill opacity=0.7, text opacity=1,
      inner sep=1pt, font=\scriptsize, rounded corners=1pt]
      at (0,1) {36.06 dB};
\end{scope}
\end{tikzpicture}
\\


\begin{tikzpicture}
\node[anchor=north west, inner sep=0] (image) at (0,0)
{\includegraphics[width=\linewidth]{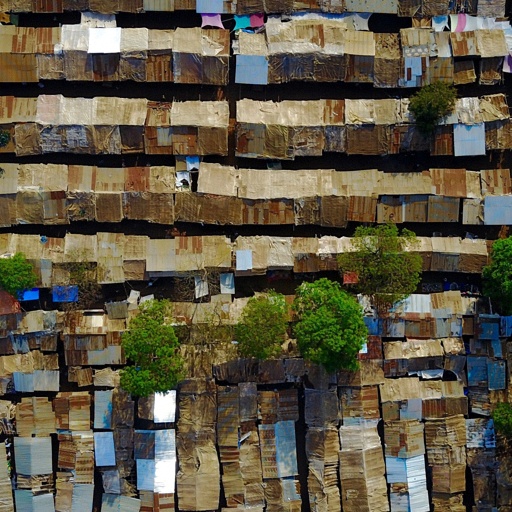}};
\begin{scope}[x={(image.south east)}, y={(image.north west)}]
\node[anchor=north west,
      fill=white, fill opacity=0.7, text opacity=1,
      inner sep=1pt, font=\scriptsize, rounded corners=1pt]
      at (0,1) {39.15 dB};
\end{scope}
\end{tikzpicture} &

\begin{tikzpicture}
\node[anchor=north west, inner sep=0] (image) at (0,0)
{\includegraphics[width=\linewidth]{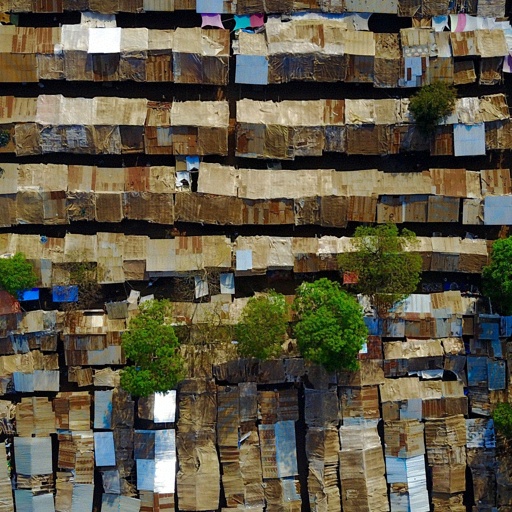}};
\begin{scope}[x={(image.south east)}, y={(image.north west)}]
\node[anchor=north west,
      fill=white, fill opacity=0.7, text opacity=1,
      inner sep=1pt, font=\scriptsize, rounded corners=1pt]
      at (0,1) {36.64 dB};
\end{scope}
\end{tikzpicture} &

\begin{tikzpicture}
\node[anchor=north west, inner sep=0] (image) at (0,0)
{\includegraphics[width=\linewidth]{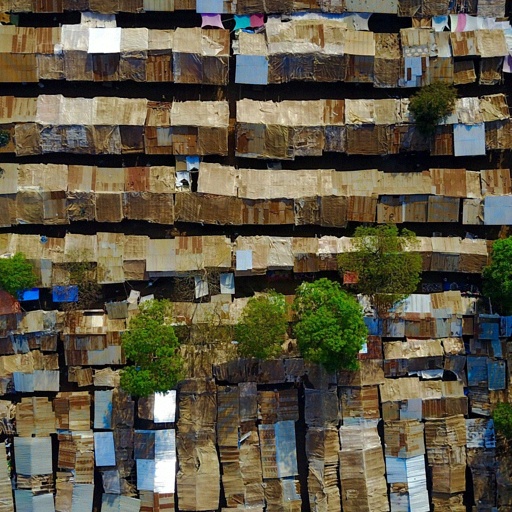}};
\begin{scope}[x={(image.south east)}, y={(image.north west)}]
\node[anchor=north west,
      fill=white, fill opacity=0.7, text opacity=1,
      inner sep=1pt, font=\scriptsize, rounded corners=1pt]
      at (0,1) {37.03 dB};
\end{scope}
\end{tikzpicture} &

\begin{tikzpicture}
\node[anchor=north west, inner sep=0] (image) at (0,0)
{\includegraphics[width=\linewidth]{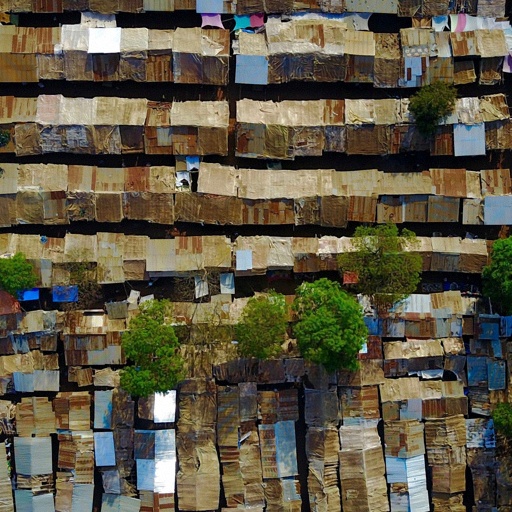}};
\begin{scope}[x={(image.south east)}, y={(image.north west)}]
\node[anchor=north west,
      fill=white, fill opacity=0.7, text opacity=1,
      inner sep=1pt, font=\scriptsize, rounded corners=1pt]
      at (0,1) {36.78 dB};
\end{scope}
\end{tikzpicture} &

\begin{tikzpicture}
\node[anchor=north west, inner sep=0] (image) at (0,0)
{\includegraphics[width=\linewidth]{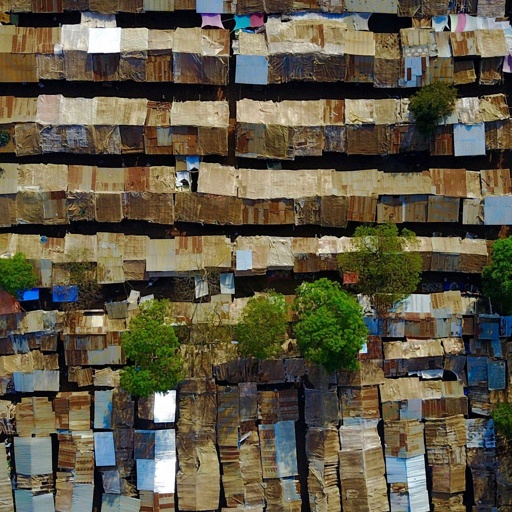}};
\begin{scope}[x={(image.south east)}, y={(image.north west)}]
\node[anchor=north west,
      fill=white, fill opacity=0.7, text opacity=1,
      inner sep=1pt, font=\scriptsize, rounded corners=1pt]
      at (0,1) {36.87 dB};
\end{scope}
\end{tikzpicture} &

\begin{tikzpicture}
\node[anchor=north west, inner sep=0] (image) at (0,0)
{\includegraphics[width=\linewidth]{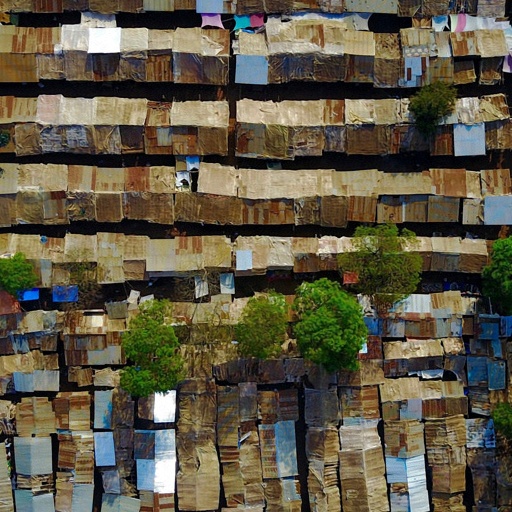}};
\begin{scope}[x={(image.south east)}, y={(image.north west)}]
\node[anchor=north west,
      fill=white, fill opacity=0.7, text opacity=1,
      inner sep=1pt, font=\scriptsize, rounded corners=1pt]
      at (0,1) {34.41 dB};
\end{scope}
\end{tikzpicture} &

\begin{tikzpicture}
\node[anchor=north west, inner sep=0] (image) at (0,0)
{\includegraphics[width=\linewidth]{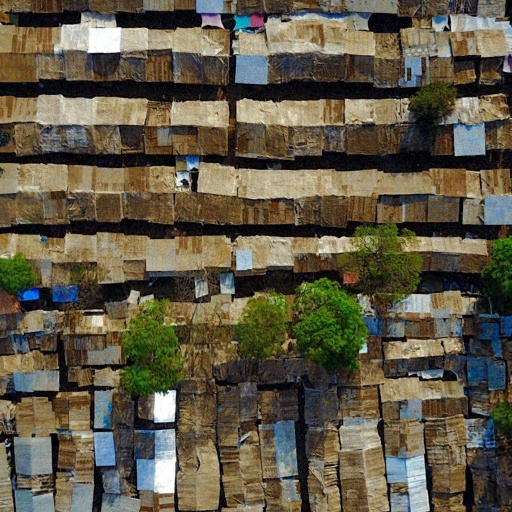}};
\begin{scope}[x={(image.south east)}, y={(image.north west)}]
\node[anchor=north west,
      fill=white, fill opacity=0.7, text opacity=1,
      inner sep=1pt, font=\scriptsize, rounded corners=1pt]
      at (0,1) {28.80 dB};
\end{scope}
\end{tikzpicture}
\\


\begin{tikzpicture}
\node[anchor=north west, inner sep=0] (image) at (0,0)
{\includegraphics[width=\linewidth]{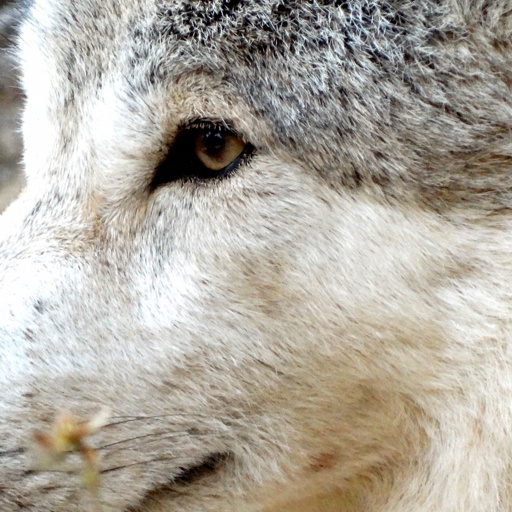}};
\begin{scope}[x={(image.south east)}, y={(image.north west)}]
\node[anchor=north west,
      fill=white, fill opacity=0.7, text opacity=1,
      inner sep=1pt, font=\scriptsize, rounded corners=1pt]
      at (0,1) {42.63 dB};
\end{scope}
\end{tikzpicture} &

\begin{tikzpicture}
\node[anchor=north west, inner sep=0] (image) at (0,0)
{\includegraphics[width=\linewidth]{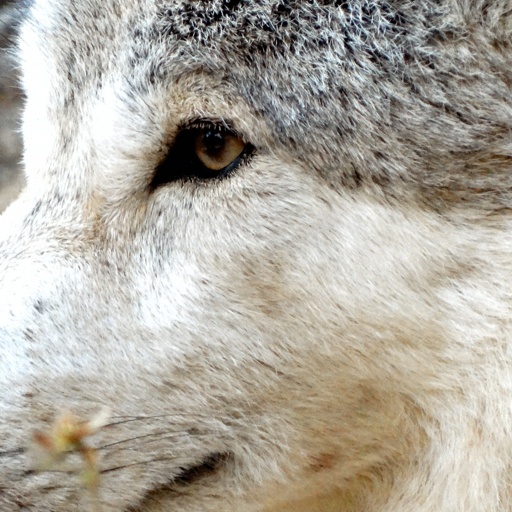}};
\begin{scope}[x={(image.south east)}, y={(image.north west)}]
\node[anchor=north west,
      fill=white, fill opacity=0.7, text opacity=1,
      inner sep=1pt, font=\scriptsize, rounded corners=1pt]
      at (0,1) {41.00 dB};
\end{scope}
\end{tikzpicture} &

\begin{tikzpicture}
\node[anchor=north west, inner sep=0] (image) at (0,0)
{\includegraphics[width=\linewidth]{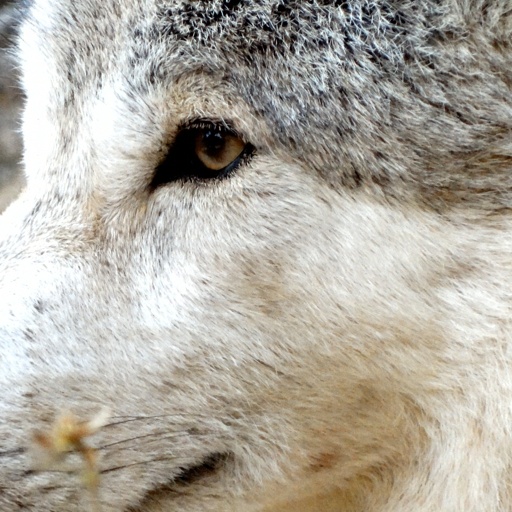}};
\begin{scope}[x={(image.south east)}, y={(image.north west)}]
\node[anchor=north west,
      fill=white, fill opacity=0.7, text opacity=1,
      inner sep=1pt, font=\scriptsize, rounded corners=1pt]
      at (0,1) {41.11 dB};
\end{scope}
\end{tikzpicture} &

\begin{tikzpicture}
\node[anchor=north west, inner sep=0] (image) at (0,0)
{\includegraphics[width=\linewidth]{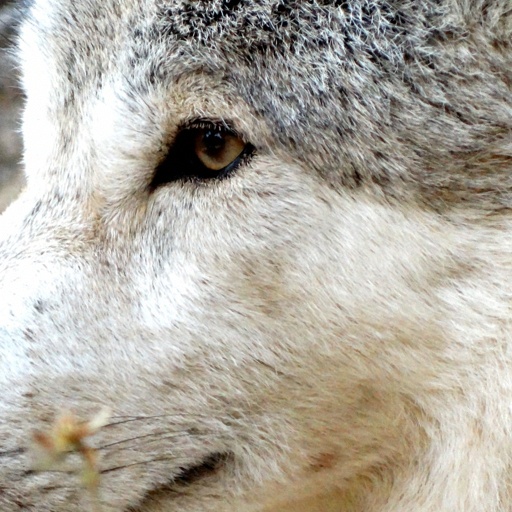}};
\begin{scope}[x={(image.south east)}, y={(image.north west)}]
\node[anchor=north west,
      fill=white, fill opacity=0.7, text opacity=1,
      inner sep=1pt, font=\scriptsize, rounded corners=1pt]
      at (0,1) {41,08 dB};
\end{scope}
\end{tikzpicture} &

\begin{tikzpicture}
\node[anchor=north west, inner sep=0] (image) at (0,0)
{\includegraphics[width=\linewidth]{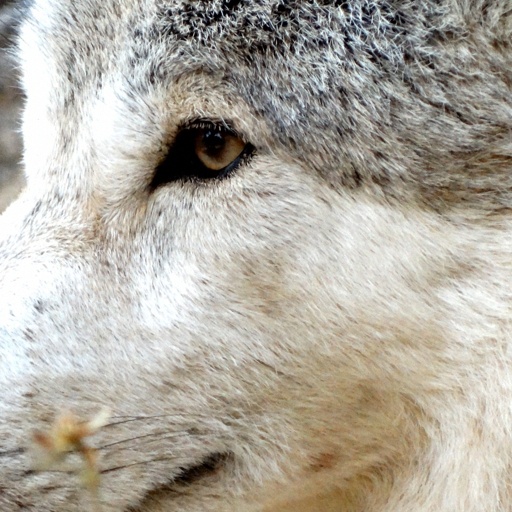}};
\begin{scope}[x={(image.south east)}, y={(image.north west)}]
\node[anchor=north west,
      fill=white, fill opacity=0.7, text opacity=1,
      inner sep=1pt, font=\scriptsize, rounded corners=1pt]
      at (0,1) {41.78 dB};
\end{scope}
\end{tikzpicture} &

\begin{tikzpicture}
\node[anchor=north west, inner sep=0] (image) at (0,0)
{\includegraphics[width=\linewidth]{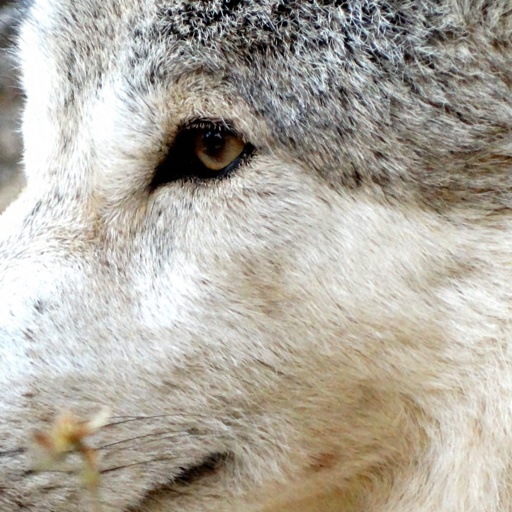}};
\begin{scope}[x={(image.south east)}, y={(image.north west)}]
\node[anchor=north west,
      fill=white, fill opacity=0.7, text opacity=1,
      inner sep=1pt, font=\scriptsize, rounded corners=1pt]
      at (0,1) {38.78 dB};
\end{scope}
\end{tikzpicture} &

\begin{tikzpicture}
\node[anchor=north west, inner sep=0] (image) at (0,0)
{\includegraphics[width=\linewidth]{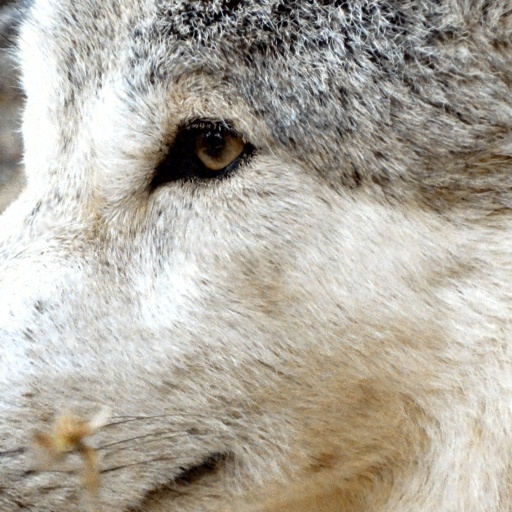}};
\begin{scope}[x={(image.south east)}, y={(image.north west)}]
\node[anchor=north west,
      fill=white, fill opacity=0.7, text opacity=1,
      inner sep=1pt, font=\scriptsize, rounded corners=1pt]
      at (0,1) {32.50 dB};
\end{scope}
\end{tikzpicture}
\\


\begin{tikzpicture}
\node[anchor=north west, inner sep=0] (image) at (0,0)
{\includegraphics[width=\linewidth]{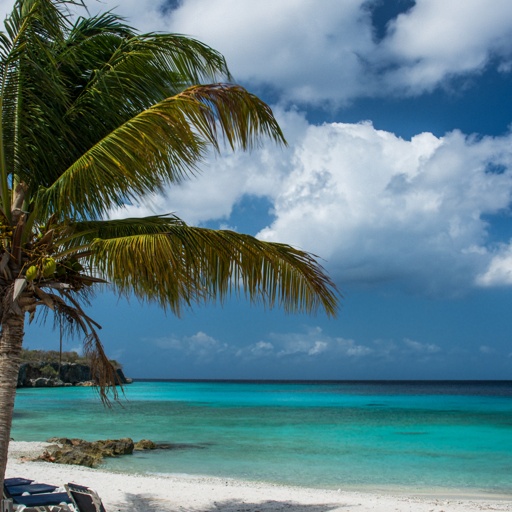}};
\begin{scope}[x={(image.south east)}, y={(image.north west)}]
\node[anchor=north west,
      fill=white, fill opacity=0.7, text opacity=1,
      inner sep=1pt, font=\scriptsize, rounded corners=1pt]
      at (0,1) {45.01 dB};
\end{scope}
\end{tikzpicture} &

\begin{tikzpicture}
\node[anchor=north west, inner sep=0] (image) at (0,0)
{\includegraphics[width=\linewidth]{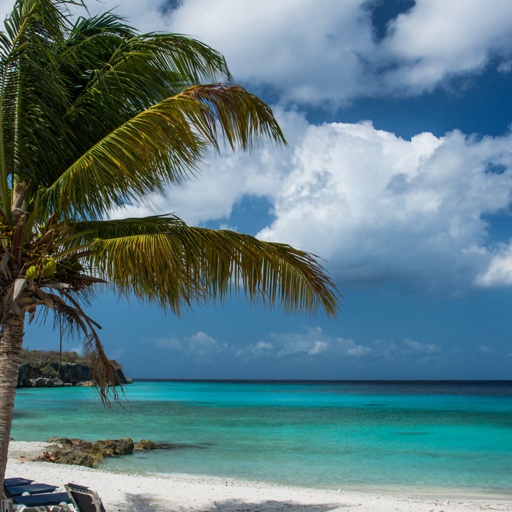}};
\begin{scope}[x={(image.south east)}, y={(image.north west)}]
\node[anchor=north west,
      fill=white, fill opacity=0.7, text opacity=1,
      inner sep=1pt, font=\scriptsize, rounded corners=1pt]
      at (0,1) {42.61 dB};
\end{scope}
\end{tikzpicture} &

\begin{tikzpicture}
\node[anchor=north west, inner sep=0] (image) at (0,0)
{\includegraphics[width=\linewidth]{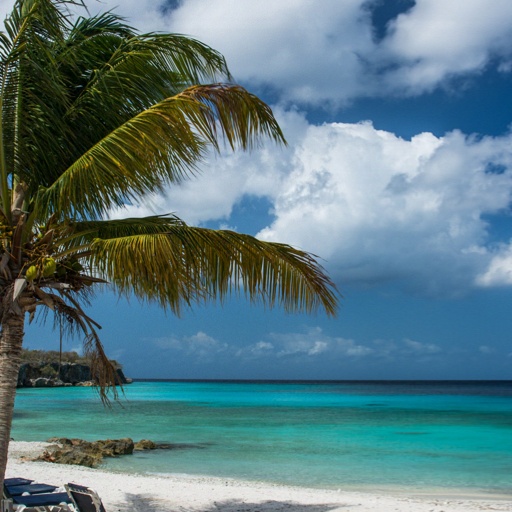}};
\begin{scope}[x={(image.south east)}, y={(image.north west)}]
\node[anchor=north west,
      fill=white, fill opacity=0.7, text opacity=1,
      inner sep=1pt, font=\scriptsize, rounded corners=1pt]
      at (0,1) {40.23 dB};
\end{scope}
\end{tikzpicture} &

\begin{tikzpicture}
\node[anchor=north west, inner sep=0] (image) at (0,0)
{\includegraphics[width=\linewidth]{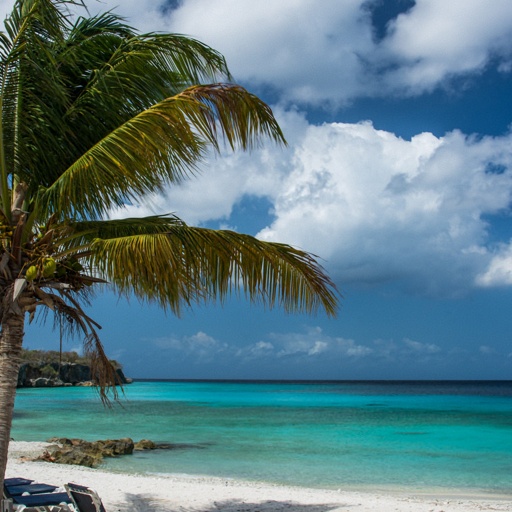}};
\begin{scope}[x={(image.south east)}, y={(image.north west)}]
\node[anchor=north west,
      fill=white, fill opacity=0.7, text opacity=1,
      inner sep=1pt, font=\scriptsize, rounded corners=1pt]
      at (0,1) {39.73 dB};
\end{scope}
\end{tikzpicture} &

\begin{tikzpicture}
\node[anchor=north west, inner sep=0] (image) at (0,0)
{\includegraphics[width=\linewidth]{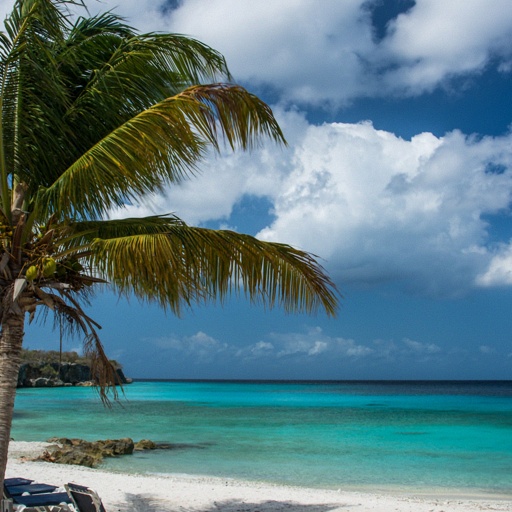}};
\begin{scope}[x={(image.south east)}, y={(image.north west)}]
\node[anchor=north west,
      fill=white, fill opacity=0.7, text opacity=1,
      inner sep=1pt, font=\scriptsize, rounded corners=1pt]
      at (0,1) {38.92 dB};
\end{scope}
\end{tikzpicture} &

\begin{tikzpicture}
\node[anchor=north west, inner sep=0] (image) at (0,0)
{\includegraphics[width=\linewidth]{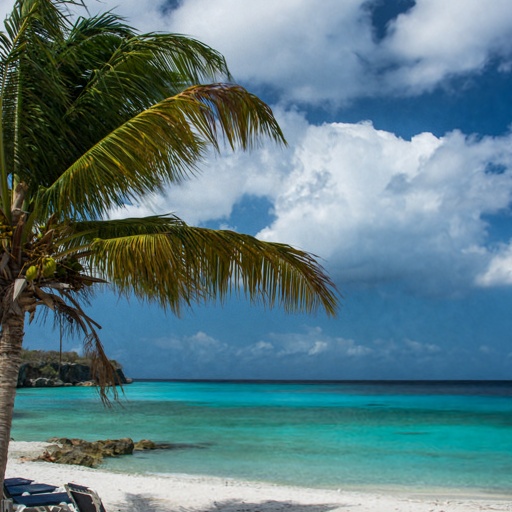}};
\begin{scope}[x={(image.south east)}, y={(image.north west)}]
\node[anchor=north west,
      fill=white, fill opacity=0.7, text opacity=1,
      inner sep=1pt, font=\scriptsize, rounded corners=1pt]
      at (0,1) {37.58 dB};
\end{scope}
\end{tikzpicture} &

\begin{tikzpicture}
\node[anchor=north west, inner sep=0] (image) at (0,0)
{\includegraphics[width=\linewidth]{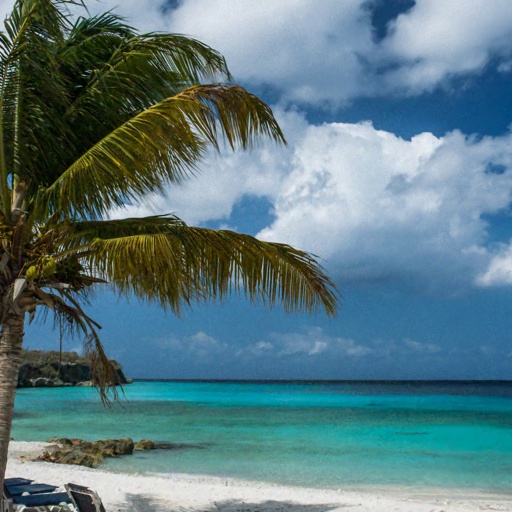}};
\begin{scope}[x={(image.south east)}, y={(image.north west)}]
\node[anchor=north west,
      fill=white, fill opacity=0.7, text opacity=1,
      inner sep=1pt, font=\scriptsize, rounded corners=1pt]
      at (0,1) {33.50 dB};
\end{scope}
\end{tikzpicture}
\\

\end{tabular}

\caption{Visualization of the 2D image fitting task. PSNR values are shown at the upper-left corner of each image.}
\label{fig:sup_2d_img_fiting_vis}
\end{figure*}
}

{
\setlength{\tabcolsep}{1pt}
\begin{figure*}[htbp]
    \centering
    \resizebox{\linewidth}{!}{
    \begin{tabular}{cccccc} %
        \includegraphics[width=0.15\textwidth]{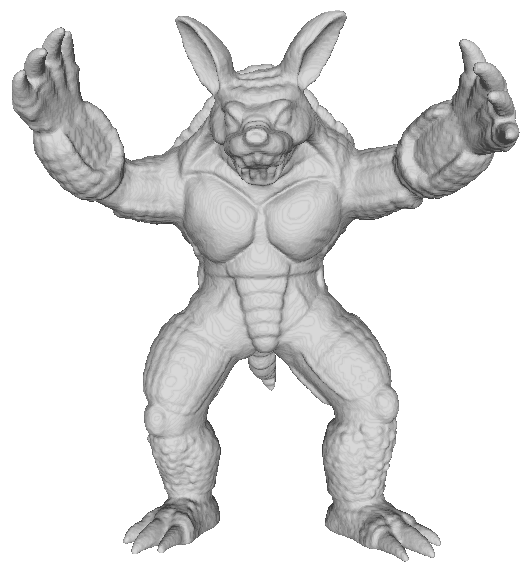} &
        \includegraphics[width=0.15\textwidth]{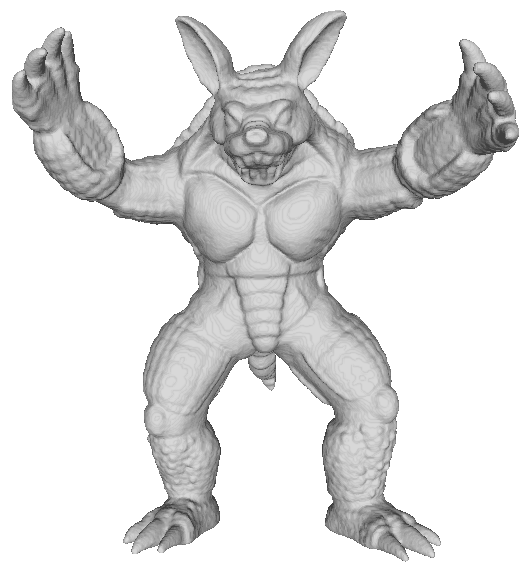} &
        \includegraphics[width=0.15\textwidth]{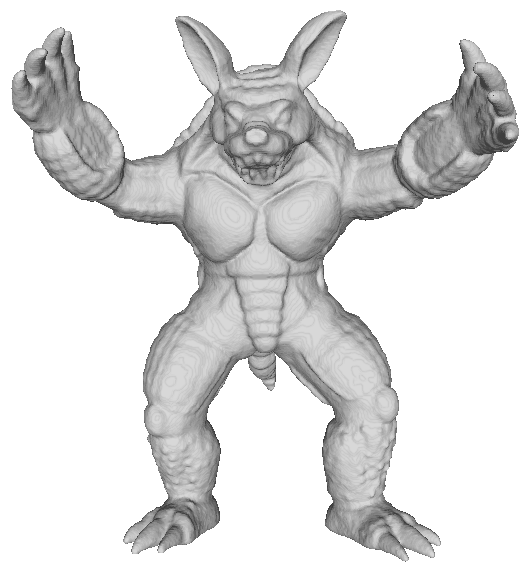} &
        \includegraphics[width=0.15\textwidth]{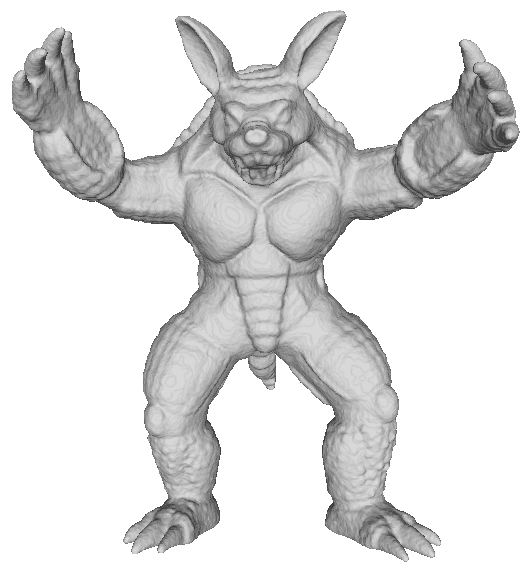} &
        \includegraphics[width=0.15\textwidth]{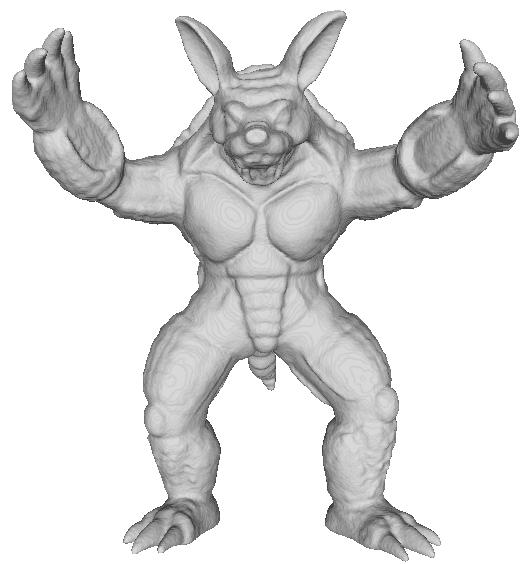} &
        \includegraphics[width=0.15\textwidth]{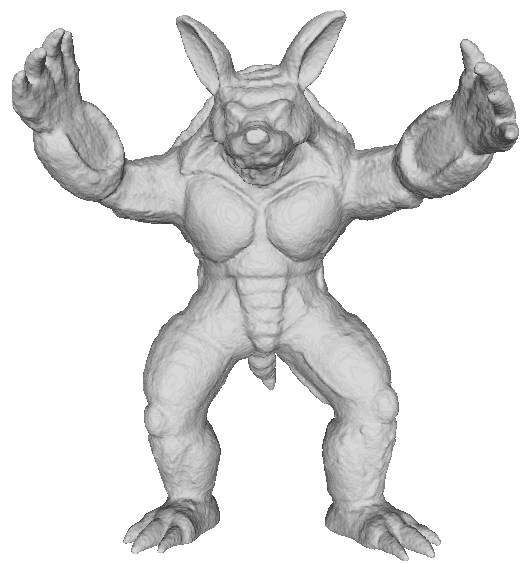} 
        
        \\

        \textbf{Ours(0.9996)} &  SL$^2$A(0.9993) & FINER(0.9950) & SCONE(0.9905) & SIREN(0.9933) & WIRE(0.9861) \\
        
        \includegraphics[width=0.16\textwidth]{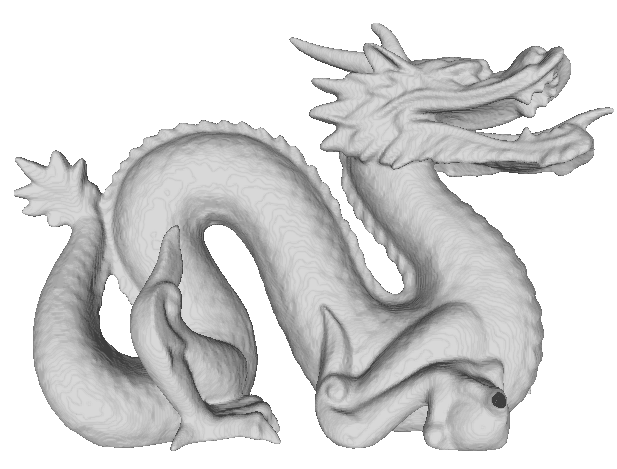} &
        \includegraphics[width=0.16\textwidth]{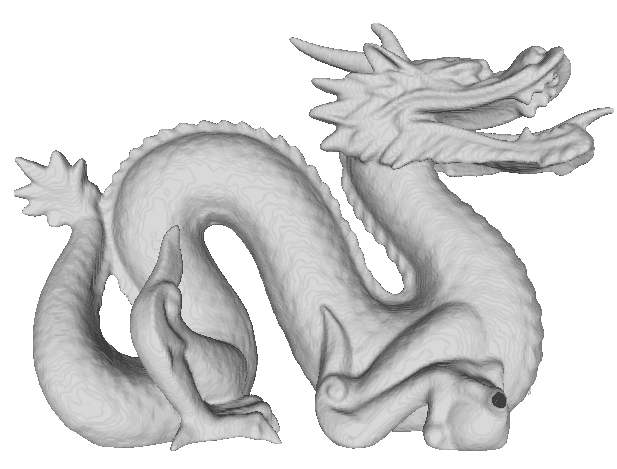} &
        \includegraphics[width=0.16\textwidth]{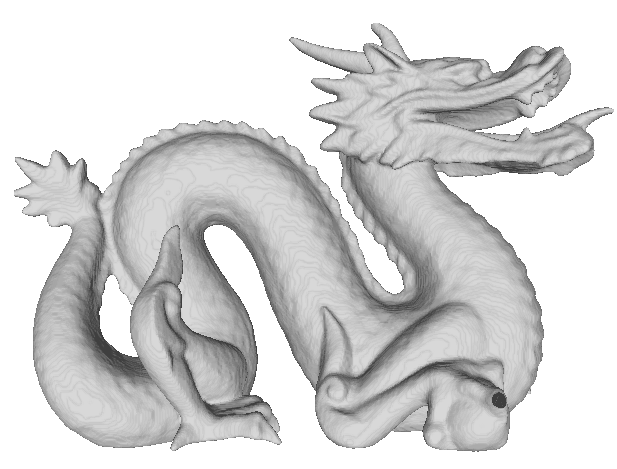} &
        \includegraphics[width=0.16\textwidth]{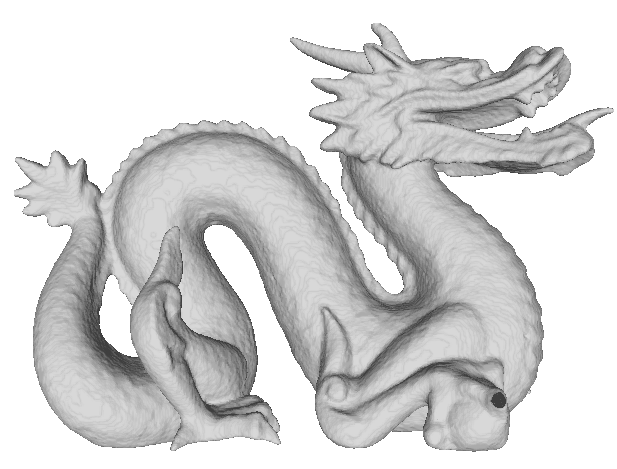} &
        \includegraphics[width=0.16\textwidth]{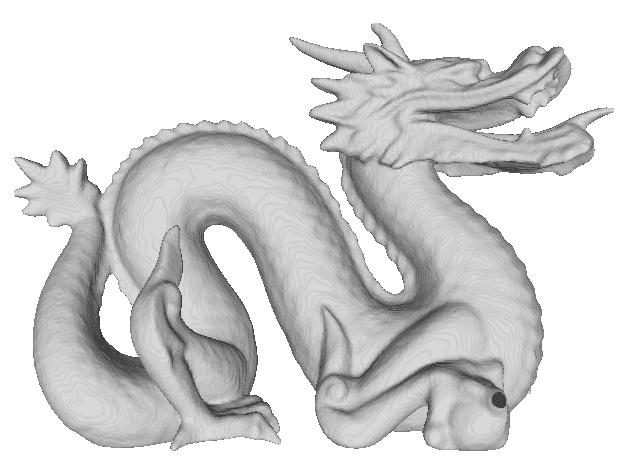} &
        \includegraphics[width=0.16\textwidth]{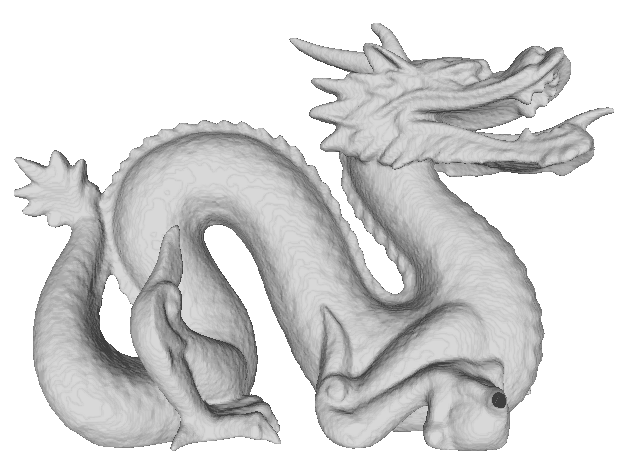}  \\

        \textbf{Ours(0.9994)} &  SL$^2$A(0.9991) & FINER(0.9959) & SCONE(0.9918) & SIREN(0.9940) & WIRE(0.9934) \\

        \includegraphics[width=0.16\textwidth]{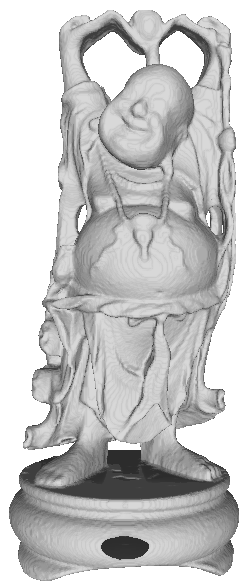} &
        \includegraphics[width=0.16\textwidth]{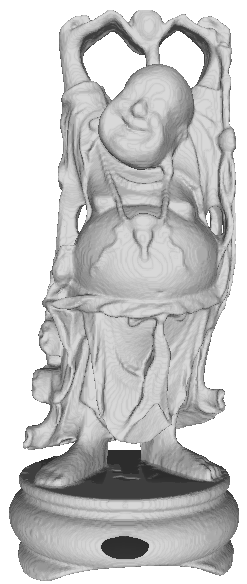} &
        \includegraphics[width=0.16\textwidth]{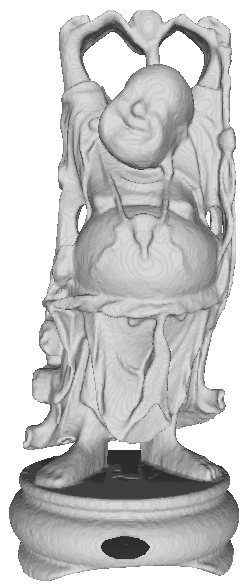} &
        \includegraphics[width=0.16\textwidth]{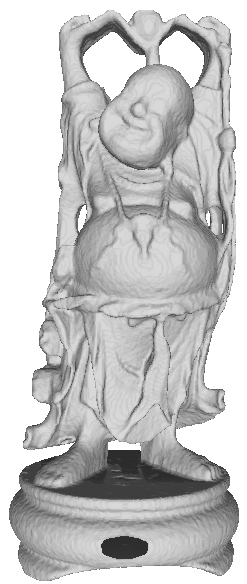} &
        \includegraphics[width=0.16\textwidth]{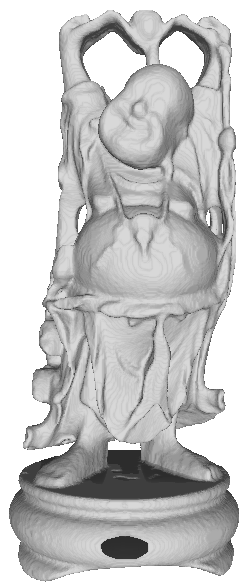} &
        \includegraphics[width=0.16\textwidth]{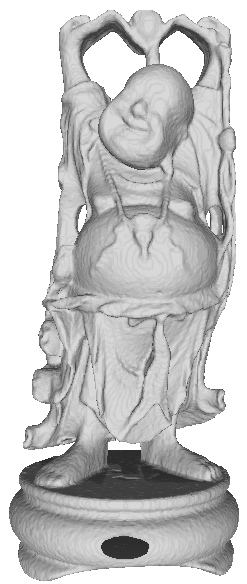}  \\

        \textbf{Ours(0.9995)} &  SL$^2$A(0.9994) & FINER(0.9958) & SCONE(0.9924) & SIREN(0.9934) & WIRE(0.9945) \\

        \includegraphics[width=0.16\textwidth]{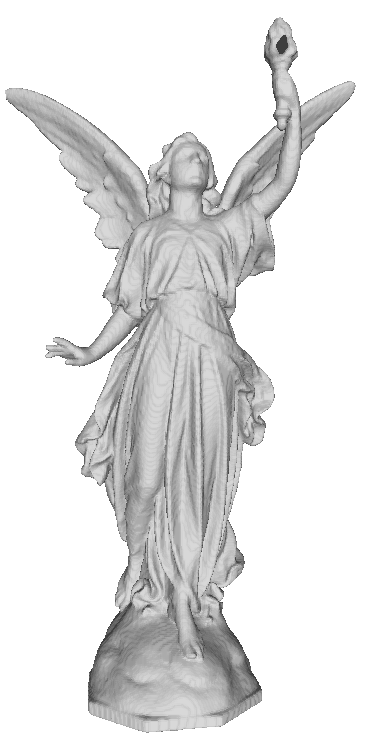} &
        \includegraphics[width=0.16\textwidth]{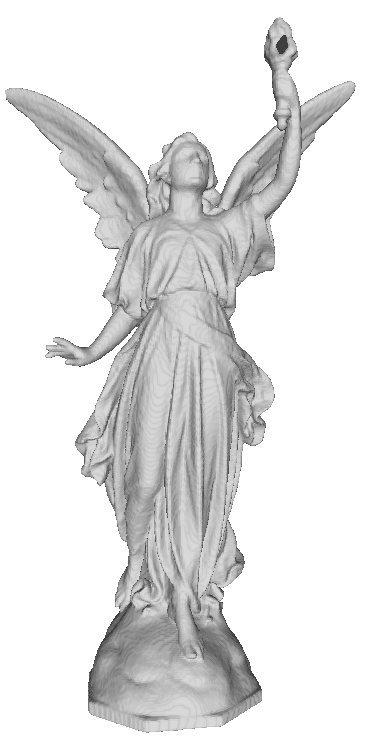} &
        \includegraphics[width=0.16\textwidth]{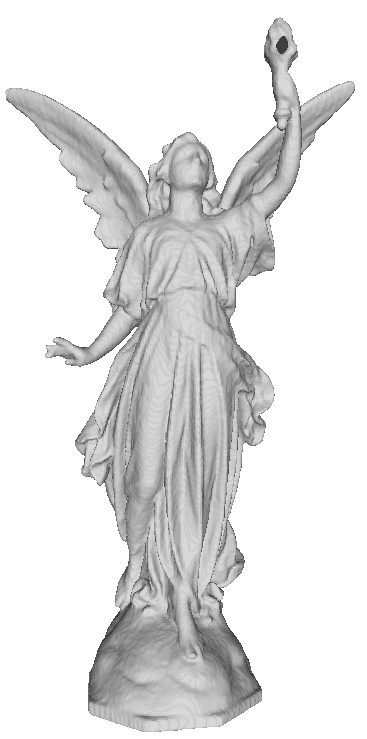} &
        \includegraphics[width=0.16\textwidth]{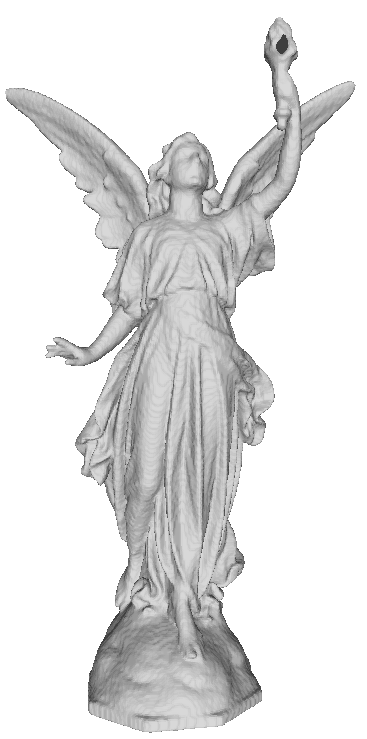} &
        \includegraphics[width=0.16\textwidth]{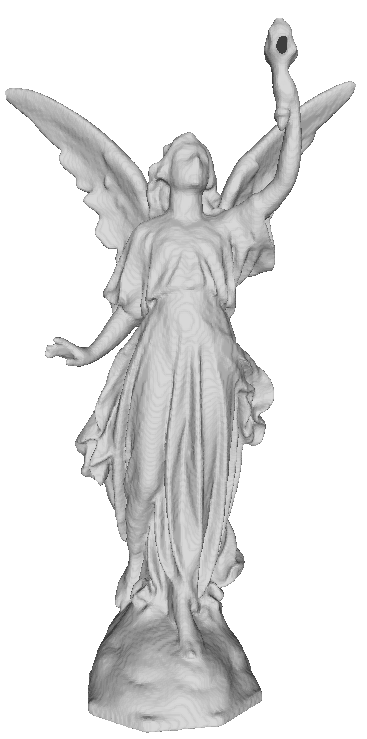} &
        \includegraphics[width=0.16\textwidth]{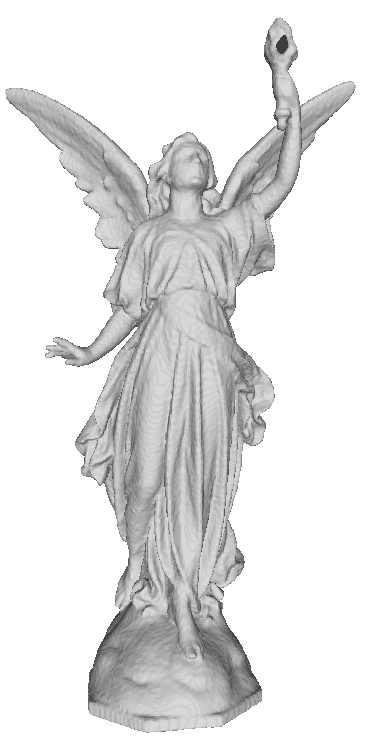}  \\

        \textbf{Ours(0.9995)} &  SL$^2$A(0.9992) & FINER(0.9929) & SCONE(0.9905) & SIREN(0.9866) & WIRE(0.9930) \\
    \end{tabular}
    }
    \caption{Visualization of the 3D shape representation task.
}
\label{sup_fig:3D_shape}
\end{figure*}
}

{
\setlength{\tabcolsep}{0.5pt}
\begin{figure*}[htbp]
    \centering
    \resizebox{\linewidth}{!}{
    \begin{tabular}{cccccc}

        \includegraphics[width=0.16\textwidth]{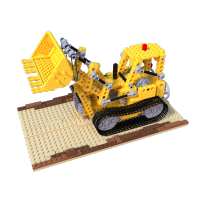} &
        \includegraphics[width=0.16\textwidth]{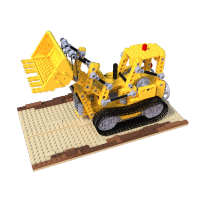} &
        \includegraphics[width=0.16\textwidth]{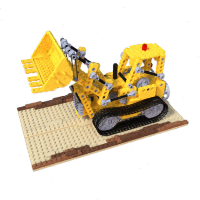} &
        \includegraphics[width=0.16\textwidth]{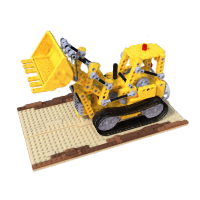} &
        \includegraphics[width=0.16\textwidth]{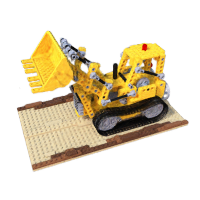} &
        \includegraphics[width=0.16\textwidth]{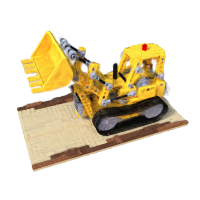} \\
        GT(PSNR) & \textbf{Ours(31.86)} & FINER(30.78) & SIREN(29.94) & WIRE(29.22) & Gauss(26.40) \\

                \includegraphics[width=0.16\textwidth]{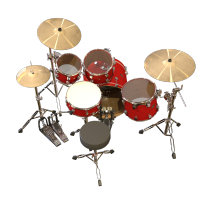} &
        \includegraphics[width=0.16\textwidth]{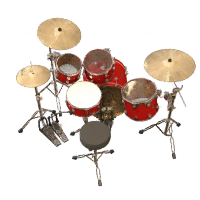} &
        \includegraphics[width=0.16\textwidth]{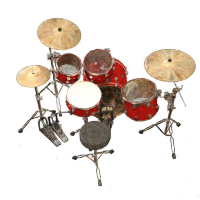} &
        \includegraphics[width=0.16\textwidth]{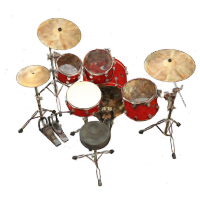} &
        \includegraphics[width=0.16\textwidth]{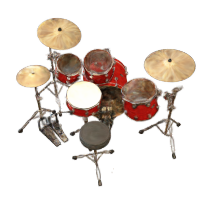} &
        \includegraphics[width=0.16\textwidth]{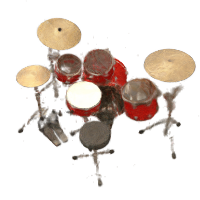} \\
                GT(PSNR) & Ours(25.59) & \textbf{FINER(25.66)} & SIREN(25.55) & WIRE(24.57) & Gauss(22.46) \\
        \includegraphics[width=0.16\textwidth]{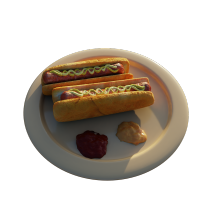} &
        \includegraphics[width=0.16\textwidth]{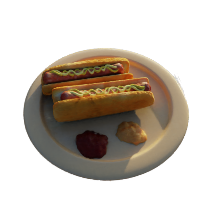} &
        \includegraphics[width=0.16\textwidth]{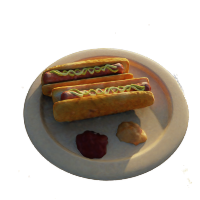} &
        \includegraphics[width=0.16\textwidth]{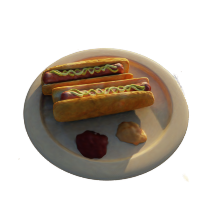} &
        \includegraphics[width=0.16\textwidth]{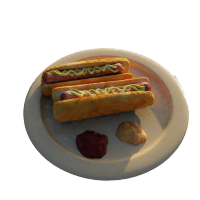} &
        \includegraphics[width=0.16\textwidth]{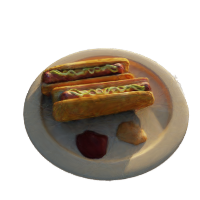} \\
                GT(PSNR) & \textbf{Ours(34.65)} & FINER(33.80) & SIREN(33.91) & WIRE(32.78) & Gauss(31.28) \\
        \includegraphics[width=0.16\textwidth]{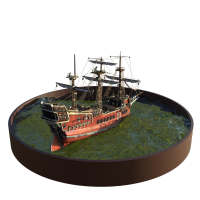} &
        \includegraphics[width=0.16\textwidth]{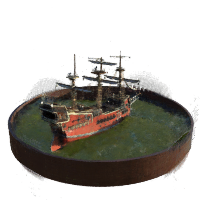} &
        \includegraphics[width=0.16\textwidth]{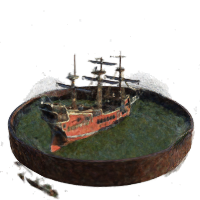} &
        \includegraphics[width=0.16\textwidth]{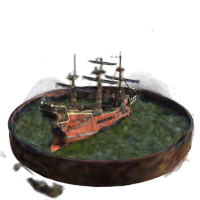} &
        \includegraphics[width=0.16\textwidth]{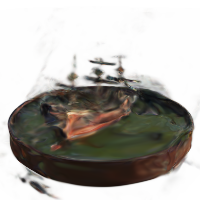} &
        \includegraphics[width=0.16\textwidth]{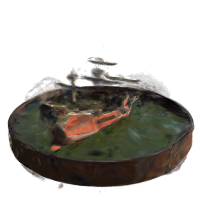} \\
                GT(PSNR) & \textbf{Ours(23.12)} & FINER(22.02) & SIREN(21.79) & WIRE(21.33) & Gauss(21.24) \\
        
    \end{tabular}
    }
    
    \caption{Visualization of the neural radiance fields task.
    \label{vis:nerf1}
}
\end{figure*}
}

\setlength{\tabcolsep}{10pt}
\begin{figure*}[htbp]
    \centering

    \begin{tabular}{ccc}
        \includegraphics[width=0.26\textwidth]{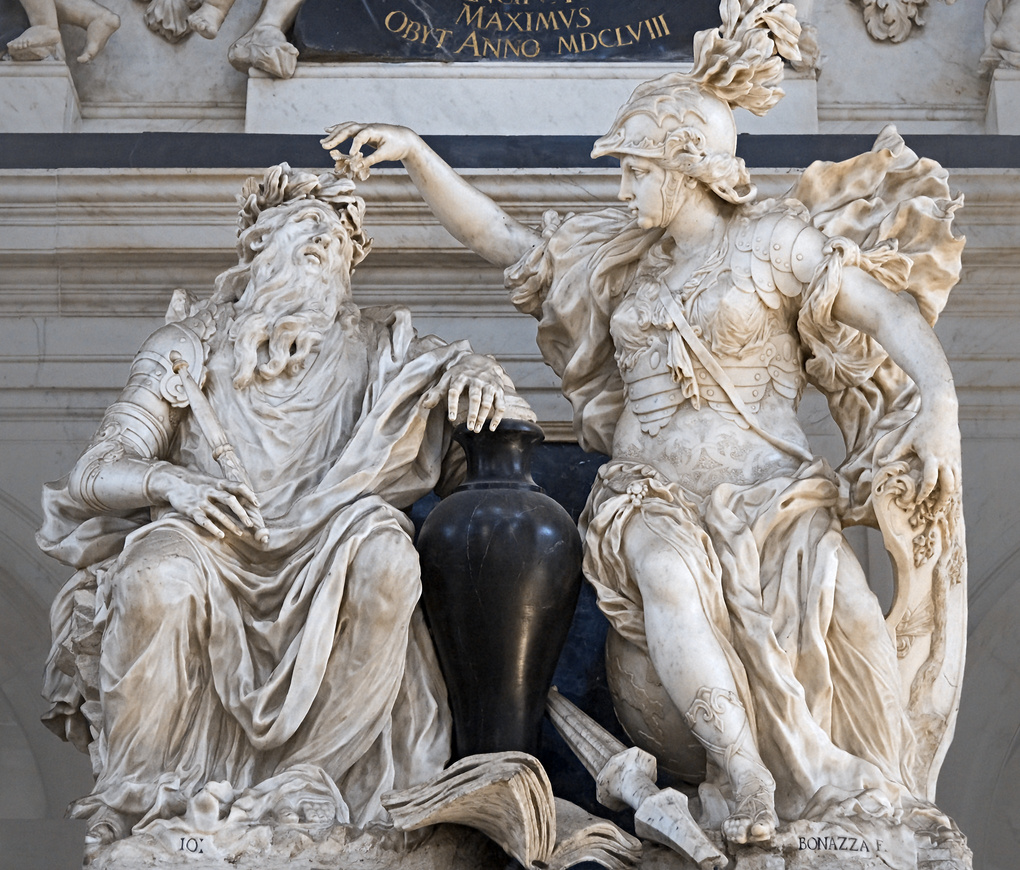} &
        \includegraphics[width=0.26\textwidth]{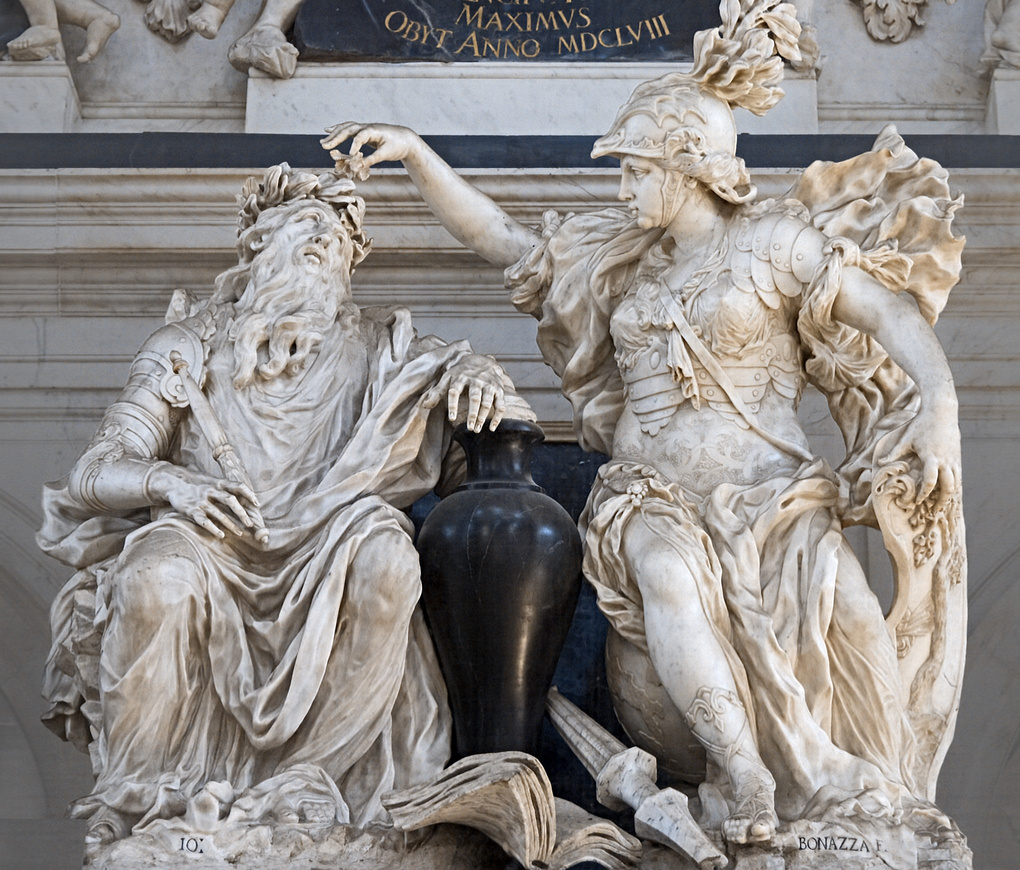} &
        \includegraphics[width=0.26\textwidth]{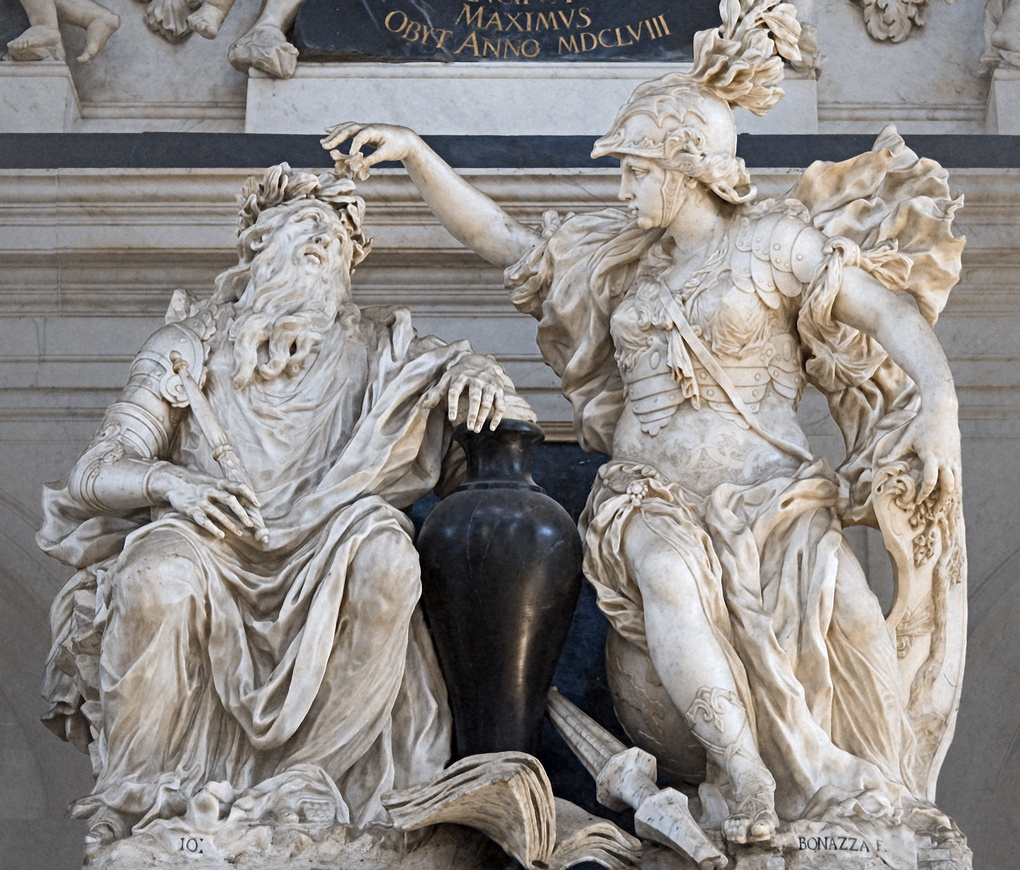} \\
        \textbf{Ours (35.79)} & SL$^2$A (34.00) & FINER (32.66) \\
       \includegraphics[width=0.26\textwidth]{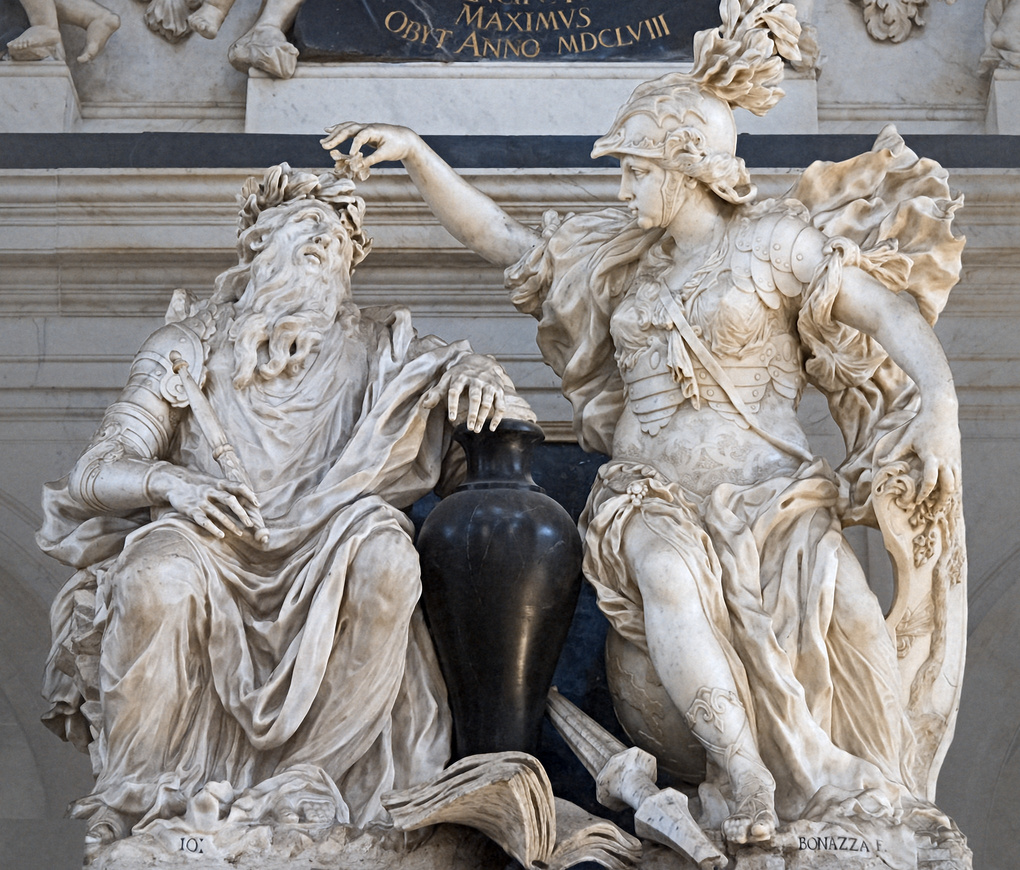} &
        \includegraphics[width=0.26\textwidth]{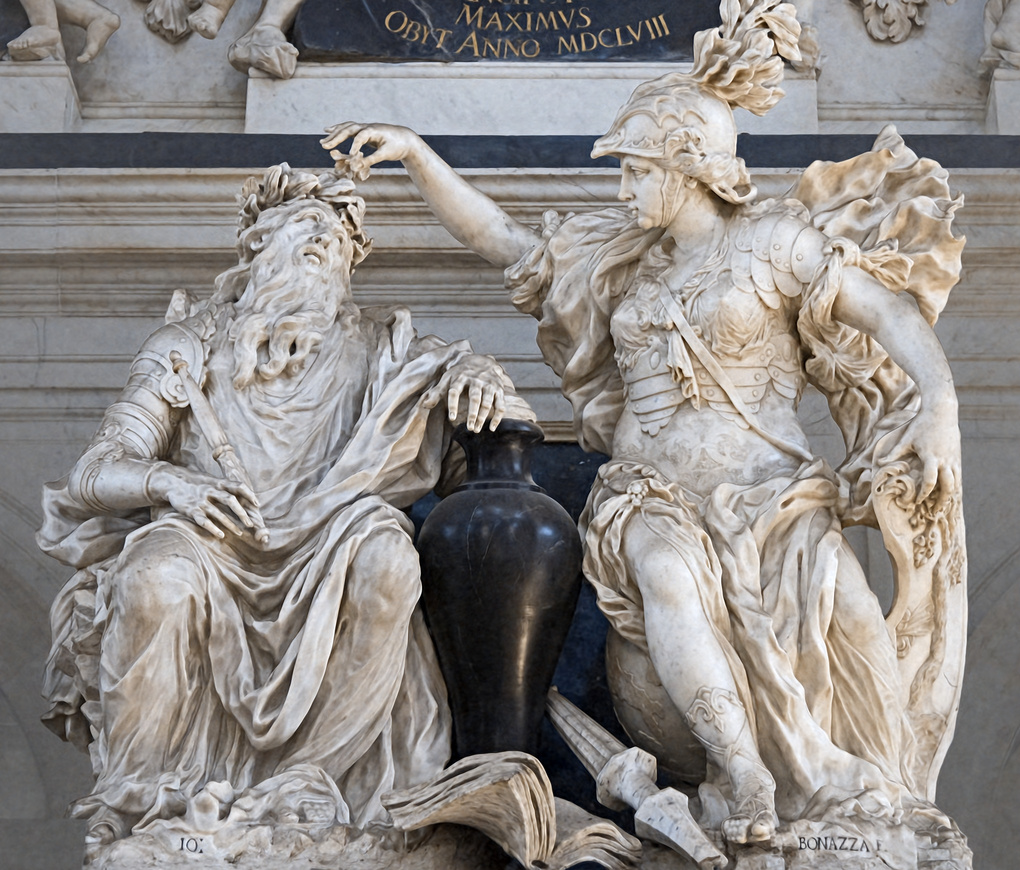} &
        \includegraphics[width=0.26\textwidth]{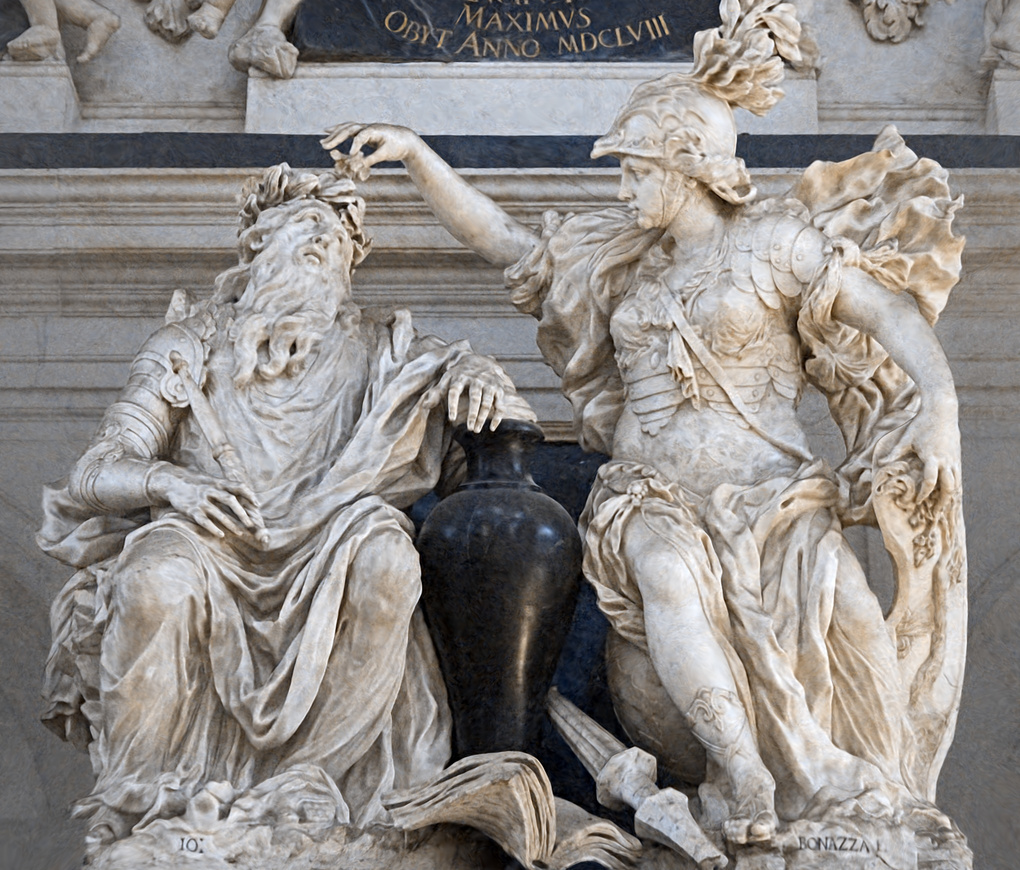} \\
        SCONE (34.09) & SIREN (32.42) & WIRE (29.12)\\

        \\
        \includegraphics[width=0.26\textwidth]{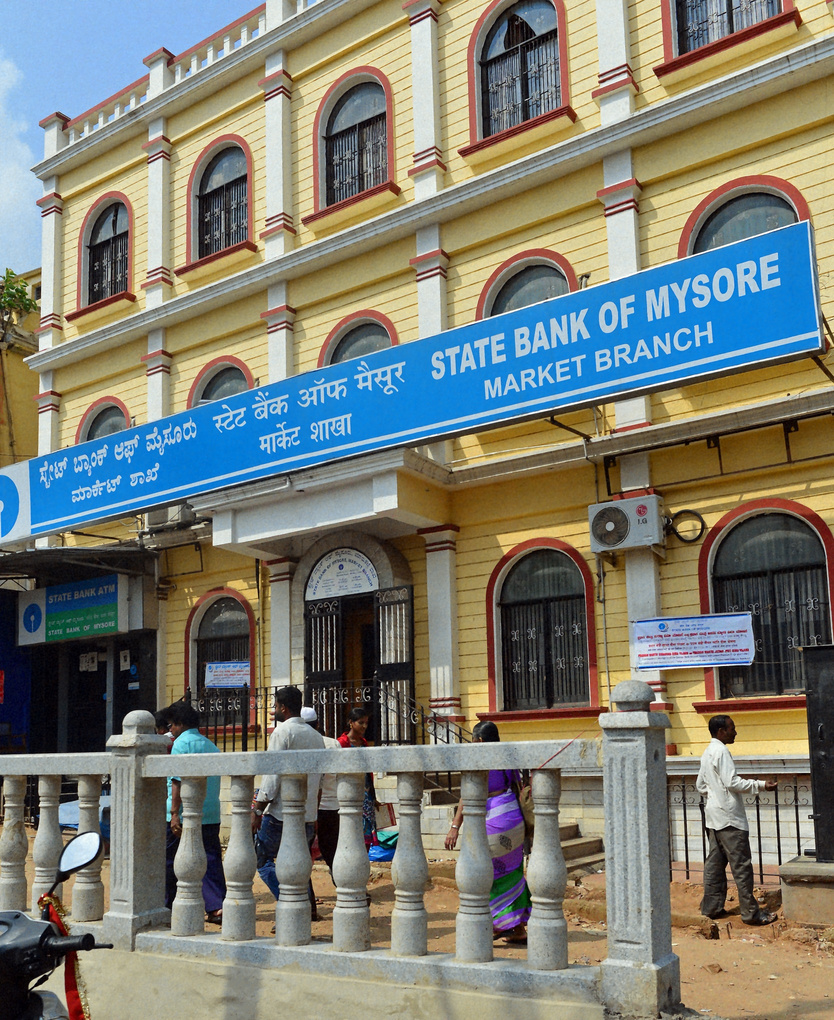} &
        \includegraphics[width=0.26\textwidth]{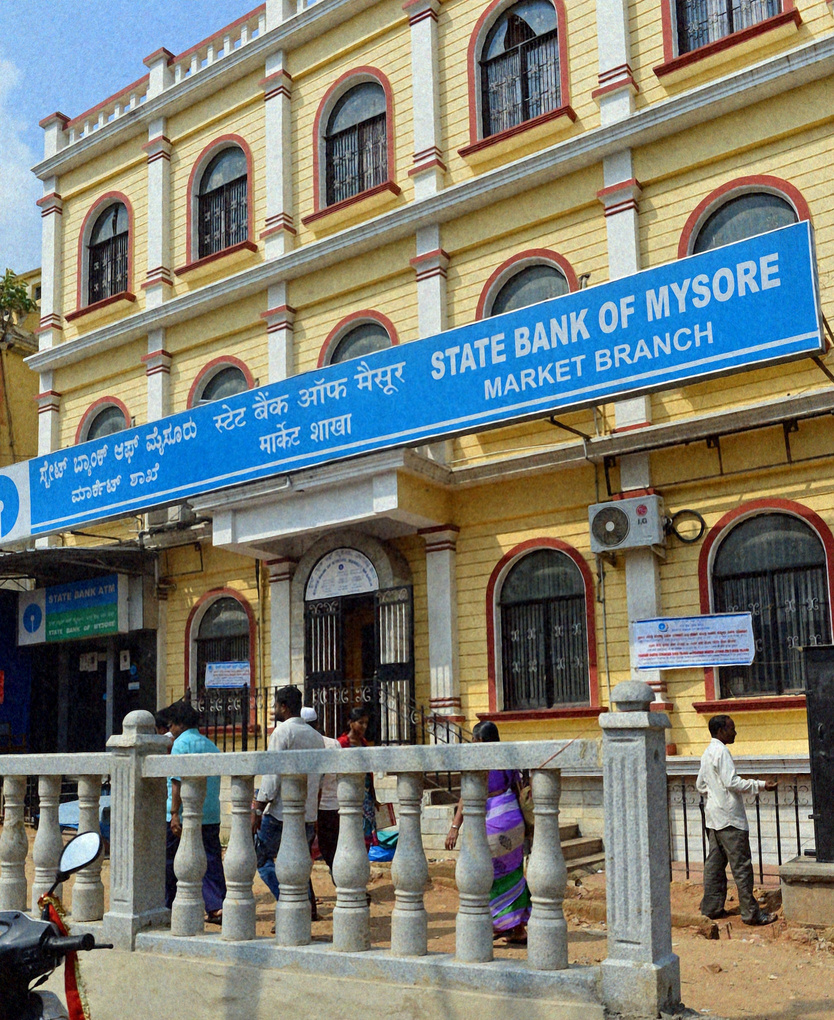} &
        \includegraphics[width=0.26\textwidth]{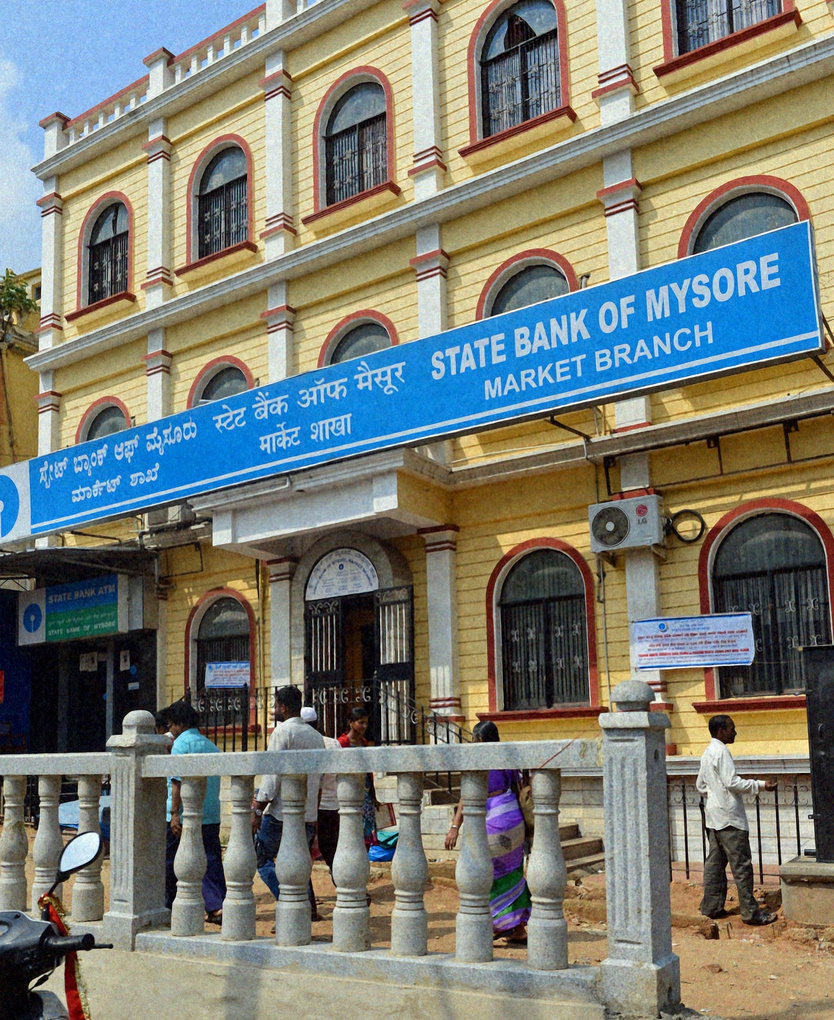} \\
        \textbf{Ours (30.79)} & SL$^2$A (26.82) & FINER (26.77) \\
       \includegraphics[width=0.26\textwidth]{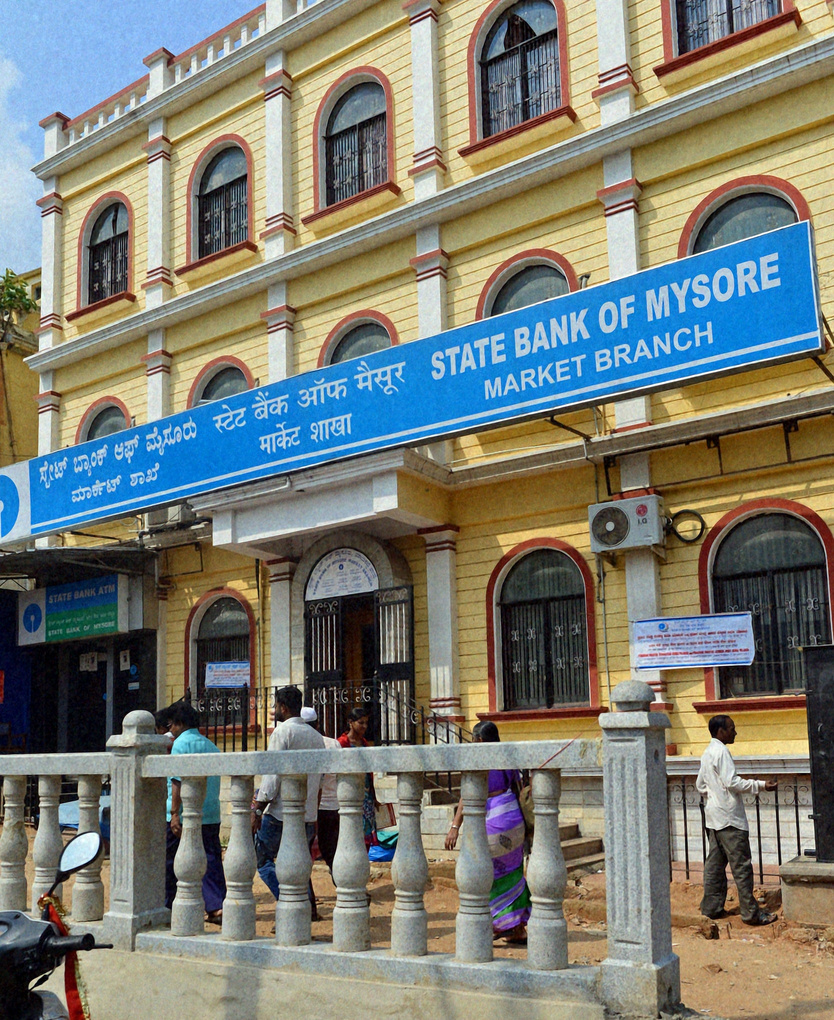} &
        \includegraphics[width=0.26\textwidth]{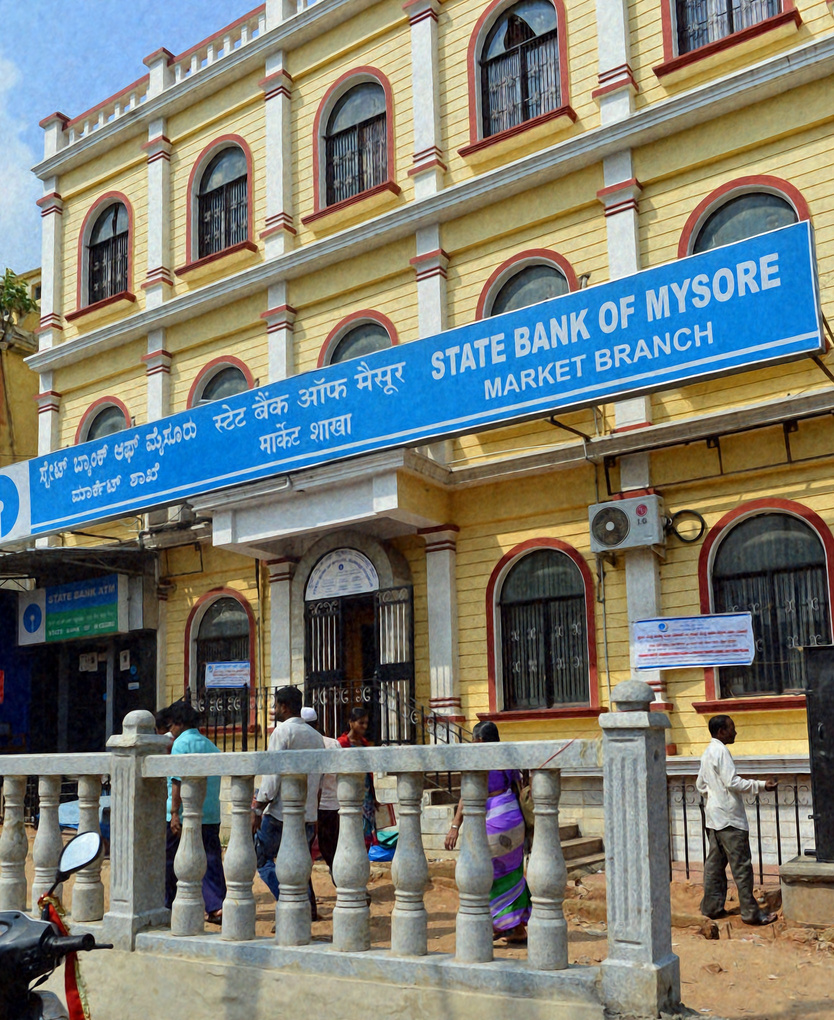} &
        \includegraphics[width=0.26\textwidth]{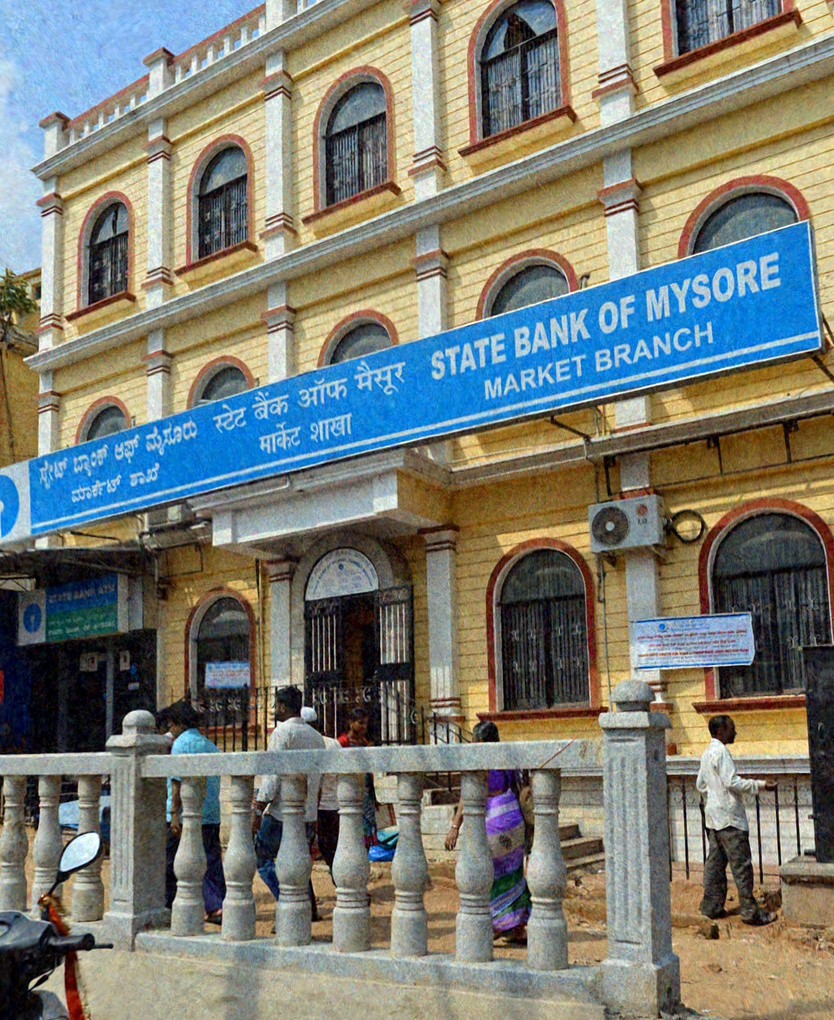} \\
        SCONE (27.68) & SIREN (27.20) & WIRE (23.99)
        
    \end{tabular}

    \caption{Visualization of the full-resolution representation results on the DIV2K images 0878 and 0891.}
    \label{vis:div2k_full1}
\end{figure*}

\end{document}